%% file: 0-main.tex
\documentclass{article}

\usepackage[final]{corl_2019} 

\usepackage[utf8]{inputenc} 
\usepackage[T1]{fontenc}    
\usepackage{hyperref}       
\usepackage{enumitem}
\usepackage{url}            
\usepackage{booktabs}       
\usepackage{amsfonts}       
\usepackage{nicefrac}       
\usepackage{microtype}      
\usepackage{algorithm} 
\usepackage{algorithmicx, algpseudocode} 
\usepackage{mathtools}
\usepackage{amsmath, amssymb, amscd}
\usepackage{amsthm}
\usepackage{wasysym} 
\usepackage{amsfonts}
\usepackage{gensymb} 
\usepackage{nicefrac} 
\usepackage{verbatim}
\usepackage[skip=2pt,font=small]{caption}
\usepackage{subcaption}
\usepackage{bbold}
\usepackage{xspace}
\usepackage{color}
\usepackage{soul} 
\usepackage{xr}
\usepackage{chngcntr} 

\newcommand{\States}{\mathcal{S}}

\newcommand{\Goals}{\mathcal{G}}
\newcommand{\Traj}{\textbf{T}}

\DeclareMathOperator*{\argmax}{arg\,max}
\newcommand{\semicheckmark}{\checkmark\kern-1.1ex\raisebox{.7ex}{\rotatebox[origin=c]{125}{--}}
}

\newcommand{\etal}{\emph{et al.}\xspace}
\newtheorem{prop}{Proposition}

\usepackage{xr}
\makeatletter
\newcommand*{\addFileDependency}[1]{
  \typeout{(#1)}
  \@addtofilelist{#1}
  \IfFileExists{#1}{}{\typeout{No file #1.}}
}
\makeatother
 
\newcommand*{\inputexternaldocument}[1]{%
    \externaldocument{#1}%
    \addFileDependency{#1.tex}%
    \addFileDependency{#1.aux}%
}

\inputexternaldocument{0-appendix}

\newcommand{\methodname}{AC-Teach\xspace}
\newcommand{\methodnameFull}{Actor-Critic with Teacher Ensembles\xspace}
\title{\LARGE{AC-Teach: A Bayesian Actor-Critic Method for Policy Learning with an Ensemble of Suboptimal Teachers}}
\author{Andrey Kurenkov$^{*\sharp}$, Ajay Mandlekar$^{*\sharp}$, Roberto Martin-Martin$^{\sharp}$, Silvio Savarese$^{\sharp}$, Animesh Garg$^{\dagger}$\\
$^{\sharp}$Stanford University, $^
{\dagger}$University of Toronto, $^
{\dagger}$Vector Institute
}
\begin{document}

\maketitle

\begin{abstract}
The exploration mechanism used by a Deep Reinforcement Learning (RL) agent plays a key role in determining its sample efficiency. 
Thus, improving over random exploration is crucial to solve long-horizon tasks with sparse rewards. 
We propose to leverage an ensemble of partial solutions as \textit{teachers} that guide the agent's exploration with action suggestions throughout training. While the setup of learning with teachers has been previously studied, our proposed approach -- \methodnameFull (\methodname) -- is the first to work with an ensemble of suboptimal teachers that may solve only part of the problem or contradict other each other, forming a unified algorithmic solution that is compatible with a broad range of teacher ensembles.
\methodname~leverages a probabilistic representation of the expected outcome of the teachers' and student's actions to direct exploration, reduce dithering, and adapt to the dynamically changing quality of the learner. We evaluate a variant of \methodname~that guides the learning of a Bayesian DDPG agent on three tasks -- path following, robotic pick and place, and robotic cube sweeping using a hook -- and show that it improves largely on sampling efficiency over a set of baselines, both for our target scenario of unconstrained suboptimal teachers and for easier setups with optimal or single teachers. 
Additional results and videos at \url{https://sites.google.com/view/acteach/home}.
\end{abstract}

\input{1-intro.tex}

\input{2-rw.tex}

\input{3-probl.tex}

\input{4-method.tex}

\input{5-exps.tex}

\input{6-results.tex}

\vspace{-5pt}
\section{Conclusion}
\vspace{-3pt}
We presented \methodname, a unifying approach to leverage advice from an ensemble of sub-optimal teachers to accelerate the learning process of an off-policy RL agent. \methodname incorporates teachers' advice into the behavioral policy based on a Thomson sampling mechanism on the probabilistic evaluations of a Bayesian critic. We compare \methodname with prior approaches and show that it  extracts useful exploratory experiences when the set of teachers is noisy, partial, incomplete, or even contradictory. In the future, we plan to apply \methodname to real robot policy learning to demonstrate its applicability to solving challenging long-horizon manipulation in the real world.

\textbf{Acknowledgments} Toyota Research Institute (“TRI”) provided funds to assist the authors with their research but this
article solely reflects the opinions and conclusions of its
authors and not TRI or any other Toyota entity. AM is supported by the Department of Defense (DoD) through the National Defense Science \& Engineering Graduate Fellowship (NDSEG) Program.

\clearpage
\renewcommand*{\bibfont}{\footnotesize}

\renewcommand{\baselinestretch}{.95}
\bibliography{refs}

\input{0-appendix.tex}

\clearpage

\end{document}

%% file: 1-intro.tex
\vspace{-3pt}
\section{Introduction}
\vspace{-3pt}

Reinforcement Learning (RL) algorithms have recently demonstrated impressive results in challenging problem domains such as robotic manipulation~\cite{andrychowicz2018learning, kalashnikov2018qt}, planning~\cite{faust2018prm}, Go~\cite{silver2017mastering}, and Atari games~\cite{mnih2015human}. 
However, RL algorithms typically require a large number of interactions with the environment to train policies that solve new tasks, which is particularly problematic for physical domains such as robotics, where gathering experience from interactions is slow and expensive. A possible approach to alleviate this problem is to train the agent to reproduce or bootstrap from the behavior demonstrated by an expert using either offline or online Imitation Learning (IL)~\cite{ross2011reduction, ho2016generative, zhang2016query,  zhang2018deep, zhu2018reinforcement, ddpgfd}.
However, IL has its own limitations: applying offline IL for long horizon tasks is costly and time consuming because large amounts of expert interactions are necessary to cover corner-cases the agent may encounter. Online IL may overcome this limitation by allowing the agent to query experiences on-demand, but requiring an online expert that can solve the entire long-horizon task may be an even stricter limitation.

\begin{figure}[t!]
    \vspace{-10pt}
    \centering
    \begin{subfigure}[b]{0.4\textwidth}
        \centering
        \includegraphics[height=1.7in]{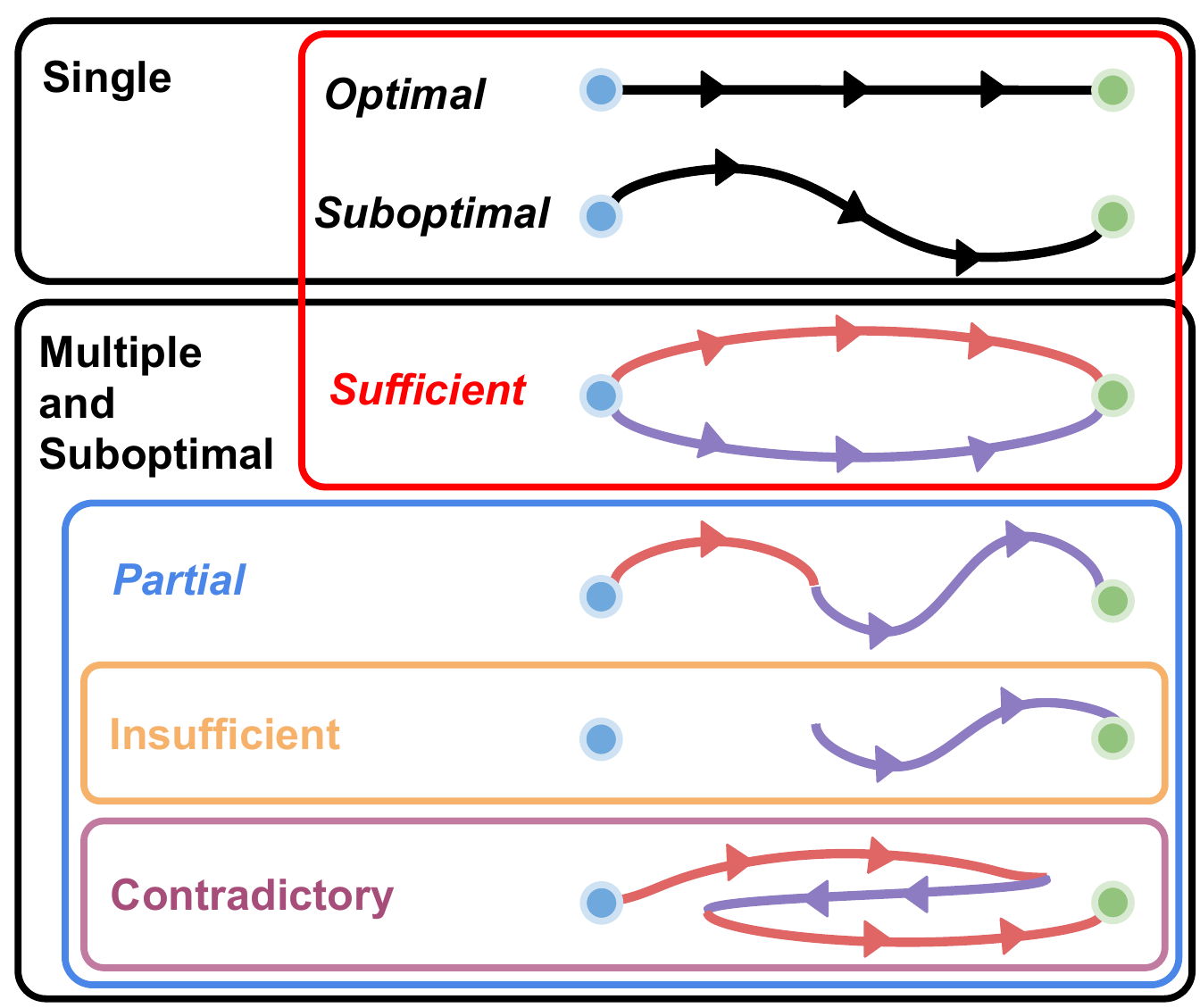}
        \caption{}
    \end{subfigure}%
    \hfill
    \begin{subfigure}[b]{0.6\textwidth}
        \centering
        \includegraphics[height=1.7in]{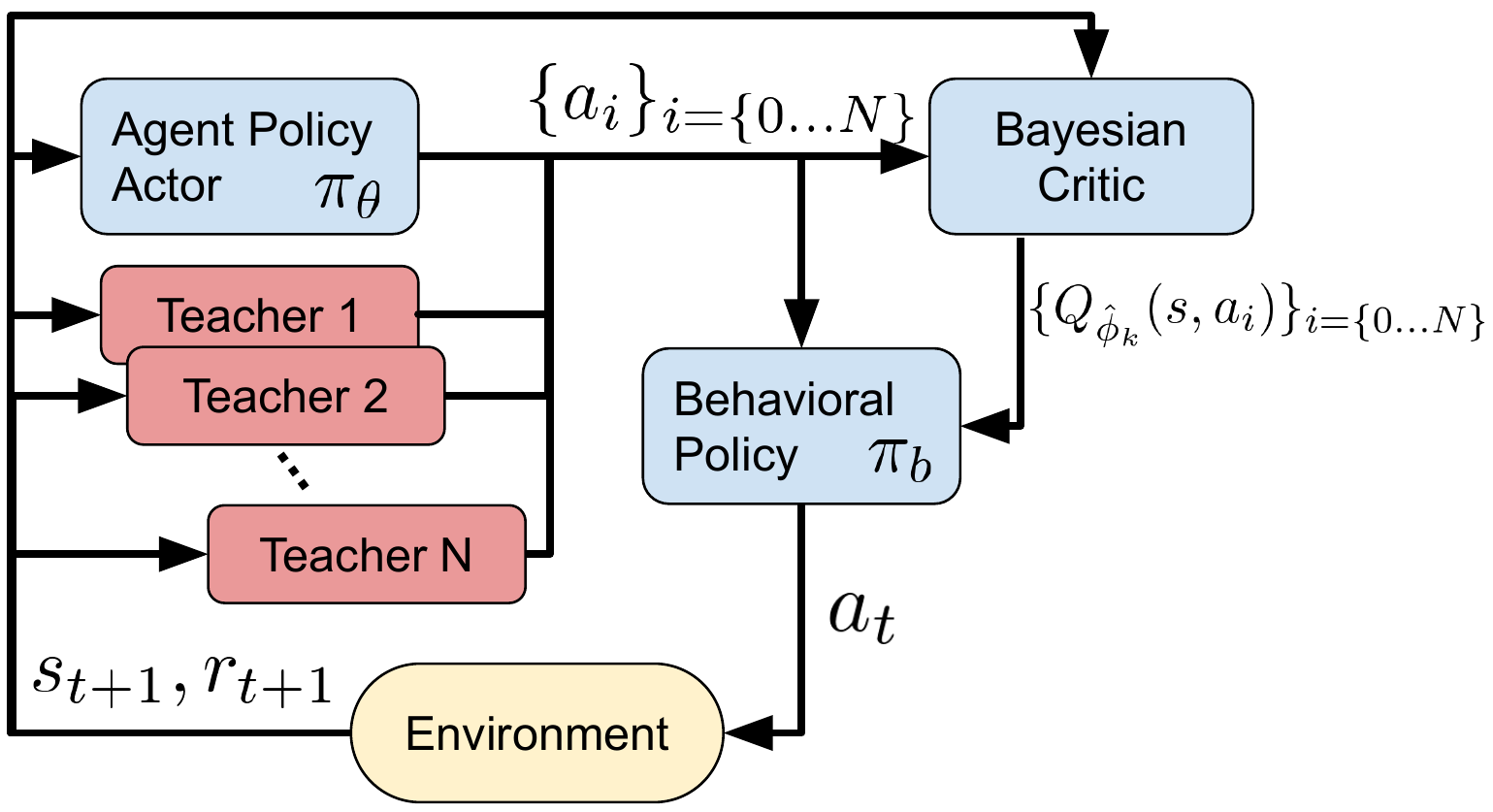}
        \caption{}
    \end{subfigure}
    \caption{(a) Visual representation of teachers' attributes, with arrows representing actions and line color representing different teacher policies. In this figure each example trajectory has the attributes of all the boxes it is contained within. Italicized terms apply to single policies, and non-italicized terms refer to sets of policies. We present formal definitions of these attributes in Sec. \ref{s:baps}. (b) Our method, \methodname, (light blue) modifies the actor-critic architecture to leverage advice from a set of teachers (red) as part of the Behavioral Policy to train the Agent Policy using the probabilistic q-values of a Bayesian critic.}
    \label{fig:teachers}
    \vspace{-15pt}
\end{figure}

We argue that, in domains like robotics, a useful alternative to collecting demonstration data or providing full solutions is to encode knowledge into an \textit{ensemble of heuristic solutions} (controllers, planners, previously trained policies, etc.) that address parts of the task. 
Leveraged properly, these heuristics act as \emph{teachers} guiding agent exploration, leading to faster convergence during training and better asymptotic performance. 
Notably, the agent can also learn to outperform its teachers and learn parts of a task for which no teacher offers adequate advice.
While prior work has addressed the problem of policy learning with teachers~\cite{ppr, azar2013regret, zhan2016theoretically, li2018optimal, li2018context}, in this work we discuss a collection of teacher attributes that characterize a single teacher's behavior compared to the optimal policy, as well as the combined behavior of the teachers in the ensemble, and propose a unifying algorithmic framework that efficiently exploits the set of teachers during the agent's training regardless of the teachers' attributes.

Specifically, we assume that the set of teachers can have the following attributes: (i) \textit{partial} -- teachers may recommend useful actions only in a part of the state space, (ii) \textit{insufficient} -- the agent cannot solve the task by only using actions suggested by the set of teachers, and (iii) \textit{contradictory} -- performing actions from different teachers in succession would hinder task completion (see Fig.~\ref{fig:teachers}).

Policy learning in this setting is not straightforward - the agent needs to collect experience by deciding between its own \textit{on-policy} actions or exploiting teacher advice, but also be able to leverage this heterogeneously generated experience to learn to solve the task. Learning from experience collected by another agent advocates using off-policy reinforcement learning~\cite{sutton1998introduction}, where a \emph{behavioral policy} is used to gather the training experience. The main challenge in the teacher ensemble setting is that the utility of each teacher can vary in different state space regions, while the utility of the agent varies over time as learning progresses. This adds to the known complications of training an agent from off-policy experience, which may be unstable and result in poor performance -- especially for Deep RL algorithms -- due to the compounding effects of function approximation and bootstrapping~\cite{fujimoto2018off, bhatt2019crossnorm, achiam2019towards}. 

To address these challenges, we present \textbf{\methodnameFull (\methodname)}, a policy learning framework to leverage advice from multiple teachers where individual teachers might only offer useful advice in certain states, offer advice that contradicts other teacher suggestions, and where the agent might need to learn behaviors from scratch in states where no teacher offers useful advice. 
\methodname~is a novel behavioral policy mechanism for continuous control domains that is robust to low-quality teacher ensembles through quantifying uncertainty over the value of actions based on a probabilistic posterior critic network~\cite{henderson2017bayesian} and using the uncertainty to choose between different action proposals. It is agnostic to the source of action proposals, since it models a single action-value function (the critic) that is not explicitly conditioned on any teacher, and can therefore easily scale to settings with larger number of teachers and even to settings with changing teacher sets. By contrast, most prior work models each teacher independently in some form. 

\methodname~also includes a \textit{commitment} mechanism that allows the behavioral policy to commit to actions from a single policy for longer periods of time; the posterior critic is used to estimate the probability of value improvement from switching the policy choice, and a switch is only enacted under high confidence. This allows the behavioral policy to collect meaningful experience from partial teachers without being hindered by contradictory teachers. Lastly, we also mitigate the instability of off-policy learning by basing the behavioral policy on the critic, and training the critic network using a target value that is generated by the behavioral policy. This helps reduce the off-policyness of the experience being used to train the critic by coupling the critic and the behavioral policy together.

\noindent \textbf{Summary of Contributions}. 
(a) we present a collection of teacher attributes that comprehensively characterizes the quality of a set of teachers for guiding agent training; (b) we propose \methodnameFull (\methodname), a policy learning framework to leverage advice from multiple teachers that is robust to low-quality teacher sets where individual teachers might only offer useful advice in certain states, offer advice that contradicts other teacher suggestions, and where the agent might need to learn behaviors from scratch in states where no teacher offers useful advice; (c) Our experiments demonstrate that \methodname is able to leverage such teacher ensembles to solve multi-step tasks while significantly improving sample efficiency over baselines; and (d) we also show that \methodname is not only able to generate policies using low-quality teacher sets but also surpasses baselines when using higher quality teachers, hence providing a unified algorithmic solution to a broad range of teacher attributes.

%% file: 2-rw.tex
\vspace{-3pt}
\section{Related Work}
\vspace{-3pt}


\textbf{Imitation and Reinforcement Learning from Offline Demonstrations:}
Imitation Learning has been used to train a policy from a set of demonstrations, which are offline samples collected from an expert that is assumed to be optimal. Recently, several works~\cite{dqfd, ddpgfd, nair2018overcoming, nac, wang2018interactive, wang2017improving} have applied off-policy deep RL algorithms to train an agent from a set of suboptimal demonstrations. 
However, off-policy deep RL can be particularly unstable due to the compounding effects of function approximation and bootstrapping leading to inaccurate extrapolation~\cite{ross2014reinforcement,fujimoto2018off,  bhatt2019crossnorm, achiam2019towards}. 
In our setting, the behavioral policy is able to follow any teacher in the policy set. Consequently, training the agent on this off-policy data requires addressing the issue of instability. 
\textbf{Imitation and Reinforcement Learning from Online Teachers:}
Prior work has also investigated the use of one or more teacher policies for training an agent. 
Several works~\cite{ross2011reduction, ross2014reinforcement, sun2017deeply} focus on the setting where an agent must learn to imitate a teacher, who is available as an oracle during training. These methods assume that the expert is optimal and can solve the entire task. 
By contrast, \cite{rosenstein2004reinforcement, johannink2018residual, silver2018residual} assume access to a stable controller or prior solution and use reinforcement learning to learn a residual to correct for suboptimal behavior. While these approaches relax the need for a teacher that is optimal, they cannot be applied to settings with more than one teacher. 

Several prior works \cite{ppr, azar2013regret, zhan2016theoretically, li2018optimal} present methods compatible with several teachers, but they assume that one teacher should be selected over the others for the entire episode and formulate the problem as regret minimization.
This is unsuitable when considering teachers that might offer only partial solutions to the task, since the final solution might require knowledge from several teachers that excel in different parts of the state space. 
Li et al.~\cite{li2018context} relaxes this assumption, but does so by treating the teachers as options~\cite{sutton1999between}, which the agent must rely on at test-time as well as train-time, and furthermore, their method does not support continuous control.
The method presented in Xie et al.~\cite{wheels} trains a DQN~\cite{mnih2015human} behavioral policy to select between an expert and a DDPG~\cite{lillicrap2015continuous} learner policy. This method can be extended to settings with multiple experts, although this was not part of the original manuscript. We developed a DQN-based behavioral policy for multiple agents as a baseline in our experiments and show that~\methodname outperforms this method.


\begin{table}[!h]
\vspace{-15pt}
\centering
\caption{~\methodname can be used in RL settings with both function approximation, and continuous control, and is able to leverage a wide variety of teacher sets. The symbol $\odot$ indicates that the algorithm could potentially be applied to this setting, but this has not been demonstrated empirically.}
\label{tab:comparison-table}
\resizebox{0.95\textwidth}{!}{
\centering
\begin{tabular}{l|c|c|c|c|c|c}
\toprule
             & \multicolumn{1}{c|}{\begin{tabular}[c]{@{}c@{}}Deep \\ RL\end{tabular}} & 
             \multicolumn{1}{c|}{\begin{tabular}[c]{@{}c@{}}Continuous \\ control\end{tabular}} & 
             \multicolumn{1}{c|}{\begin{tabular}[c]{@{}c@{}}No teachers\\ at test time\end{tabular}} & 
             \multicolumn{1}{c|}{\begin{tabular}[c]{@{}c@{}}Multiple \\ teachers\end{tabular}} & 
             \multicolumn{1}{c|}{\begin{tabular}[c]{@{}c@{}}Partial \\ teachers\end{tabular}} & 
             \multicolumn{1}{c}{\begin{tabular}[c]{@{}c@{}}Contradictory \\ Teachers\end{tabular}} \\ \midrule 
PRQL/OpsTL \cite{ppr,li2018optimal} & $\odot$ &   & \checkmark  &  & \checkmark  & $\odot$ \\ \hline
DQfD/DDPGfD \cite{dqfd,ddpgfd}      & \checkmark   & \checkmark   & \checkmark  &  $\odot$ &    &  $\odot$ \\ \hline
CAPS  \cite{li2018context}          &  &  & & \checkmark   & \checkmark & $\odot$ \\ \hline
Residuals \cite{johannink2018residual,silver2018residual}   & \checkmark  & \checkmark   &   &  &   &  $\odot$ \\ \hline
Wheels \cite{wheels}                & \checkmark  & \checkmark  & \checkmark &  $\odot$  & $\odot$  &  $\odot$ \\ \hline
\methodname (ours)                      & \checkmark  & \checkmark & \checkmark & \checkmark  & \checkmark  & \checkmark  \\   
\bottomrule 
\end{tabular}
}
\vspace{-15pt}
\end{table}

%% file: 3-probl.tex
\section{Problem Statement and Notation}
\label{s:baps}

Let $\mathcal{M} = (\mathcal{S}, \mathcal{A}, R, \mathcal{T}, \gamma, \rho_0)$ denote a discrete-time Markov Decision Process (MDP), where $\mathcal{S}$ is the state space, $\mathcal{A}$ is the action space, $R(s, a)$ is the reward function, $\mathcal{T}(s'|s, a)$ defines the transition dynamics, $\gamma \in [0, 1)$ is the discount factor, and $\rho_0$ is the start state distribution from which a start state $s_0 \sim \rho_0(\cdot)$ is sampled in every episode. 
At every step, the agent observes the state $s_t$, chooses an action $a_t \sim \pi_{\theta}(\cdot | s_t)$ sampled from the policy $\pi$ parametrized by $\theta$, and observes $s_{t+1} \sim \mathcal{T}(\cdot|s_t,a_t)$ and a reward $r_t = R(s_t, a_t)$. 
The goal in reinforcement learning is to learn an \textit{agent} policy $\pi_{\theta}$ that maximizes the expected discounted sum of rewards $\mathbb{E}_{\Traj} [\sum_{t=0}^{\infty} \gamma^t R(s_t, a_t)]$ from experience acquired during exploration, where $\Traj = (s_0, a_0, s_1, \ldots, a_{T-1}, s_T)$ is the trajectory produced by the policy.

In off-policy RL algorithms we consider two independent policies: (1) \textit{a behavioral policy}, $\pi_b$, that decides on actions to collect experience, and 2) \textit{an agent policy}, $\pi_{\theta}$, that consumes the experience and is trained to accomplish the original task. Both policies are part of the training process but only the agent policy, $\pi_{\theta}$, is evaluated at test time. We note that since we seek to train an agent using experience generated by different teacher policies, our algorithm must necessarily be off-policy.

\noindent \textbf{Policy Learning with Teacher Sets:}
We consider the policy advice setting in which the behavioral policy $\pi_b$ can receive action suggestions from a set of teacher policies $\Pi = \{\pi_1, \pi_2, \ldots, \pi_N\}$ that are available during training but not at test time~\cite{azar2013regret}. 
The problem is then to specify a behavioral policy $\pi_b$ that efficiently leverages the advice of a set $\Pi$ of teacher policies to generate experience to train an agent policy $\pi_\theta$ with the goal of achieving good test-time performance in minimal train-time environment interactions. 

\textbf{Teacher Attributes:} To formalize the teacher attributes let us first consider an infinite-horizon MDPs with \textit{deterministic} transition function $s' = \mathcal{T}(s, a)$, a set of absorbing goal states $\Goals\subset\States$, and a sparse reward function $r_\Goals(s)=\mathbb{1}({s\in\Goals})$. A generalization of the teacher attributes to MDPs with stochastic dynamics and policies can be found in Appendix~\ref{app:expandeddefs}. 
For a given deterministic policy $\pi$ and a goal set $\Goals$, the state-value function at the initial state is $V_\Goals^\pi(s_0) = r_\Goals(s_0) + \sum_{t=1}^{\infty} \gamma^t r_\Goals(s_t)$. 
For this class of MDPs, the state-value function of deterministic policy $\pi$ has a higher value at the state $s$ than at another state $s'$ if and only if $\pi$ can reach a goal in $\Goals$ in fewer timesteps from $s$ than from $s'$ (see Proposition~\ref{app:proposition} and associated proof in Appendix~\ref{app:expandeddefs}). In this context, we define three teacher attributes: \textit{partial}, \textit{sufficient}, and \textit{contradictory}.

\textbf{Definition 3.1} (Partial Teachers) Let $\pi^*$ denote the optimal policy and $V_\Goals^*(s)$ denote the optimal state value function, with respect to a fixed set of goals $\Goals$. Then, a teacher policy is partial if in some non-empty strict subset $ \exists S' \subset \mathcal{S}$, for all states $s \in S'$, we have that $V_\Goals^*(\mathcal{T}(s, a)) > V_\Goals^*(s)$.
 
By Proposition~\ref{app:proposition}, it follows that in some region of the state space, a partial teacher transitions from one state to another state such that an optimal policy would reach the goal faster from the new state. Thus, in that region, the partial teacher offers useful advice. 

\textbf{Definition 3.2} (Sufficient Teachers and Teacher Sets) A teacher policy $\pi$ is sufficient if it has non-zero value for some start state: $\exists s_0 \sim \rho_0(\cdot)$ such that $V^\pi_\Goals(s_0) > 0$. A teacher set $\Pi = \{\pi_1, \pi_2, \dots, \pi_N\}$ is sufficient if there exists a mixture policy that is sufficient (a mixture policy is one that chooses $\pi_{\Pi}(s) \in \{\pi_1(s), \pi_2(s), \dots, \pi_N(s)\}$).  

In other words, by Proposition~\ref{app:proposition}, a teacher is sufficient if it can reach a goal state from some start state and a teacher set is sufficient if some combination of teachers in the set can be used to reach a goal state from some start state. Conversely, a teacher or teacher set that is not sufficient is said to be \textit{insufficient}. The ability to learn in settings with insufficient teachers is an important property for our method, since it requires learning parts of a task that teachers are unable to solve.

\textbf{Definition 3.3} (Contradictory Teachers) A teacher policy $\pi_2$ is contradictory to teacher policy $\pi_1$ and a goal set $\Goals$ if there exists a state $s \in \mathcal{S}$ where following the advice of $\pi_2$ causes $\pi_1$ to take more timesteps to reach a goal in the goal set. Equivalently, following $\pi_2$ in $s$ leads to a state with lower value for $\pi_1$:  $V_\Goals^{\pi_1}(\mathcal{T}(s, \pi_2(s)) < V_\Goals^{\pi_1}(s)$.

In other words, the policy $\pi_2$ hinders the progress of $\pi_1$ in some portion of the state space with respect to the set of goals $\Goals$. As previously established, our goal is to define a behavioral policy $\pi_b$ that benefits from teacher sets that may or may not be sufficient or contain contradictory teachers.

There are multiple methods to maximize the RL objective based on the policy gradient theorem~\cite{sutton1998introduction}. 
One family of solutions is \textit{actor-critic} methods~\cite{lillicrap2015continuous,konda2000actor,mnih2016asynchronous,haarnoja2018soft,fujimoto2018addressing}. Herein, we build \methodname with Bayesian DDPG~\cite{lillicrap2015continuous, henderson2017bayesian}, a popular actor-critic algorithm for continuous action spaces. The agent policy $\pi_{\theta}$ in \methodname is the actor in the DDPG architecture. Conceptually, \methodname can leverage any off-policy actor-critic algorithm, such as, DDPG~\cite{lillicrap2015continuous}, SAC~\cite{haarnoja2018soft}, or TD3~\cite{fujimoto2018addressing}, among others.

\noindent \textbf{DDPG.} DDPG maintains a critic network $Q_{\phi}(s, a)$ and a deterministic actor network $\pi_{\theta}(s)$ (the agent policy), parametrized by $\phi$ and $\theta$ respectively. A behavioral policy $\pi_b$ (usually the same as the agent policy, $\pi_{\theta}$, with an additional exploration noise) is used to select actions that are executed in the environment, and state transitions are stored in a replay buffer $\mathcal{B}$. DDPG alternates between collecting experience and sampling the buffer to train the policy $\pi_{\theta}$ and the critic $Q_{\phi}$.
The critic is trained via the Bellman residual loss $\mathcal{L}_{\text{critic}} = (r + \gamma Q_{\phi'}(s', \pi_{\theta'}(s')) - Q_{\phi}(s, a))^2$ and the actor is trained with a deterministic policy gradient update to choose actions that maximize the critic $\mathcal{L}_{\text{actor}} = -Q_{\phi}(s, \pi_{\theta}(s))$ where $\phi'$ and $\theta'$ denote the use of target critic and actor networks.


The stability and performance of DDPG varies strongly between tasks. 
To alleviate these problems, Henderson~\etal~\cite{henderson2017bayesian} introduced Bayesian DDPG, a Bayesian Policy Gradient method that extends DDPG by estimating a posterior value function for the critic.
The posterior is obtained based on Bayesian dropout~\cite{gal2016dropout} with an $\alpha$-divergence loss. \methodname trains a Bayesian critic and actor in a similar fashion - see Appendix~\ref{app:bddpgtraining} for details.

%% file: 4-method.tex
\vspace{-5pt}
\section{\methodnameFull (\methodname)}
\label{s:method}
\vspace{-3pt}
In this section we introduce \methodname, an approach to leverage policy advice from a set of unconstrained teachers $\Pi$ in actor-critic RL algorithms as illustrated in Fig.~\ref{fig:teachers}(b) and Algorithm~\ref{alg:choice}. \methodname addresses four key challenges with regards to how to implement an efficient behavioral policy.
\begin{enumerate}[
    topsep=0pt,
    noitemsep,
    partopsep=0.5ex,
    parsep=0.5ex,
    leftmargin=*,
    itemindent=2.5ex
    ]
\item How to \textit{evaluate the quality of the advice} from any given teacher in each state for a continuous state and action space? \methodname is based on a novel \emph{critic-guided behavioral policy} that evaluates both the advice from the teachers as well as the actions of the learner. 
\item How to \textit{balance exploitation and exploration} in the behavioral policy? \methodname uses \emph{Thompson sampling} on the posterior over expected action returns provided by a Bayesian critic to help the behavioral policy to select which advice to follow.
\item How to deal with \textit{contradictory teachers}? \methodname implements a temporal \textit{commitment} method based on the posterior from the Bayesian critic that executes actions from the same policy until the confidence in return improvement from switching to another policy is significant.
\item How to \textit{alleviate extrapolation errors in the agent} arising from large differences between the behavioral and the agent policy, the ``large off-policy-ness'' problem? \methodname introduces a \emph{behavioral target} into DDPG's policy gradient update, such that the critic is optimized with the target Q-value of the behavioral policy rather than the agent policy.
\end{enumerate}

\noindent\textbf{1. Critic-Guided Behavioral Policy. }
As introduced in Sec.~\ref{s:baps}, the behavioral policy, $\pi_{b}$ collects experience in the environment during training and is typically the output of the actor network  $\pi_{\theta}$ with added noise. However, when teachers are available, the behavioral policy should take their advice into consideration. To leverage teacher advice for exploration, we propose to use the critic to implement $\pi_b$ in \methodname as follows. Given state $s$, the agent policy $\pi_{\theta}$ and each teacher policy $\pi_i \in \Pi$ generate a set of action proposals $\{\pi_{\theta}(s),\pi_1(s),\ldots,\pi_N(s)\}$.
The critic $Q_{\phi}$ evaluates the set of action proposals and selects the most promising one to execute in the environment. This is equivalent to selecting between the teachers and the agent, but notice that this selection mechanism is agnostic to the source of the actions, enabling \methodname to scale to large teacher sets.

\noindent\textbf{2. Thompson Sampling over a Bayesian Critic for Behavioral Policy.}
The behavioral policy needs to balance between exploration of teachers, whose utility in different states is not known at the start of training, and the agent policy, whose utility is non-stationary during learning versus exploitation of teachers that provided highly rewarded advice in the past. Inspired by the similarity between policy selection and the multi-arm bandits problem, we use Thompson sampling, a well-known approach for efficiently balancing exploration and exploitation in the bandit setting~\cite{kaufmann2012thompson}. Thompson sampling is a Bayesian approach for decision making where the uncertainty of each decision is modeled by a posterior reward distribution for each arm. In our multi-step setting, we model the posterior distribution over action-values using a Bayesian dropout critic network similar to Henderson et al.~\cite{henderson2017bayesian}.

Concretely, instead of maintaining a point estimate of $\phi$ and $Q_{\phi}$, we maintain a distribution over weights, and consequently over values by using Bayesian dropout~(Sec.~\ref{s:baps}). To evaluate an action for a given state-action pair, a new dropout mask is sampled at each layer of $Q_{\phi}$, resulting in a set of weights $\hat{\phi}$, and then a forward pass through the network results in a sample $Q_{\hat{\phi}}(s, a)$. We then use critic $Q_{\hat{\phi}}$ to evaluate the set of action proposals $\{a_0,a_1,\ldots,a_N\}$ and selects $a_i = \arg\max_{a_0, a_1, ..., a_N} Q_{\hat{\phi}}(s, a)$ as the action to consider executing. The choice whether to execute this action depends on our commitment mechanism, explained next.

\begin{algorithm*}[!t]
\caption{\methodnameFull (\methodname): Behavioral Policy}
\label{alg:choice}
\begin{algorithmic}[1]
\Require $s_t$, $Q_{\phi}$, $\pi_{\theta}$, $\Pi = \{\pi_1, \dots \pi_N\}$ \Comment{Current observation, Bayesian critic, Actor, Teacher set}
\State $A \leftarrow [\pi_{\theta}(s_t), \pi_1(s_t), \dots \pi_N(s_t)]$ \Comment{Collect action proposals from actor and teachers}
\State $\hat{\phi} \sim q(\phi)$ \Comment{Posterior sampling of critic weights via Bayesian Dropout}
\State $c_t \leftarrow \argmax_{i \in \left\{[0,1,\dots, N\right\}} Q_{\hat{\phi}}(s_t, A_i)$ \Comment{Proposed policy choice via Thompson Sampling}


\State $p_{better} \leftarrow P[Q_{\phi}(s_t, A_{c_t}) > Q_{\phi}(s_t, A_{c_{t-1}})]$ \Comment{Probability of value improvement (Appendix~\ref{app:commitestimate})}


\If{$p_{better}$ < $\beta\psi^{t_c}$}
    \State $c_t \leftarrow c_{t-1}$ \Comment{Retain previous policy choice for low improvement probability}
    \State $t_c \leftarrow t_c+1$
\Else{\quad $t_c \leftarrow 0$} \Comment{Accept proposed policy choice for high improvement probability}
\EndIf
\State \Return $A_{c_t}$ \Comment{Return action based on policy choice}
\end{algorithmic}
\end{algorithm*}


\setlength{\textfloatsep}{2pt}

\noindent\textbf{3. Confidence Based Commitment. }%
In our problem setup, we consider the possibility of contradictory advice from different teachers that hinders task progress.
Therefore, it is crucial to avoid switching excessively between teachers and \emph{commit} to the advice from the same teacher for longer time periods.
This is particularly important at the beginning of the training process when the critic has not yet learned to provide correct evaluations.

To achieve a longer time commitment, we compare the policy selected at this timestep $\pi_i$ via the Thompson Sampling process to the policy selected at the previous timestep, $\pi_j$. We use the posterior critic to estimate the probability for the value of $a_i$ to be larger than the value of $a_j$ (see Appendix~\ref{app:commitestimate}). If the probability of value improvement is larger than a threshold $\beta\psi^{t_c}$, the behavioral policy acts using the new policy, otherwise it acts with the previous policy. The threshold $\beta$ controls the behavioral policy's aversion to switch, and $\psi$ controls the degree of multiplicative decay, to prevent over-commitment and make it easier to switch the policy choice when a policy is selected for several consecutive time steps, as in prior work~\cite{ppr,li2018optimal}.

\noindent\textbf{4. Addressing Extrapolation Error in the Critic. }%
Off-policy learning can be unstable for deep reinforcement learning approaches~\cite{fujimoto2018off, bhatt2019crossnorm, achiam2019towards}. 
We follow Henderson et al.~\cite{henderson2017bayesian} in training the learner policy $\pi_{\theta}$ and the Bayesian critic $Q_{\phi}$ on samples from the experience replay $\mathcal{B}$. However, we further improved the stability of training by modifying the critic target values used for the $\alpha$-divergence Bayesian critic loss. Instead of using $r + \gamma Q_{\phi'}(s', \pi_{\theta'}(s'))$ as the target value for the critic, we opt to use $r + \gamma Q_{\phi'}(s', \pi_{b}(s'))$. 

In other words, the behavioral policy is used to select target values for the critic. We observed that using this modified target value in conjunction with basing the behavioral policy on the critic improved off-policy learning (see Appendix~\ref{app:extraperror}). 
This, along with the behavioral policy summarized in Algorithm~\ref{alg:choice}, forms our method,~\methodname.

%% file: 5-exps.tex
\vspace{-5pt}
\section{Experimental Setup}
\label{s:es}
\vspace{-3pt}
We designed~\methodname~to be able to leverage experience from challenging sets of teachers that do not always provide good advice (see Table~\ref{tab:comparison-table}). In our experiments, we compare \methodname~to other learning algorithms that use experience from teachers with the aforementioned attributes (see Sec.~\ref{s:baps}) for the following three control tasks (see Appendix~\ref{app:tasks} for more task details). We outline the tasks and teacher sets below. For each task, we design a \textit{sufficient} teacher that chooses the appropriate \textit{partial} teacher to query per state, and unless differently specified, we add a Gaussian action perturbation to every teacher's actions during learning so that their behavior is more suboptimal.

\textbf{1. \texttt{Path Following}:} A point agent starts each episode at the origin of a 2D plane and needs to visit the four corners of a square centered at the origin. These corners must be visited in a specific order that is randomly sampled at each episode. The agent applies delta position actions to move. \textit{Teacher Set:} We designed one teacher per corner that, when queried, moves the agent a step of maximum length closer to that corner. Each of these teachers is \textit{partial} since it can only solve part of the task (converging to the specific corner). The four teachers are needed for the teacher set to be \textit{sufficient}. We additionally designed a set of \textit{contradictory} teachers, which we explain in full detail in Appendix~\ref{app:path_following_exps}.

\textbf{2. \texttt{Pick and Place}:} A robot manipulation problem from~\cite{fetch}. The objective is to pick up a cube and place it at a target location on the table surface. The initial position of the object and the robot end effector are randomized at each episode start, but the goal is constant. The agent applies delta position commands to the end effector and can actuate the gripper. \textit{Teacher Set:} We designed two partial teachers for this task, \textit{pick} and \textit{place}. The pick teacher moves directly toward the object and grasps it when close enough. The place agent is implemented to move the grasped cube in a parabolic motion towards the goal location and dropping it on the target location once it is overhead. 

\textbf{3. \texttt{Hook Sweep}:} A robot manipulation problem adapted from~\cite{silver2018residual}. The objective is to actuate a robot arm to move a cube to a particular goal location. The cube is initialized out of reach of the robot arm, and so the robot must use a hook to sweep the cube to the goal. The goal location and initial cube location are randomized such that in some episodes the robot arm must use the hook to sweep the cube closer to its base and in other episodes the robot arm must use the hook to push the cube away from its base to a location far from the robot. \textit{Teacher Set:} We designed four partial teachers for this task, \textit{hook-pick}, \textit{hook-position}, \textit{sweep-in}, and \textit{sweep-out}. The \textit{hook-pick} teacher guides the end-effector to the base of the hook and grasps the hook. The \textit{hook-position} teacher assumes that the hook has been grasped at the handle and attempts to move the end effector into a position where the hook would be in a good location to sweep the cube to the goal. Note that this teacher is agnostic to whether the hook has actually been grasped and tries to position the arm regardless. The \textit{sweep-in} and \textit{sweep-out} teachers move the end effector toward or away from the robot base respectively such that the hook would sweep the cube into the goal, if the robot were holding the hook and the hook had been positioned correctly, relative to the cube.

\begin{figure}[!t]
   \centering
   \vspace{-15pt}
    \includegraphics[width = \textwidth]{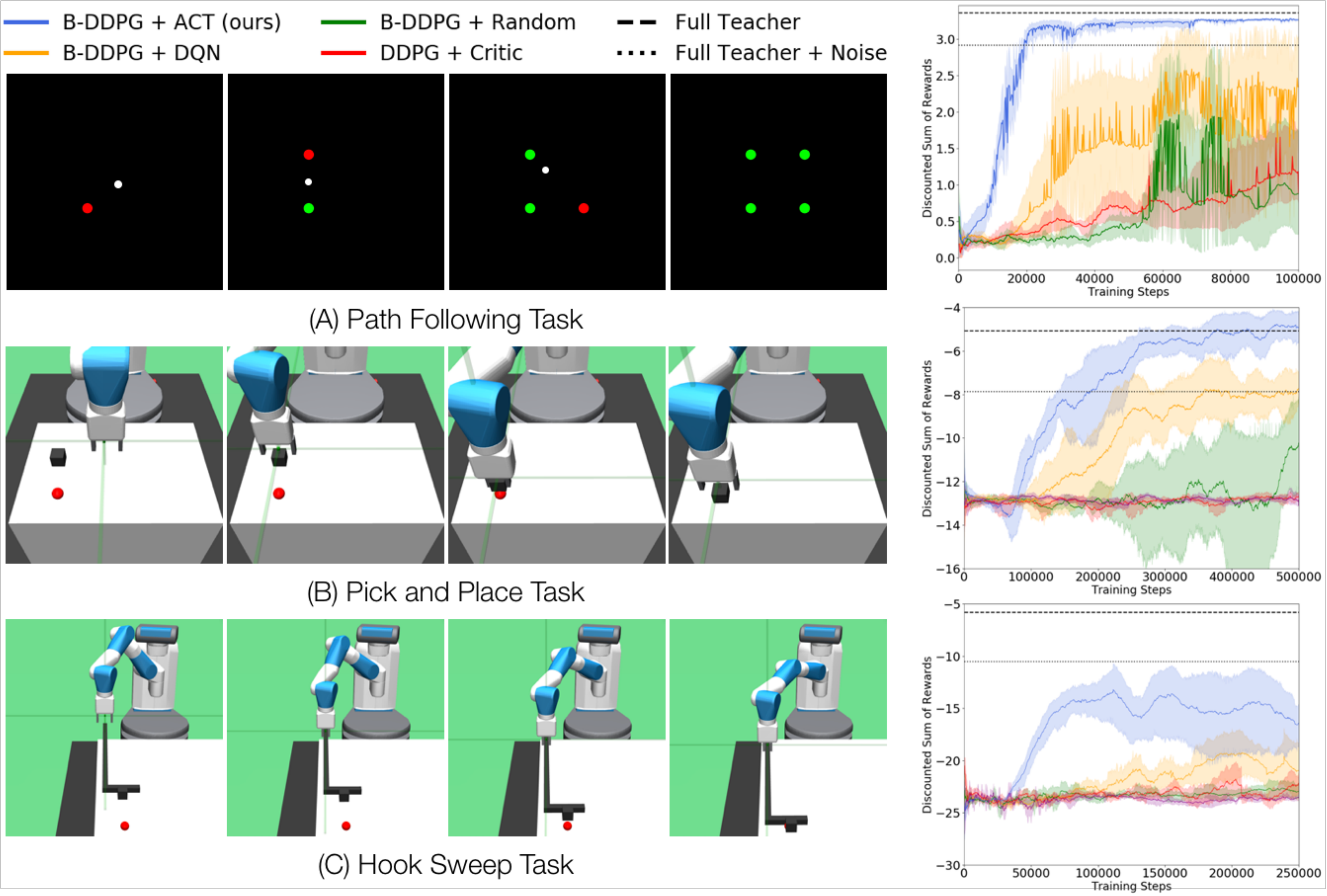}
   \caption{\textbf{Tasks and Performance with \textit{Sufficient Partial} Teacher Set.} The \texttt{Path Following} (top), \texttt{Pick and Place} (middle), and \texttt{Hook Sweep} (bottom) tasks are shown above, during one task completion. The plots (right), showing the average evaluation return of the agent during training, demonstrate that our algorithm exhibits faster and better convergence than other baselines with a \textit{sufficient} teacher set of noisy \textit{partial} solutions.}
   \label{fig:main}
\end{figure}

\noindent \textbf{Baselines.} We compare \methodname against the following set of baselines:
\begin{enumerate}[
    topsep=0pt,
    noitemsep,
    partopsep=0.5ex,
    parsep=0.5ex,
    leftmargin=*,
    itemindent=2.5ex
    ]
\item \textbf{BDDPG:} Vanilla DDPG without teachers, using a Bayesian critic as in \cite{henderson2017bayesian}.
\item \textbf{DDPG + Teachers (Critic):} Train a point estimate of the critic parameters instead of using Bayesian dropout. The behavioral policy still uses the critic to choose a policy to run. 
\item \textbf{BDDPG + Teachers (Random):} BDDPG with a behavioral policy that picks an agent to run uniformly at random.
\item \textbf{BDDPG + Teachers (DQN):}  BDDPG with a behavioral policy that is a Deep Q Network (DQN), trained alongside the agent to select the source policy as in \cite{wheels}.
\end{enumerate}

%% file: 6-results.tex
\begin{figure}[!t]
    \centering
    \includegraphics[width=\textwidth]{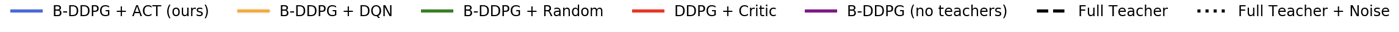}
    \begin{subfigure}[b]{0.245\textwidth}
        \centering
        \includegraphics[width=\textwidth]{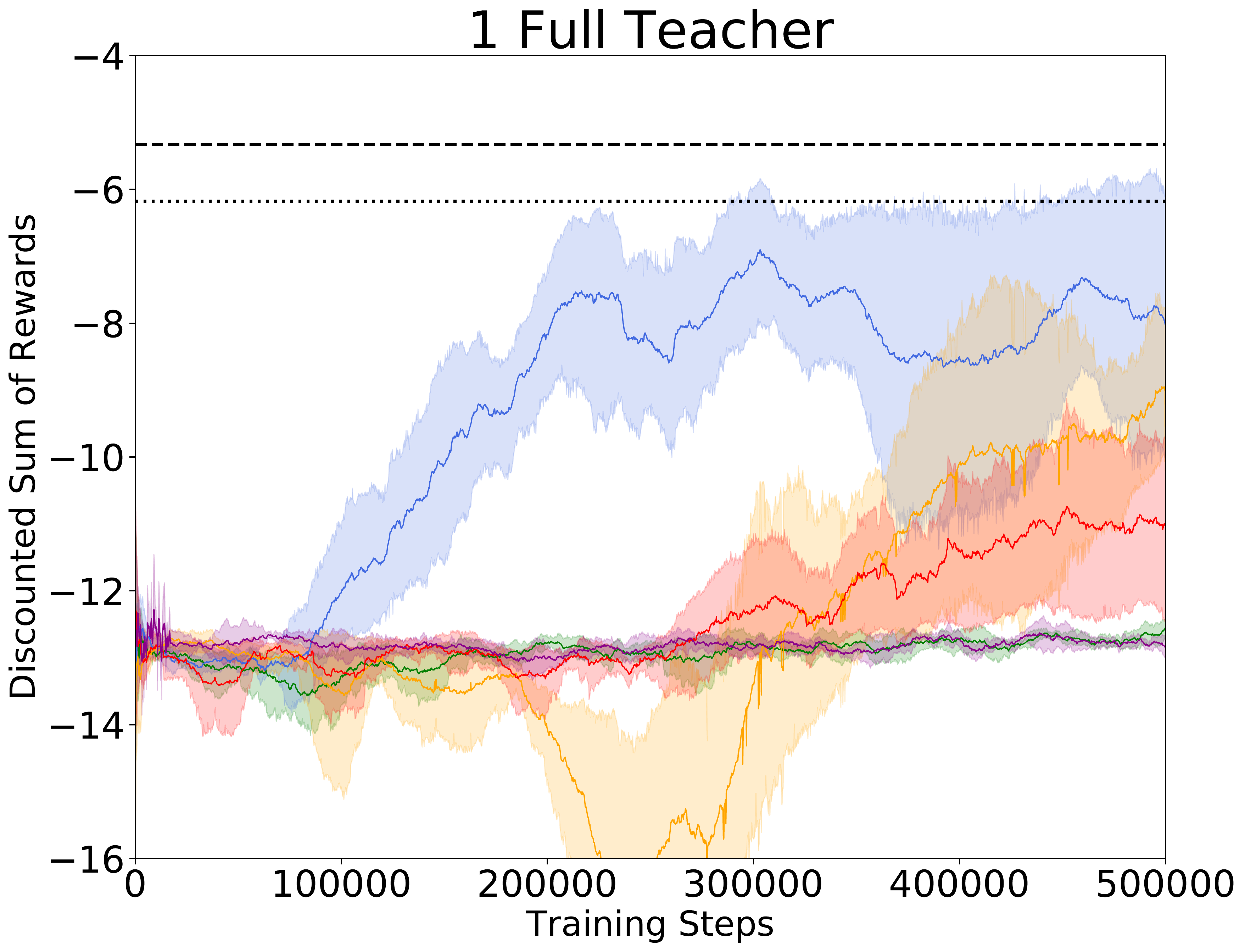}
    \end{subfigure}%
    \hfill
    \begin{subfigure}[b]{0.245\textwidth}
        \centering
        \includegraphics[width=\textwidth]{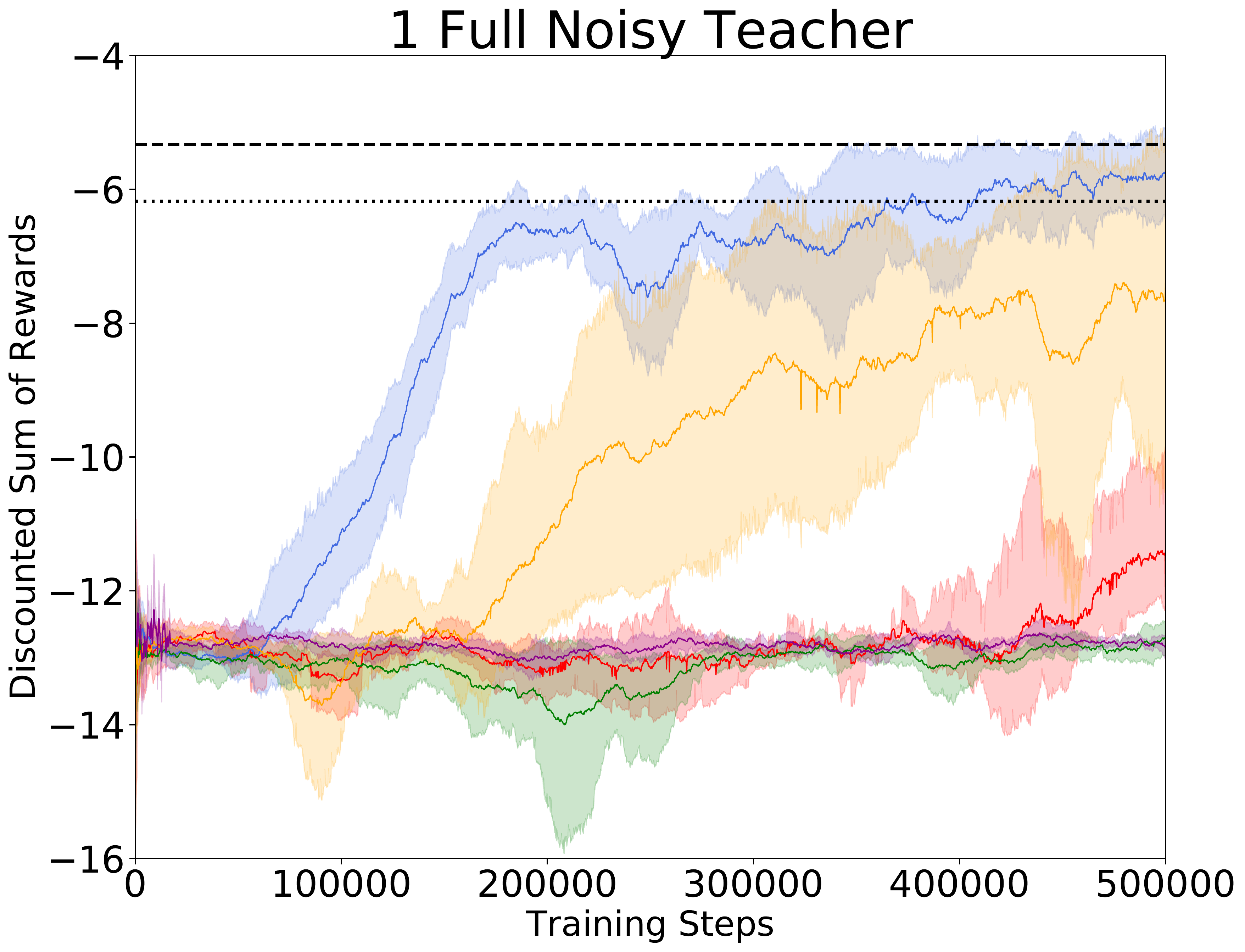}
    \end{subfigure}
    \hfill
    \begin{subfigure}[b]{0.245\textwidth}
        \centering
        \includegraphics[width=\textwidth]{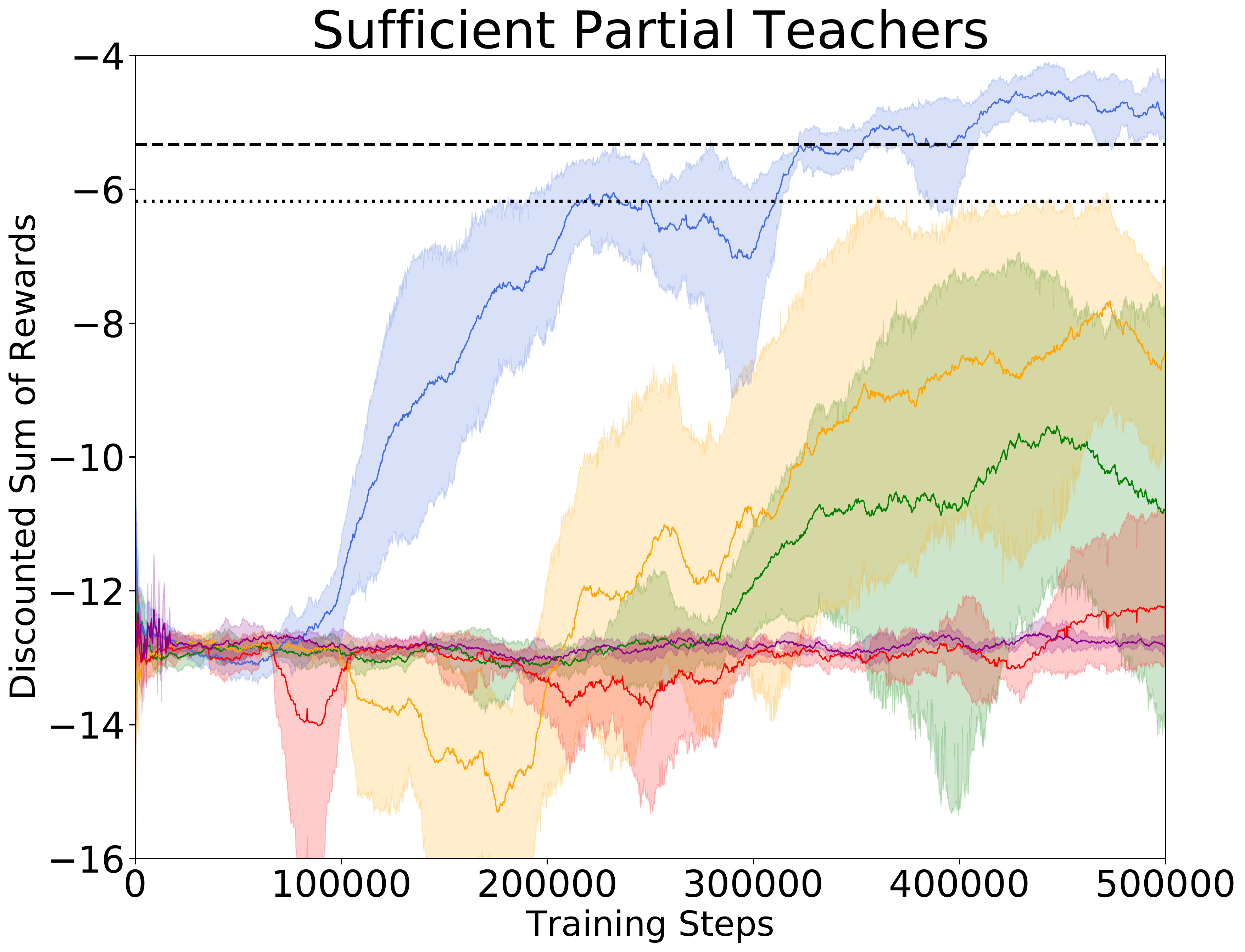}
    \end{subfigure}
    \hfill
    \begin{subfigure}[b]{0.245\textwidth}
        \centering
        \includegraphics[width=\textwidth]{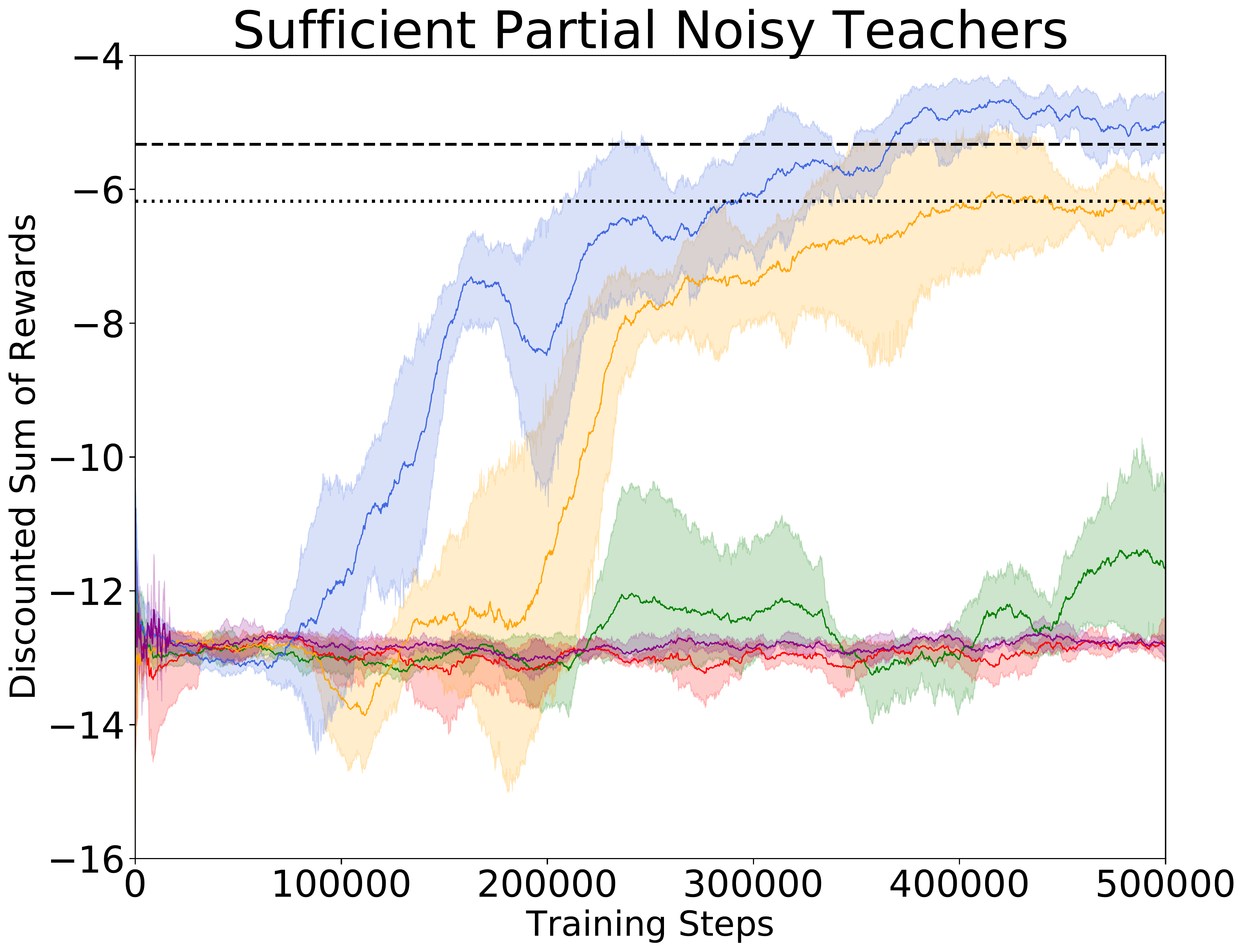}
    \end{subfigure}
   \caption{\textbf{Evaluation with \textit{sufficient} teacher sets:} Test time performance of an agent trained for the \texttt{pick-and-place} task with~\methodname and baselines. Whether the teacher set contains just one \textit{sufficient} teacher (leftmost two graphs) or multiple \textit{partial} teachers (rightmost two graphs), our method significantly outperforms all others in both convergence speed and asymptotic performance. With multiple \textit{partial} teachers, the~\methodname agent even outperforms our hand-coded teacher without noise, despite it being close to optimal.}
  \label{fig:mujoco1}
\end{figure}

\vspace{-5pt}
\section{Results}
\vspace{-5pt}
Our experiments answer the following questions: (1) to what extent does \methodname improve the number of interactions needed by the agent to learn to solve the task by leveraging a set of teachers that are \textit{partial}? (2) can \methodname~still improve sample efficiency when the set of teachers is \textit{insufficient} and parts of the task must be learned from scratch? (3) how sensitive is the performance of \methodname~to the quality and size of the teacher set? 

Fig.~\ref{fig:mujoco1} addresses (1) by showing that \methodname outperforms all baselines in improving both the sample efficiency of learning and the asymptotic performance of the agent. 
We present our results on the \texttt{Pick-and-Place} task and defer results on the \texttt{Path Following}  and \texttt{Hook Sweep} tasks in Appendix~\ref{app:path_following_exps} and Appendix~\ref{app:hook_sweep_exps}, along with ablation analysis (Appendix~\ref{app:ablation}) and experiments on parameter sensitivity (Appendix~\ref{app:sensitivity}). Ablation studies suggest that all 4 components of our approach are needed for it to work well across tasks.

We evaluate \methodname with \textit{sufficient} and \textit{partial} teachers by varying whether the teacher set contains a single \textit{sufficient} teacher or both the pick and place \textit{partial} teachers, and whether these teachers are made worse through the addition of noise. As shown in Fig.~\ref{fig:mujoco2}, our method consistently performs better than all baselines across all teacher sets. Interestingly, both introducing suboptimality into the teacher(s) and using partial teachers, instead of a single full teacher, improve performance for our algorithm as well as the next best performing baseline. We believe this is due to suboptimal teachers providing more varied advice and increasing the diversity of data in the replay buffer, which mitigates extrapolation error.

\begin{figure}[t!]
    \centering
    \includegraphics[width=\textwidth]{figs/legend.png}
    \begin{subfigure}[b]{0.245\textwidth}
        \centering
        \includegraphics[width=\textwidth]{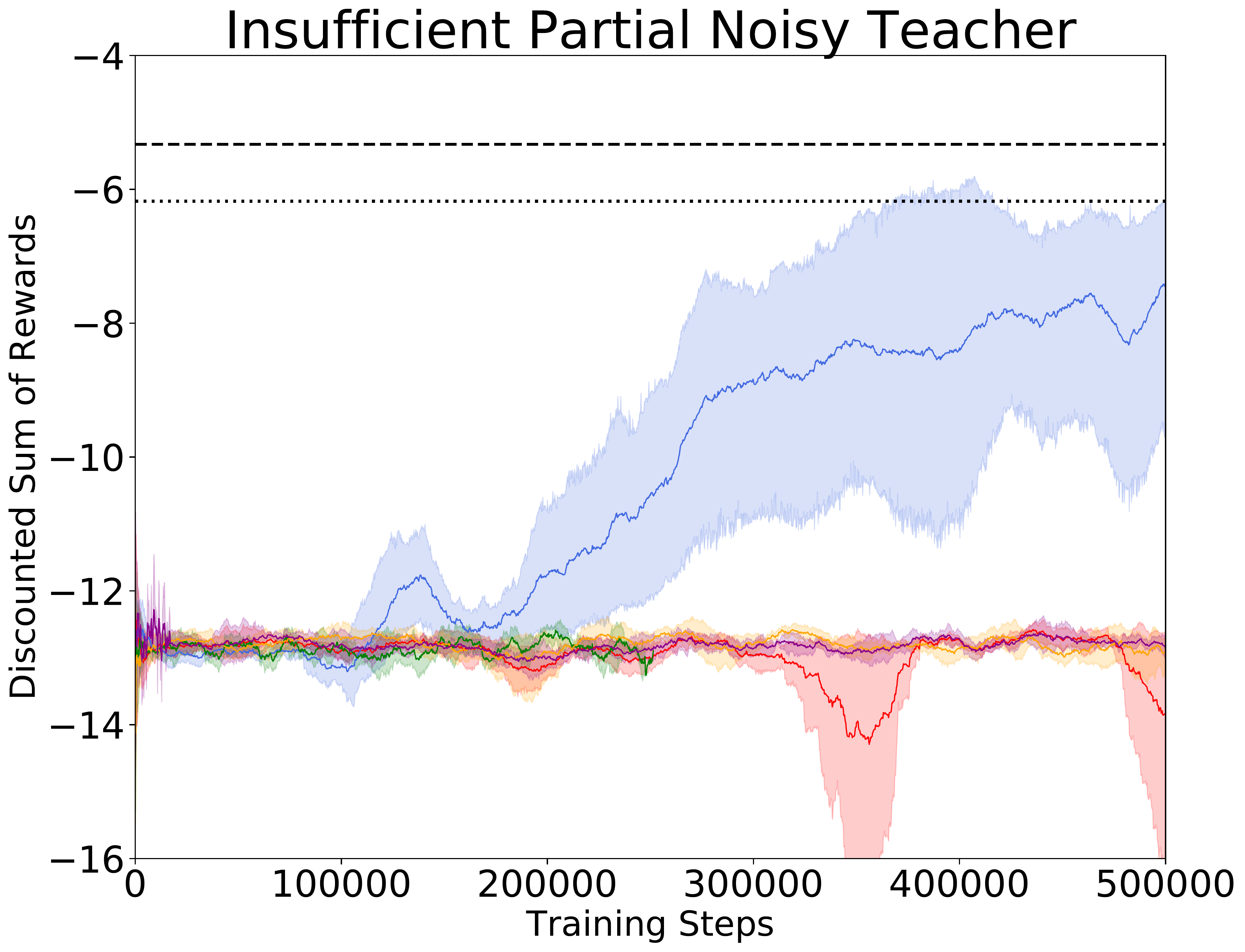}
    \end{subfigure}
    \hfill
    \begin{subfigure}[b]{0.245\textwidth}
        \centering
        \includegraphics[width=\textwidth]{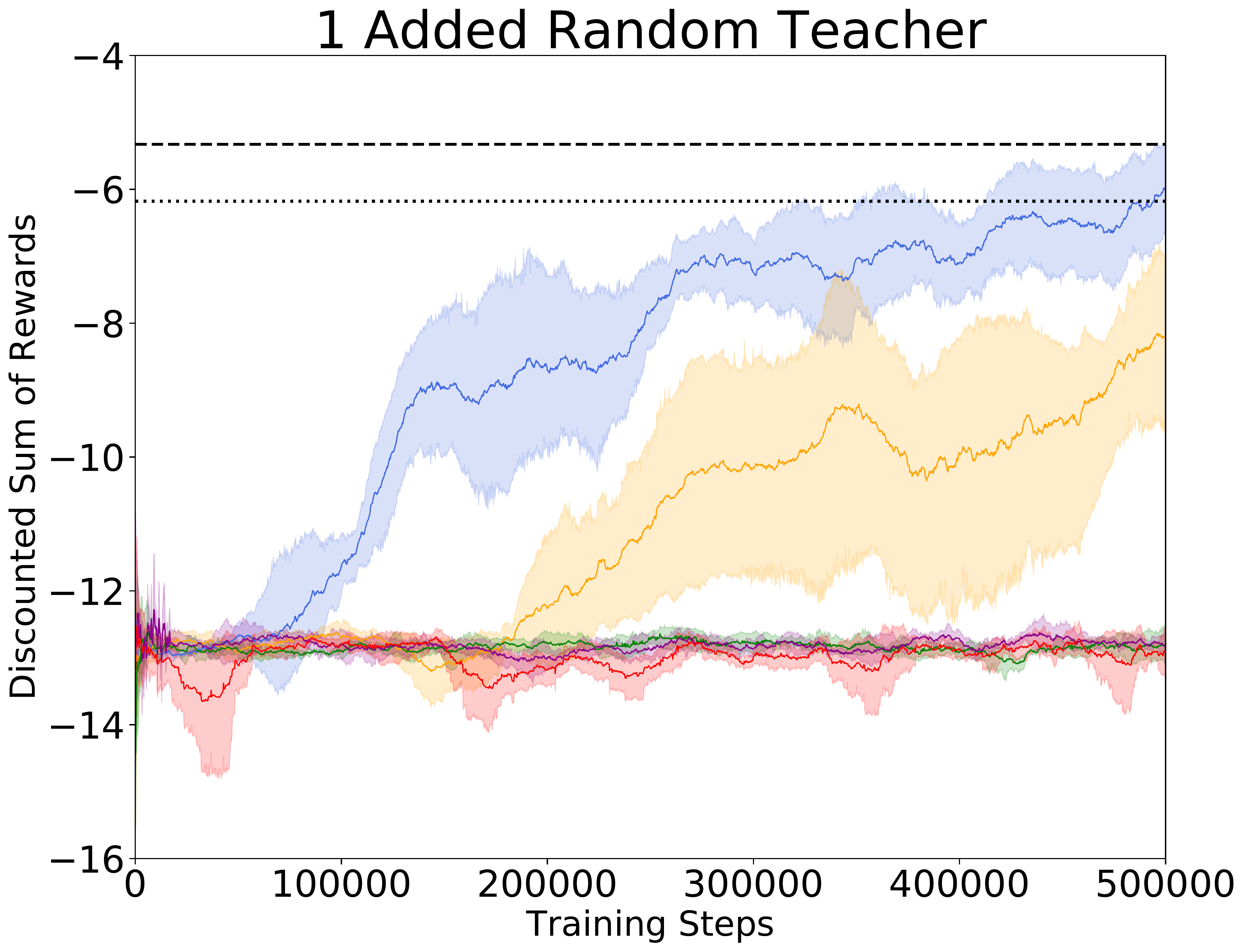}
    \end{subfigure}
    \hfill
    \begin{subfigure}[b]{0.245\textwidth}
        \centering
        \includegraphics[width=\textwidth]{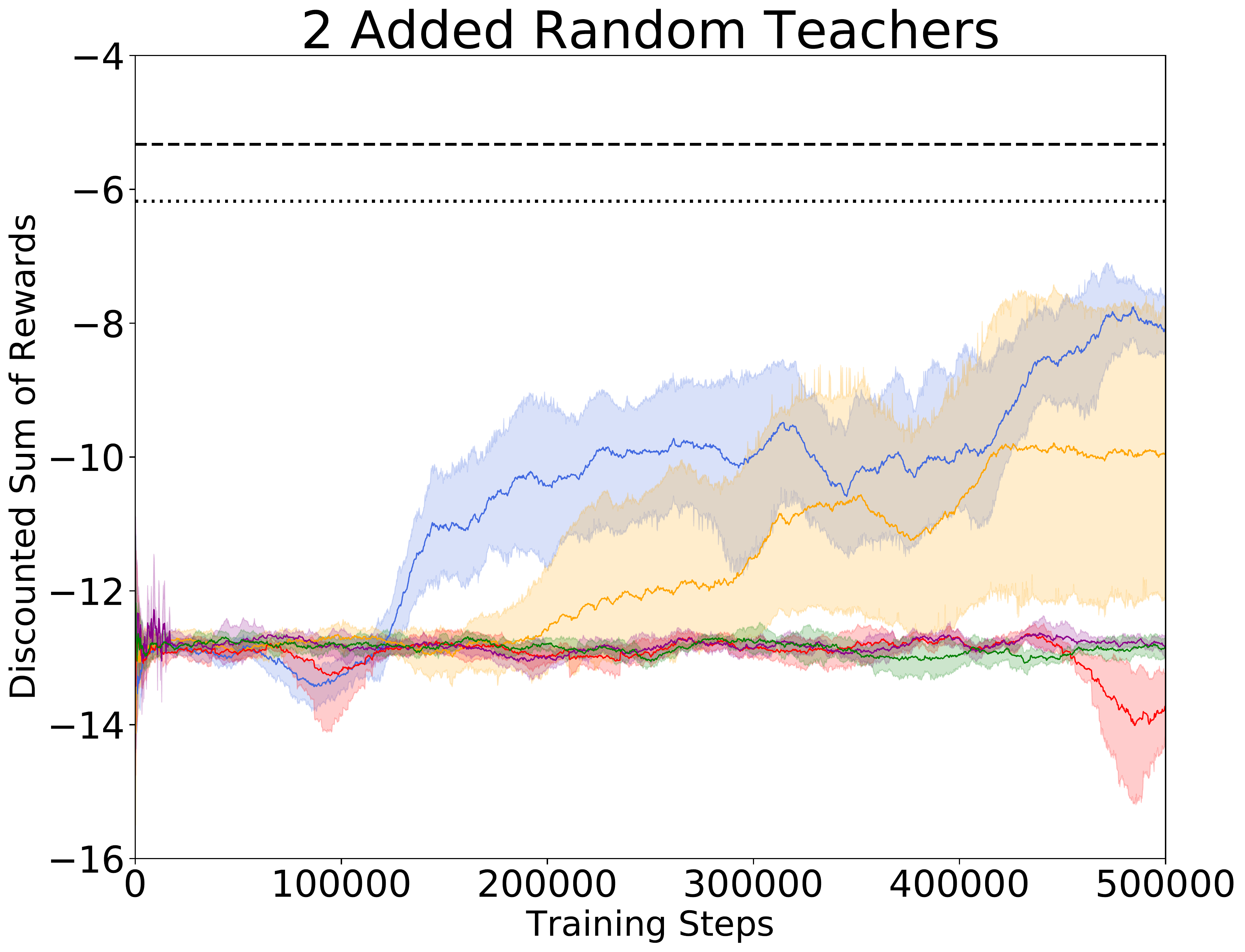}
    \end{subfigure}
    \hfill
    \begin{subfigure}[b]{0.245\textwidth}
        \centering
        \includegraphics[width=\textwidth]{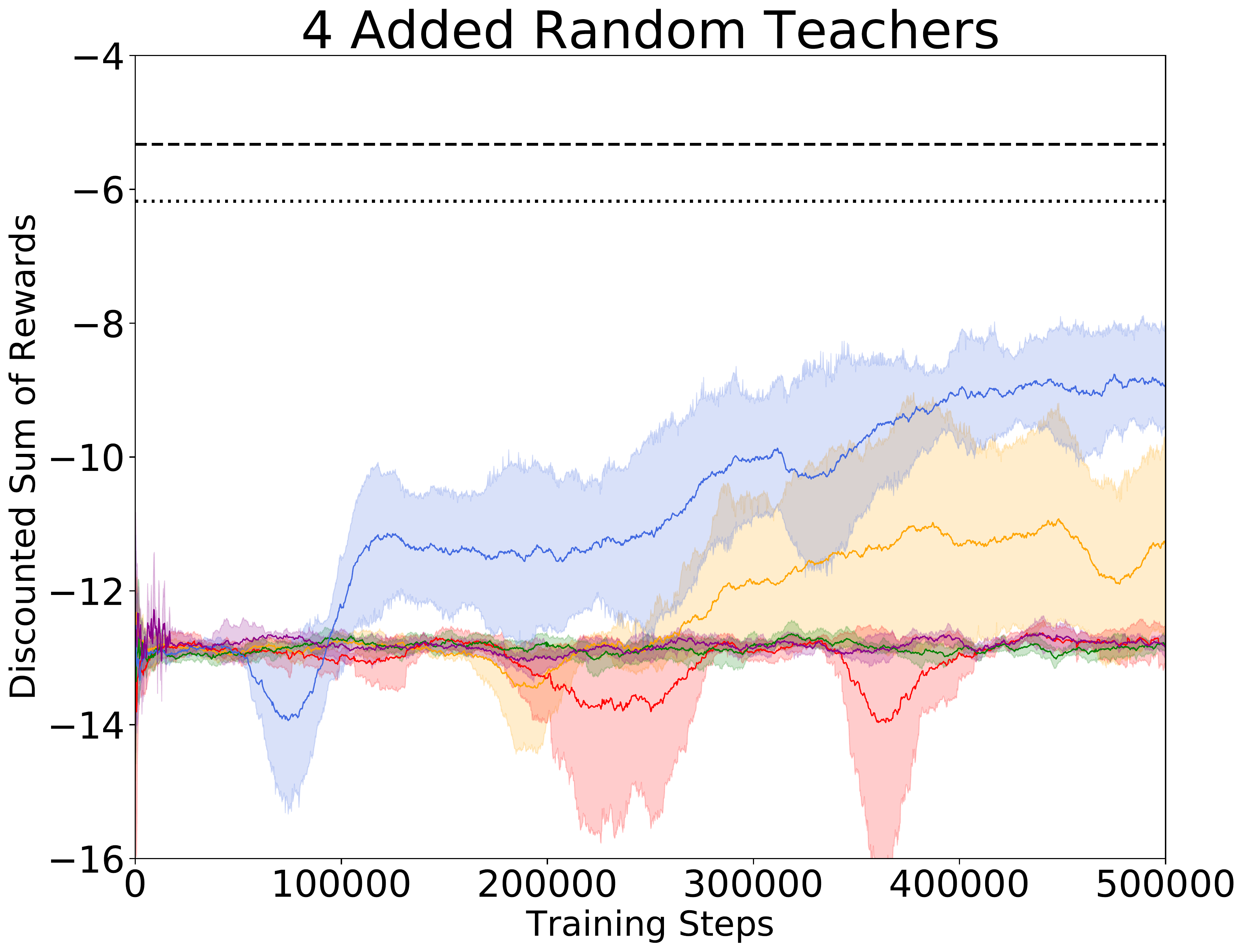}
    \end{subfigure}
   \caption{\textbf{Evaluation with \textit{insufficient} and low-quality teacher sets.} Test time agent performance on the \texttt{pick-and-place} task. Both when utilizing just the suboptimal pick teacher (leftmost graph) or when utilizing both the pick and place teachers along with added teachers which always suggest random actions (right three graphs), our method outperforms all others in both convergence speed and asymptotic performance.}
  \label{fig:mujoco2}
\end{figure}

Next, we evaluate \methodname with an \textit{insufficient} teacher set by providing only the noisy pick teacher and thus requiring the agent to learn the place part of the task on its own. As shown in the \textit{insufficient partial noisy teacher} plot in Figure 4, none of the baselines are able to learn in this condition, whereas our method is able to approach the performance of our noisy full teacher. 

Lastly, we demonstrate \methodname's robustness to learning with miss-specified teachers by training the agent with the two noisy pick and place teachers as well as 1, 2, and 4 teachers that always suggest random actions. As can be seen in  Fig.~\ref{fig:mujoco2}, the agent's learning performance does degrade as we add more random teachers but \methodname consistently outperforms the other baselines and is able to get some benefit from the teacher set.

%% file: 0-appendix.tex
\newpage
\appendix
\part*{Appendix}

\renewcommand\thesection{\Alph{section}}

\counterwithin{figure}{section}
\counterwithin{algorithm}{section}

\section{Problem Formulation}
\label{app:problem}

\subsection{Relationship between Value Function and Reaching Time}
\label{app:propositionproof}

We consider the class of infinite-horizon MDPs (defined in Sec.~\ref{s:baps}) that have a deterministic transition function $s' = \mathcal{T}(s, a)$, a set of absorbing goal states $\Goals\subset\States$, and a sparse reward function $r_\Goals(s)=\mathbb{1}({s\in\Goals})$. Note that for a given deterministic policy $\pi$ and a goal set $\Goals$, we have $$V_\Goals^\pi(s_0) = r_\Goals(s_0) + \sum_{t=1}^{\infty} \gamma^t r_\Goals(s_t) = \mathbb{1}(s_0 \in \Goals) + \sum_{t=1}^{\infty} \gamma^t \mathbb{1}(s_t \in \Goals).$$

\begin{prop}
\label{app:proposition}
In this class of MDPs, a deterministic policy $\pi$ has a higher value for state $s$ than another state $s'$ if and only if $\pi$ can reach a goal in $\Goals$ in fewer timesteps from $s$ than from $s'$.
\end{prop}

\begin{proof}
Consider any goal state $s_g \in \Goals$. Since every goal state is absorbing with reward $1$, we have $V^{\pi}(s_g) = \frac {1} {1 - \gamma}$. Now consider any arbitrary state $s$. Since both the policy $\pi$ and transition function $\mathcal{T}$ are deterministic, an agent acting according to $\pi$ either reaches a goal state $s_g \in \Goals$ in a finite number of timesteps $t_g$ or never reaches a goal state. Therefore, the value function $V^{\pi}(s) = \frac {\gamma^{t_g - 1}} {1 - \gamma}$ if $\pi$ can reach a goal state from $s$, and $0$ otherwise. Now, assume $\pi$ reaches a goal from state $s$ in $t_s$ timesteps and from $s'$ in $t_{s'}$ timesteps. Then states $s$ and $s'$ have value $V^{\pi}(s) = \frac {\gamma^{t_s - 1}} {1 - \gamma}$ and $V^{\pi}(s') = \frac {\gamma^{t_{s'} - 1}} {1 - \gamma}$ respectively. Since the function $f(t) = \frac {\gamma^{t - 1}} {1 - \gamma}$ is monotone decreasing in $t$ for $\gamma \in (0, 1)$, we have that $t_s < t_{s'} \leftrightarrow V^{\pi}(s) > V^{\pi}(s')$, and the proposition holds.
\end{proof}

\subsection{Extension to Stochastic MDPs}
\label{app:expandeddefs}

We now consider infinite-horizon MDPs with stochastic transition dynamics $\mathcal{T}(s' | s, a)$, a set of absorbing goal states $\Goals \subset \States$, and a sparse reward function $r_\Goals(s)=\mathbb{1}({s\in\Goals})$. 

\subsubsection{Interpretation of value function}

The value function of a stochastic policy $\pi$ is
$$V_{\Goals}^{\pi}(s_0) = \mathbb{E}\bigg[\sum_{t=0}^{\infty} \gamma^t \mathbb{1}(s_t \in \Goals)\bigg].$$

Consider a trajectory collected by $\pi$, $\tau = (s_0, a_0, s_1, a_1, \dots)$ where $a_t \sim \pi(\cdot | s_t)$, $s_t \sim \mathcal{T}(\cdot | s_t, a_t)$. Every trajectory either reaches a goal state and stays there so that $\lim_{t \to \infty} s_t = s_g$ for some $s_g \in \Goals$, or never reaches a goal state. Then, all possible trajectories collected by $\pi$ either belong to the set of goal-reaching trajectories $\mathrm{T}_g$ or the complementary set $\mathrm{T}_g^C$. Every trajectory $\tau \in \mathrm{T}_g$ reaches a goal in some number of steps $t_g$ and consequently generates return $\frac {\gamma^{t_g - 1}} {1 - \gamma}$ and every trajectory $\tau \in \mathrm{T}_g^C$ generates $0$ return since no state generates reward.

Now we write the value function as the expected discounted sum of rewards of trajectories $R(\tau)$

\begin{align*} 
V_{\Goals}^{\pi}(s_0) &=  \mathbb{E}_{\tau}[R(\tau)] \\ 
  &= P(\tau \in \mathrm{T}_g)~\mathbb{E}_{\tau}[R(\tau)~|~\tau \in \mathrm{T}_g] + P(\tau \in \mathrm{T}_g^C)~\mathbb{E}_{\tau}[R(\tau)~|~\tau \in \mathrm{T}_g^C] \\
  &= P(\tau \in \mathrm{T}_g)~\mathbb{E}_{t_g}\bigg[\frac {\gamma^{t_g - 1}} {1 - \gamma}~|~t_g < \infty\bigg]
\end{align*}
where the second line follows by the law of iterated expectations and the third line follows by rewriting the first term as an expectation over the reaching time of trajectories that reach a goal and using the fact that the second term vanishes since all returns of trajectories in the complementary set are $0$.

Note that there is a natural decomposition in the value function. It is a product of $P(\tau \in \mathrm{T}_g)$, which is the probability that $\pi$ reaches a goal when it starts in state $s_0$, and a return term that depends on how fast $\pi$ can reach a goal. Consequently, in the stochastic case, the value of a state is dependent on two factors: reaching a goal from $s_0$ with higher certainty and reaching goals quickly, while in the deterministic case, the value of a state only depends on reaching goals quickly. Thus, we can continue to use value functions to extend our teacher attributes to the stochastic case, keeping in mind that the value function now describes a tradeoff between reaching goals reliably, and reaching goals quickly. 

\subsubsection{Teacher Attributes for Stochastic MDPs}

We have justified continuing to use value functions for defining our teacher attributes in the stochastic MDP case. As such, we simply amend our definitions to account for stochasticity.

\textbf{Definition A1.1} (\textit{Partial} Teachers) Let $\pi^*$ denote the optimal policy and $V_\Goals^*(s)$ denote the optimal state value function, with respect to a fixed set of goals $\Goals$. Then, a teacher policy is \textit{partial} if in some non-empty strict subset $ \exists S' \subset \mathcal{S}$, for all states $s \in S'$, we have that $\mathbb{E}_{s'}V_\Goals^*(s') > V_\Goals^*(s)$ for $s' \sim \mathcal{T}(\cdot | s, a)$ and $a \sim \pi(\cdot | s)$.

Here we have just amended the definition to take an expectation over the next state reached when following $\pi$. 

\textbf{Definition A1.2} (Sufficient Teachers and Teacher Sets) A teacher policy $\pi$ is sufficient if it has non-zero value for some start state: $\exists s_0 \sim \rho_0(\cdot)$ such that $V^\pi_\Goals(s_0) > 0$. A teacher set $\Pi = \{\pi_1, \pi_2, \dots, \pi_N\}$ is sufficient if there exists a mixture policy that is sufficient (a mixture policy is one that chooses $\pi_{\Pi}(s) \in \{\pi_1(s), \pi_2(s), \dots, \pi_N(s)\}$).  

This is the same definition as earlier - a non-zero state value can only occur if and only if there is some non-zero probability of reaching the goal. 

\textbf{Definition A1.3} (\textit{Contradictory} Teachers) A teacher policy $\pi_2$ is \textit{contradictory} to teacher policy $\pi_1$ and a goal set $\Goals$ if there exists a state $s \in \mathcal{S}$ where following the advice of $\pi_2$ causes $\pi_1$ to take more timesteps to reach a goal in the goal set. Equivalently, following $\pi_2$ in $s$ leads to a state with lower value for $\pi_1$:  $V_\Goals^{\pi_1}(s') < V_\Goals^{\pi_1}(s)$ for $s' \sim \mathcal{T}(\cdot | s, a)$ and $a \sim \pi_2(\cdot | s)$.

Once again we have just amended the definition to take an expectation over the next state reached when following $\pi_2$.

\section{Additional Details about \methodname}

\subsection{Training the Bayesian Critic and Agent Policy}
\label{app:bddpgtraining}

\begin{algorithm*}[!t]
\caption{Actor-Critic with Teachers (\methodname) Training Loop}
\label{alg:train}
\begin{algorithmic}[1]
\Require $\pi_b$, $Q_{\phi}$, $\pi_{\theta}$, $\Pi$, $\mathcal{B}$ \Comment{\methodname Behavioral Policy, Critic, Actor, Teacher Set, Replay Buffer}


\For{epoch $= 1, \dots, n_{\text{epoch}}$}
    \For{episode $= 1, \dots, n_{\text{episodes}}$}
        \For{$t = 1, \dots, T$}
            \State $a_t \leftarrow \pi_b(s_t~|~Q_{\phi}, \pi_{\theta}, \Pi)$ \Comment{Action from behavioral policy using Algorithm~\ref{alg:choice}}
            \State $\mathcal{B} = \mathcal{B} \cup (s_t, a_t, r_t, s_{t+1}$ \Comment{Execute action and collect transition into replay buffer}
        \EndFor
    \EndFor
    \For{$u = 1, \dots, n_{\text{updates}}$}
        \State $\{(s_i, a_t, r_i, s_{i+1})\}_{i=1}^B \sim \mathcal{B}$ \Comment{Sample mini-batch from replay buffer}
        \State $\phi \leftarrow \phi - \nabla_{\phi}\mathcal{L}_{\text{critic}}(\phi)$ \Comment{Update critic using loss in Equation~\ref{eq:bddpg-critic}}
        \State $\theta \leftarrow \theta - \nabla_{\theta}\mathcal{L}_{\text{actor}}(\theta)$ \Comment{Update actor using loss in Equation~\ref{eq:bddpg-actor}}
    \EndFor
\EndFor
\end{algorithmic}
\end{algorithm*}

Algorithm~\ref{alg:choice} details how the \methodname behavioral policy $\pi_b$ collects samples into the replay buffer $\mathcal{B}$. In Algorithm~\ref{alg:train}, we outline how these samples are used to train the Bayesian critic $Q_{\phi}$ and actor policy $\pi_{\theta}$. Following Henderson et al.~\cite{henderson2017bayesian}, we use an $\alpha$-divergence loss for the critic, but use a behavioral target value, as described in Sec.~\ref{s:method} to mitigate extrapolation error

\begin{equation}
\mathcal{L}_{\text{critic}} = -\frac {1} {\alpha} \log \sum_{k=1}^K \exp\Big(- \frac{\alpha \tau} {2} \big(r + \gamma Q_{\phi'}(s', \pi_b(s')) - Q_{\hat{\phi}_k} (s, a)\big)^2 \Big) + (1 - p_{\text{drop}}) ||\phi||_2^2~~,
\label{eq:bddpg-critic}
\end{equation}
where $K$ is the number of Monte Carlo dropout samples of the critic network weights $\hat{\phi}_k$ and $p_{\text{drop}}$ is the dropout rate. Notice the use of the \methodname behavioral policy $\pi_b$ for computing the target value.

The agent actor loss is simply the normal DDPG actor loss averaged over $K$ Monte Carlo dropout critic weight samples 
\begin{equation}
\mathcal{L}_{\text{actor}} =  -\sum_{k=1}^K Q_{\hat{\phi}_k}(s, \pi_{\theta}(s))~~~.
\label{eq:bddpg-actor}
\end{equation}%



\subsection{Commitment Probability Estimation}
\label{app:commitestimate}

As discussed in Sec.~\ref{s:method}, when considering to switch to a policy $\pi_i$ from a previously used policy $\pi_j$, we use the posterior distributions over the action-values of both policy actions to estimate the probability for the expected return of $a_i$ to be larger than the expected return of $a_j$. We consider the posterior distribution over expected value provided by the Bayesian critic for each action as an independent Gaussian distributed random variable $Z_i = Q_{\phi}(s, a_i)$ and $Z_j = Q_{\phi}(s, a_j)$. To obtain these distributions we take $K$ Monte Carlo samples of our posterior critic for each action by sampling different dropout masks~\cite{gal2016dropout}. We use the MC samples to fit a Gaussian such that
\begin{equation} \label{eq1}
\begin{split}
\mu_{Z_i} &= \frac {1} {K} \sum_{k=1}^K Q_{\hat{\phi}_k}(s, a_i) \\
 \sigma_{Z_i}^2 &= \frac {1} {K} \sum_{k=1}^K (Q_{\hat{\phi}_k}(s, a_i))^2 - \mu_{Z_i}^2 + \frac {1} {\tau}
\end{split}
\end{equation}
, where $\tau$ is the precision of the variance estimate. We can compute the estimate of the probability $P(Z_i > Z_j)$ by leveraging the fact that $Z_i - Z_j \sim \mathcal{N}(\mu_{Z_i} - \mu_{Z_j}, \sigma_{Z_i}^2 + \sigma_{Z_j}^2)$ and computing $P(Z_i - Z_j) > 0$ using the Gaussian cumulative distribution function. Intuitively, this probability estimate corresponds to the belief that the new policy prescribed by the critic will yield higher Q-value than the policy from the current timestep. 

\subsection{Mitigating Extrapolation Error}
\label{app:extraperror}
As discussed in Sec.~\ref{s:method}, a key challenge when training with teachers is avoiding instability when training the critic from off-policy experience. We observed that basing the behavioral policy on the critic of the learner agent, as opposed to a separate model such as the DQN in~\citet{wheels}, significantly helped with this problem, as did using the behavioral target described in Appendix~\ref{app:bddpgtraining}. We hypothesize that since the Bayesian critic is indirectly being used to generate data through the behavioral policy and it is trained on the same distribution of data, the algorithm enjoys similar stability benefits as on-policy methods. We intend to explore this further in future work.


\section{Additional Experiments}

\subsection{\texttt{Path Following} Experiments}
\label{app:path_following_exps}

We replicate the set of experiments presented in Fig.~\ref{fig:mujoco1} and Fig.~\ref{fig:mujoco2} for the \texttt{path-following} task, and also present the performance of the behavioral policies at train time (effectively, the quality of the collected experience the DDPG agent is trained with) in addition to the test time performance of the learned agent. The results are depicted in Fig.~\ref{fig:path_a}. We observe that our proposed method, \methodname, converges faster and to larger reward level at test time than the baselines across the variety of teacher setups as shown in the main text for \texttt{Path Following} task. Furthermore, unlike the baselines it performs at test time than at train time, indicating that it better avoids extrapolation error or over-reliance on the teachers.

\begin{figure}[!t]
   \centering
    \includegraphics[width=\textwidth]{figs/legend.png}
    \begin{subfigure}[b]{0.245\textwidth}
        \centering
        \includegraphics[width=\textwidth]{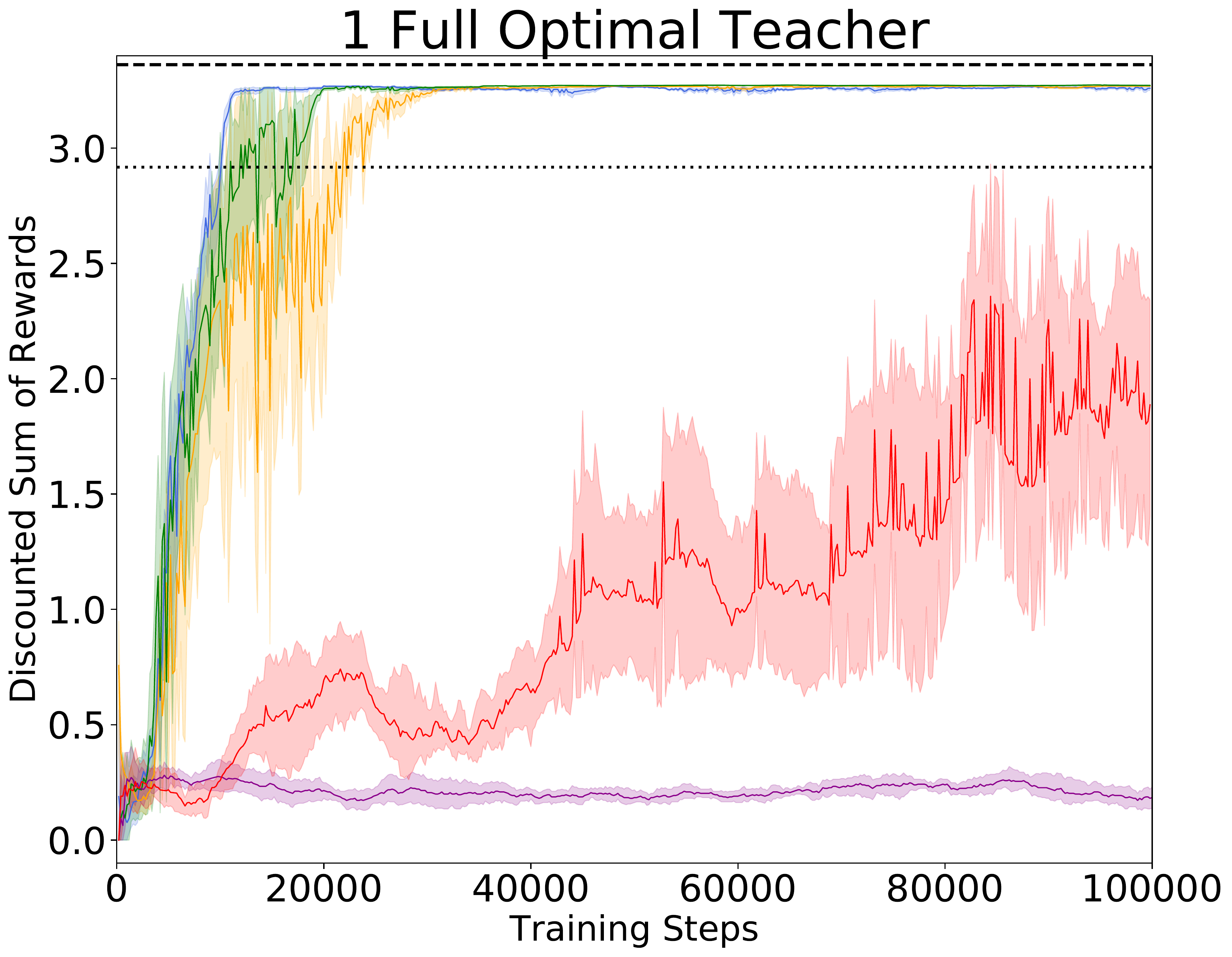}
    \end{subfigure}%
    \hfill
    \begin{subfigure}[b]{0.245\textwidth}
        \centering
        \includegraphics[width=\textwidth]{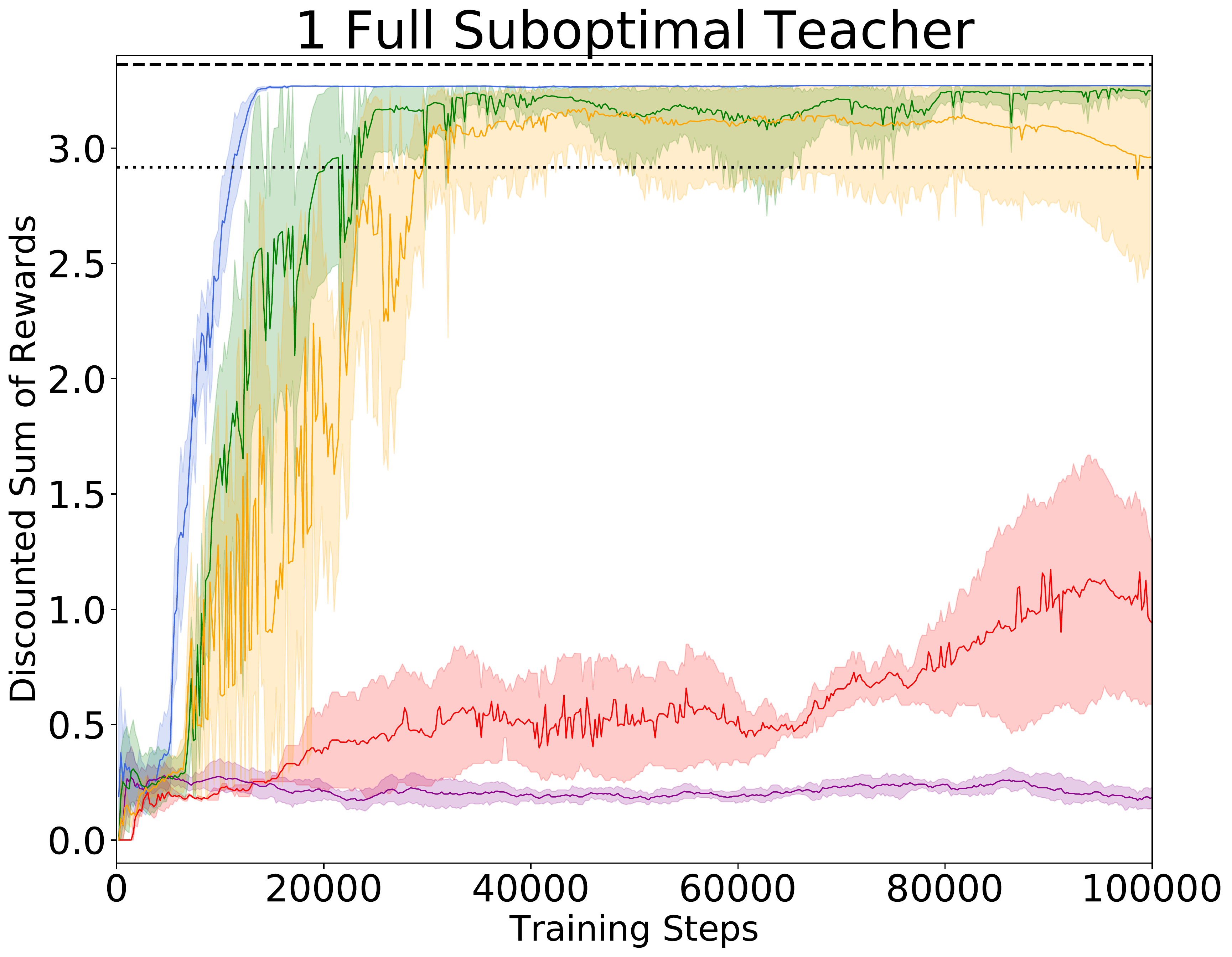}
    \end{subfigure}
    \hfill
    \begin{subfigure}[b]{0.245\textwidth}
        \centering
        \includegraphics[width=\textwidth]{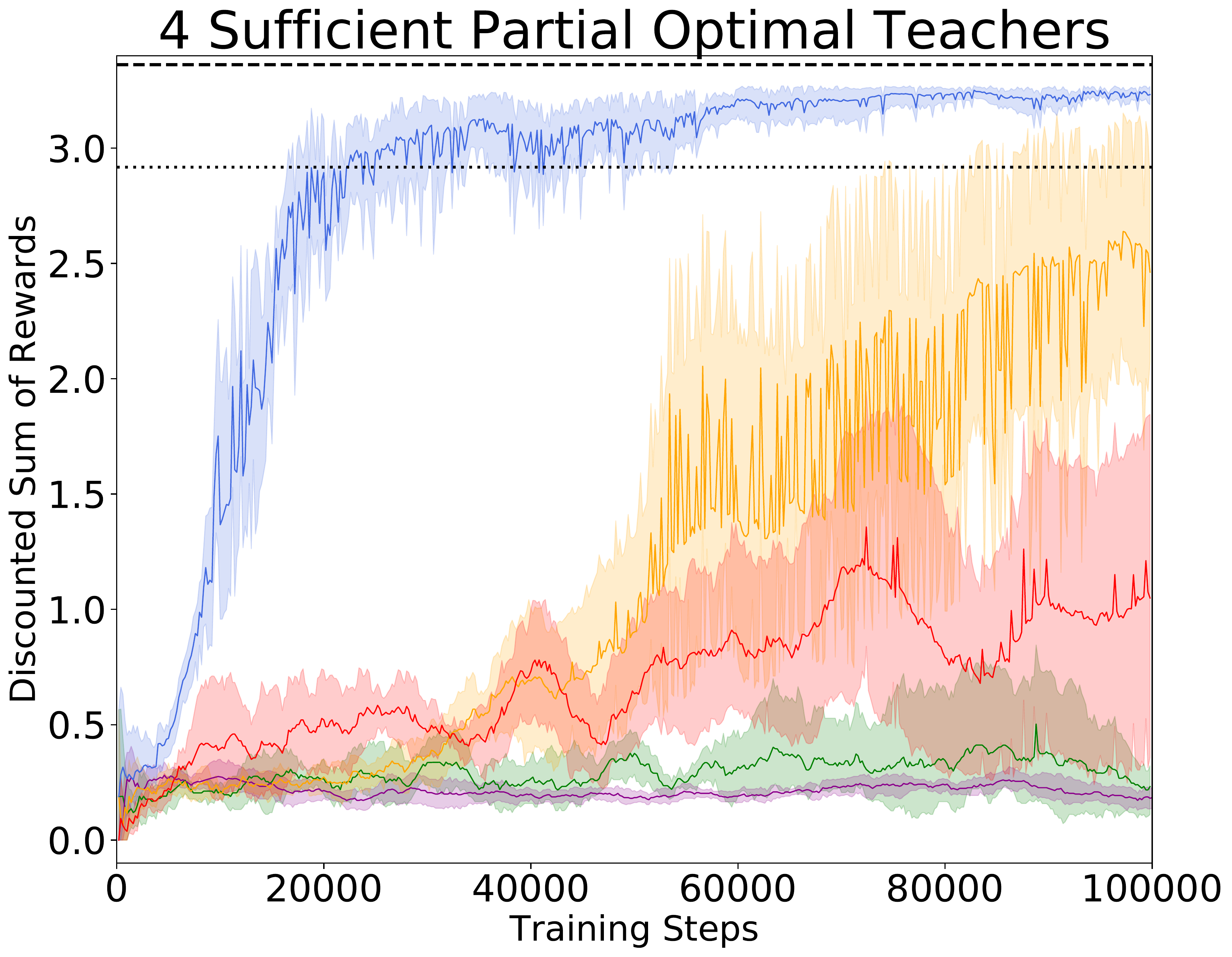}
    \end{subfigure}
    \hfill
    \begin{subfigure}[b]{0.245\textwidth}
        \centering
        \includegraphics[width=\textwidth]{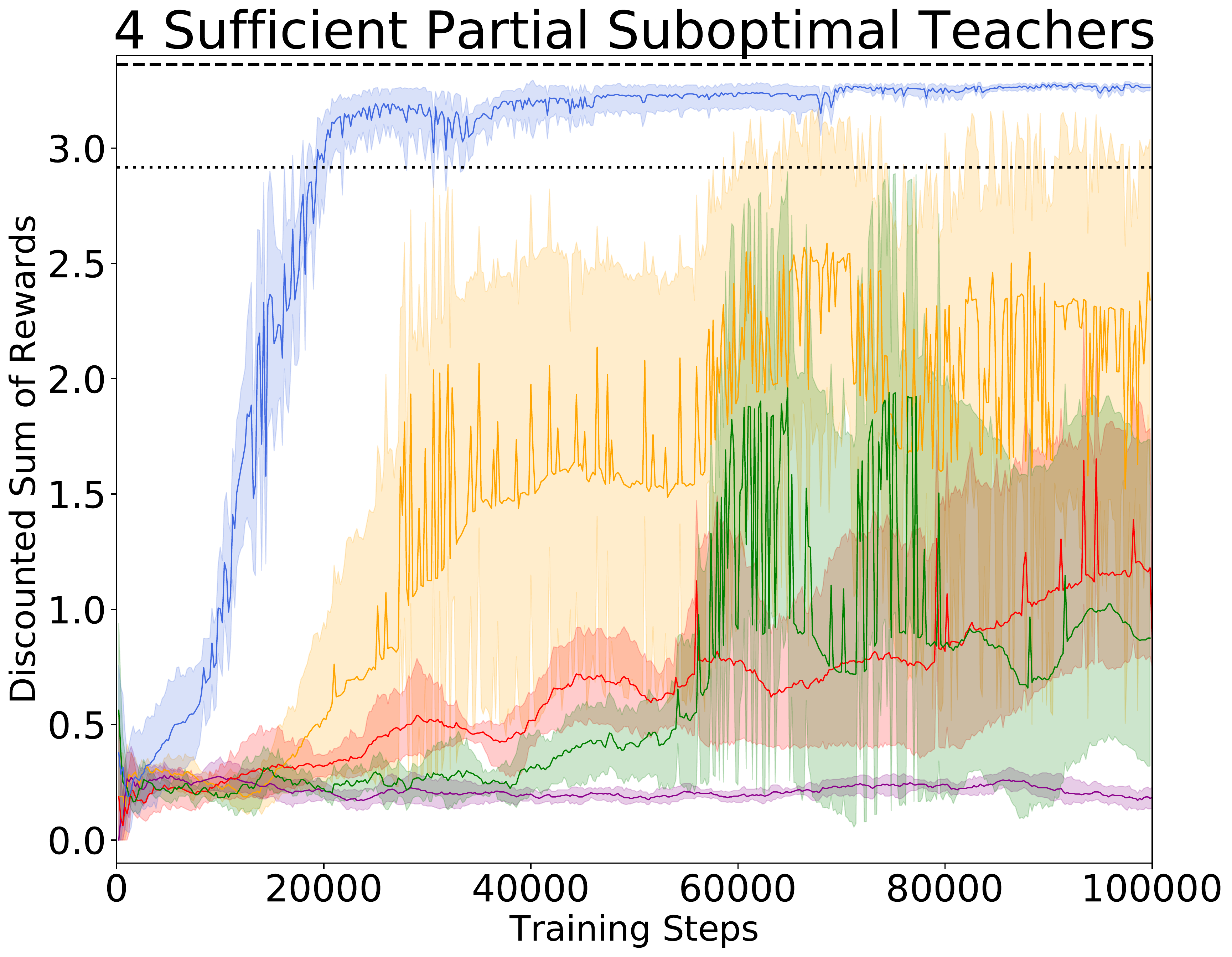}
    \end{subfigure}
     \begin{subfigure}[b]{0.245\textwidth}
        \centering
        \includegraphics[width=\textwidth]{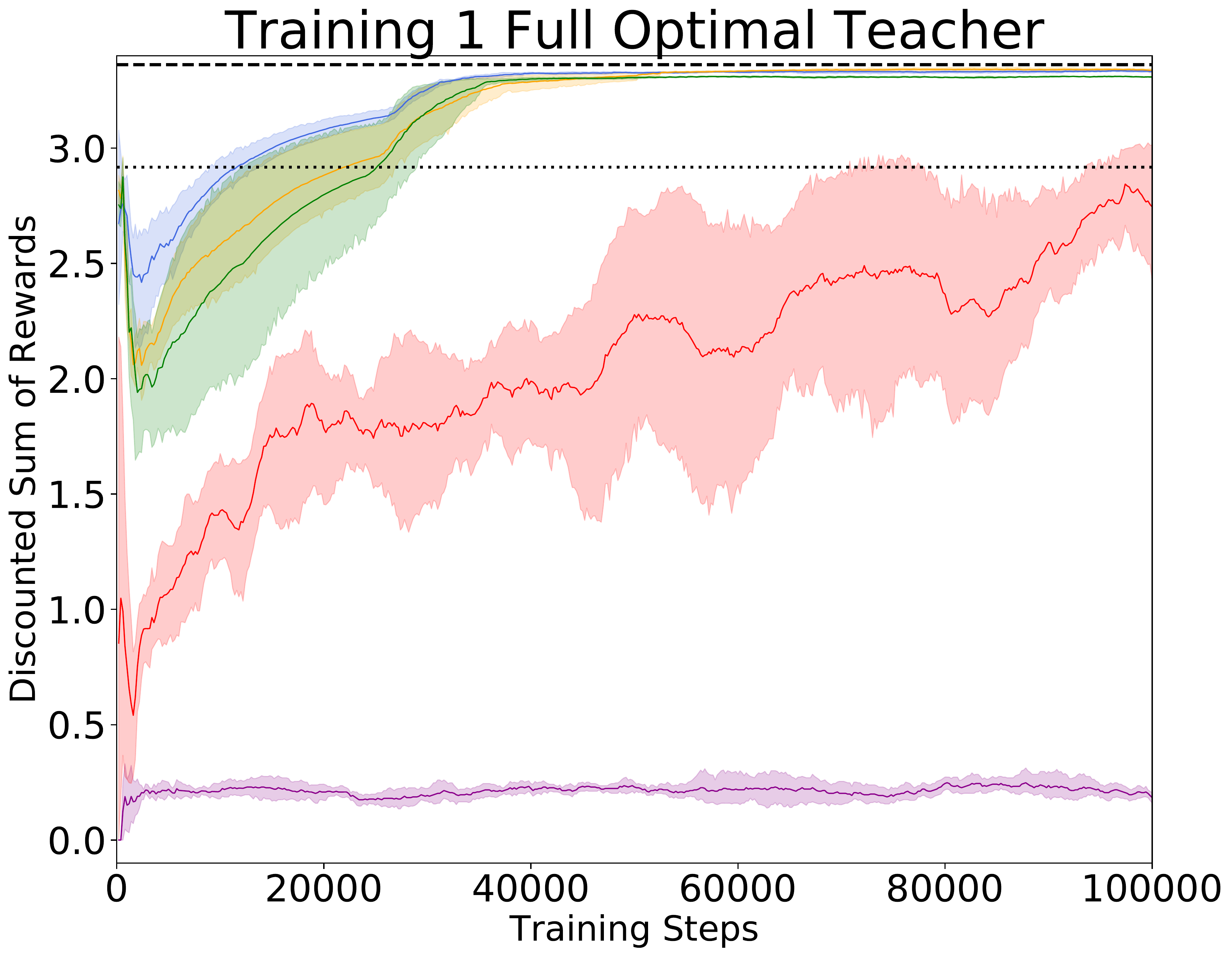}
    \end{subfigure}%
    \hfill
    \begin{subfigure}[b]{0.245\textwidth}
        \centering
        \includegraphics[width=\textwidth]{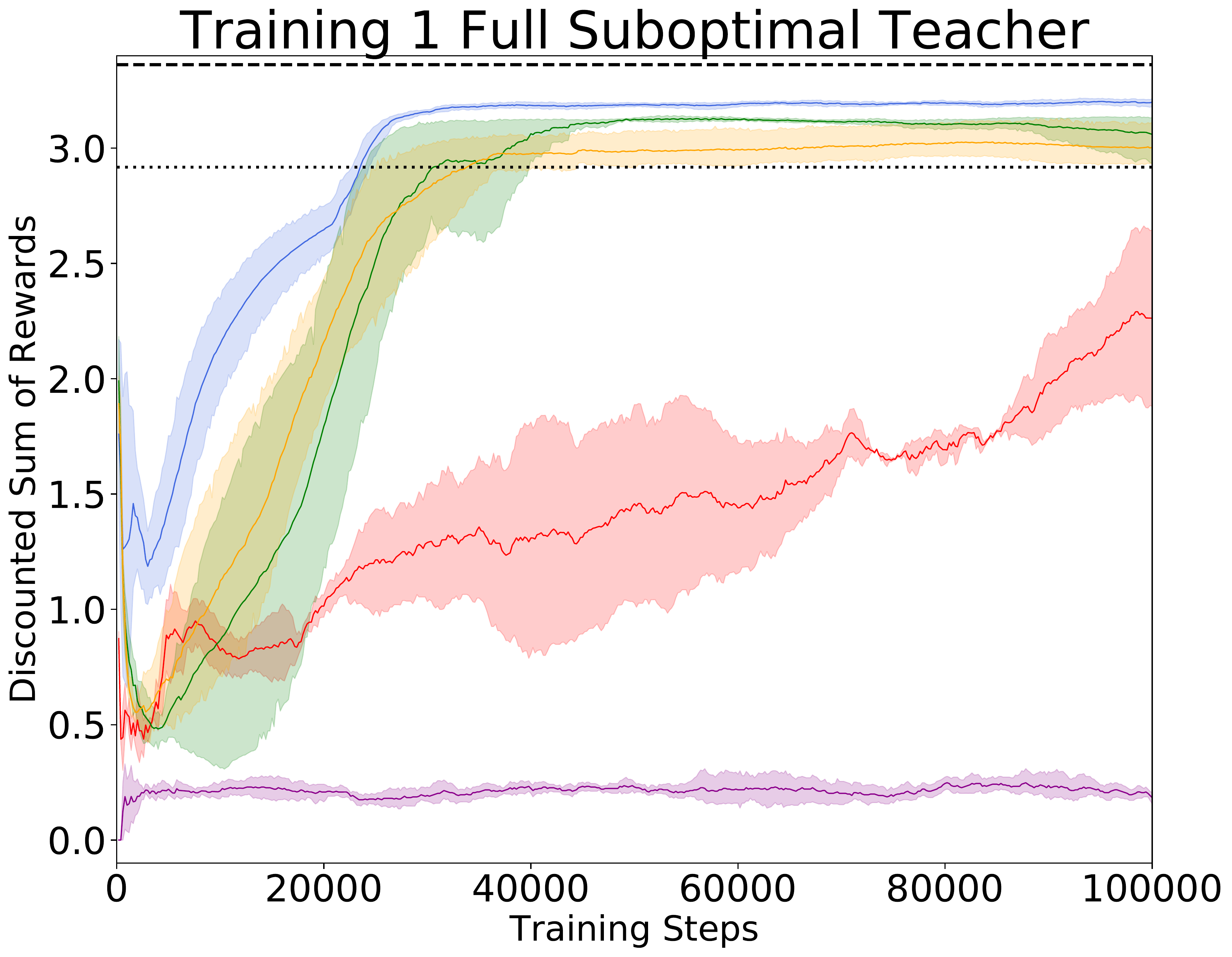}
    \end{subfigure}
    \hfill
    \begin{subfigure}[b]{0.245\textwidth}
        \centering
        \includegraphics[width=\textwidth]{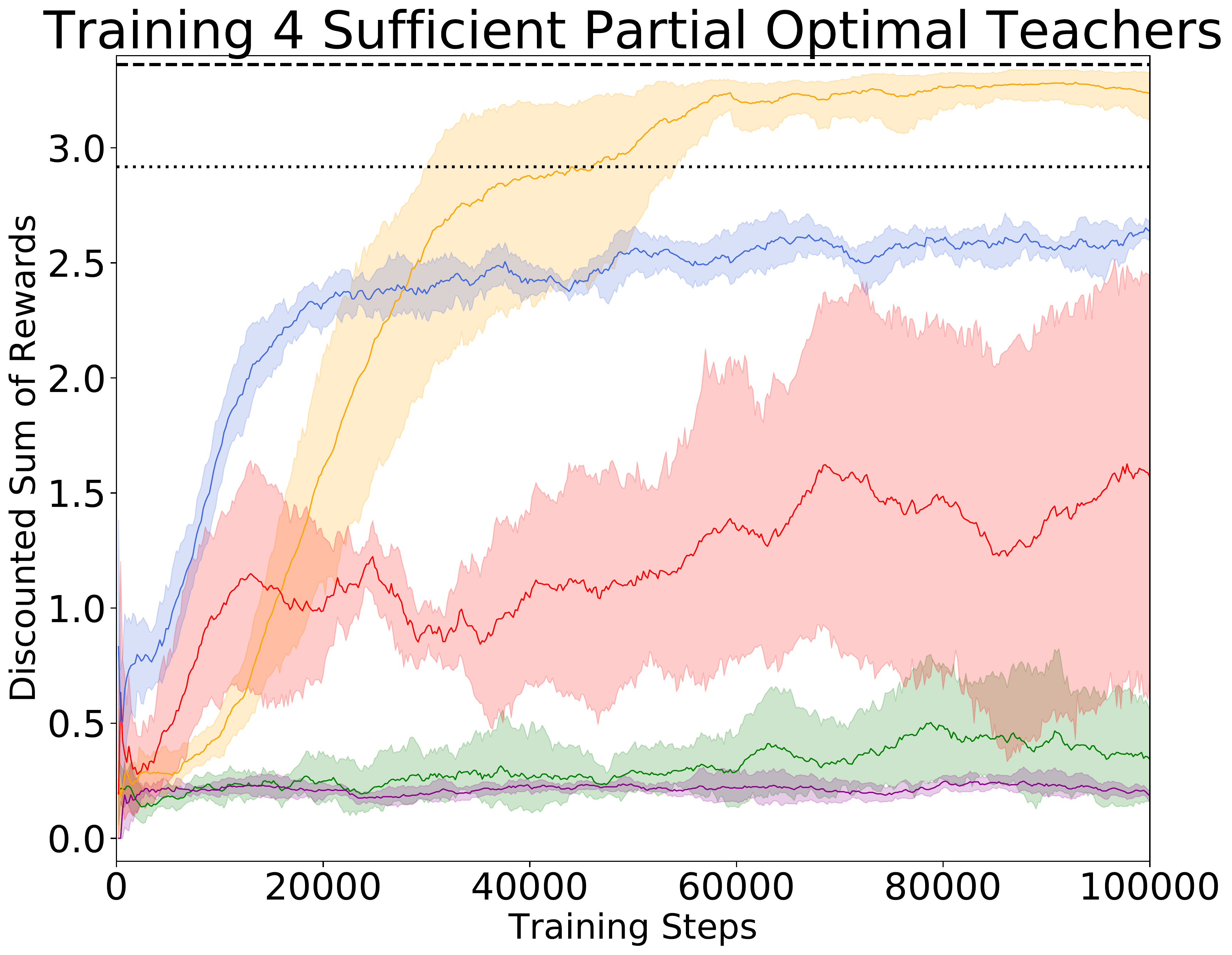}
    \end{subfigure}
    \hfill
    \begin{subfigure}[b]{0.245\textwidth}
        \centering
        \includegraphics[width=\textwidth]{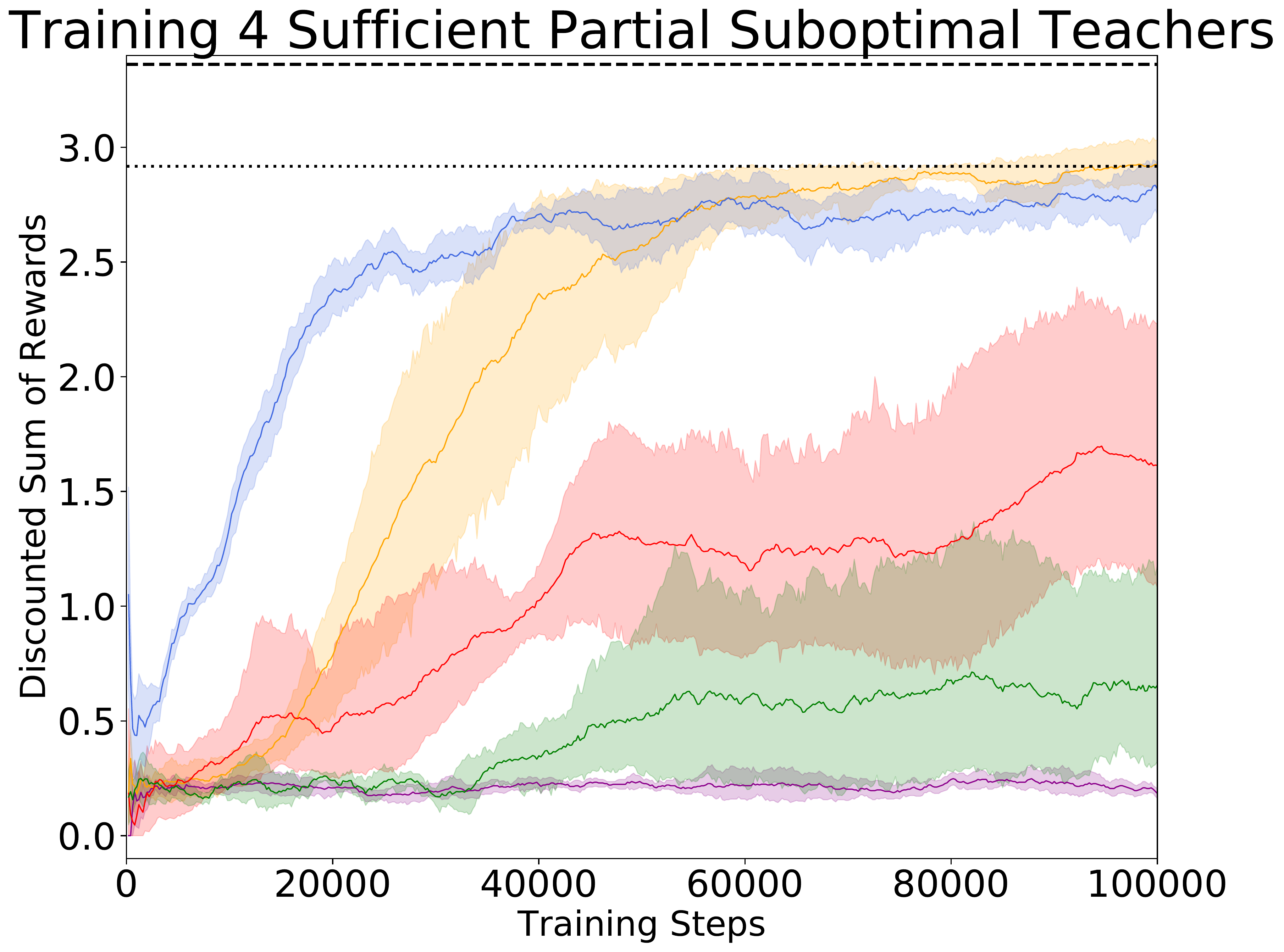}
    \end{subfigure}
   \caption{\textbf{Evaluating~\methodname with \textit{sufficient} teacher sets.} Test and train time agent performance on the \texttt{Path Following} task for~\methodname and other baselines. At test time, several baselines perform almost as well as our method when utilizing just one \textit{sufficient} teacher (leftmost two graphs), but when there are multiple \textit{partial} teachers (rightmost two graphs) our method significantly outperforms all others in both convergence speed and asymptotic performance. At train time, it is notable that the DQN baseline is rarely able to surpass the quality of the teachers it is supplied with and that the performance of the learned agent decreases substantially at test time despite having been trained with the experience of the better-performing DQN behavioral policy at train time.}
   \label{fig:path_a}
\end{figure}

\begin{figure}[t!]
   \centering
    \includegraphics[width=\textwidth]{figs/legend.png}
    \begin{subfigure}[b]{0.245\textwidth}
        \centering
        \includegraphics[width=\textwidth]{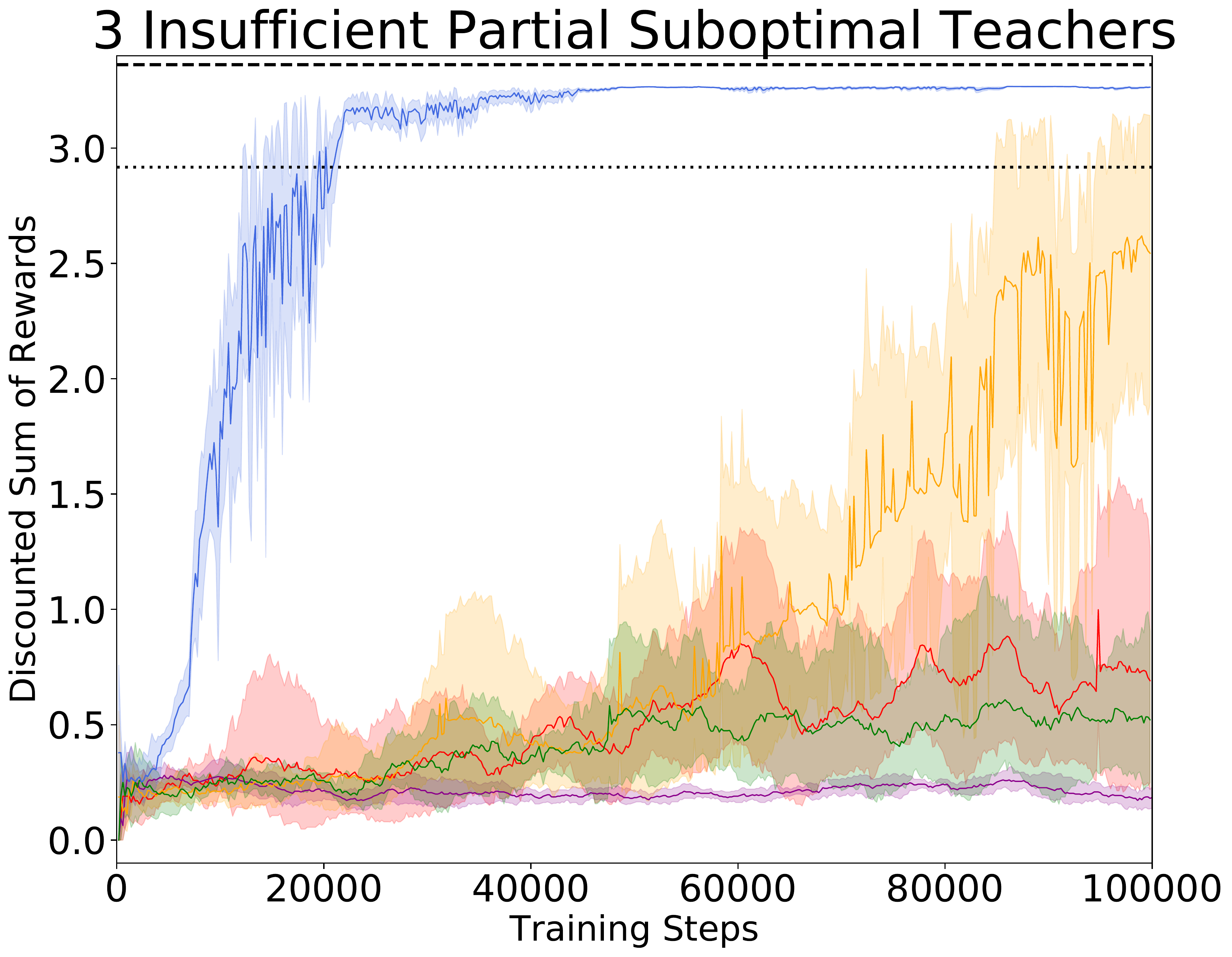}
    \end{subfigure}%
    \hfill
    \begin{subfigure}[b]{0.245\textwidth}
        \centering
        \includegraphics[width=\textwidth]{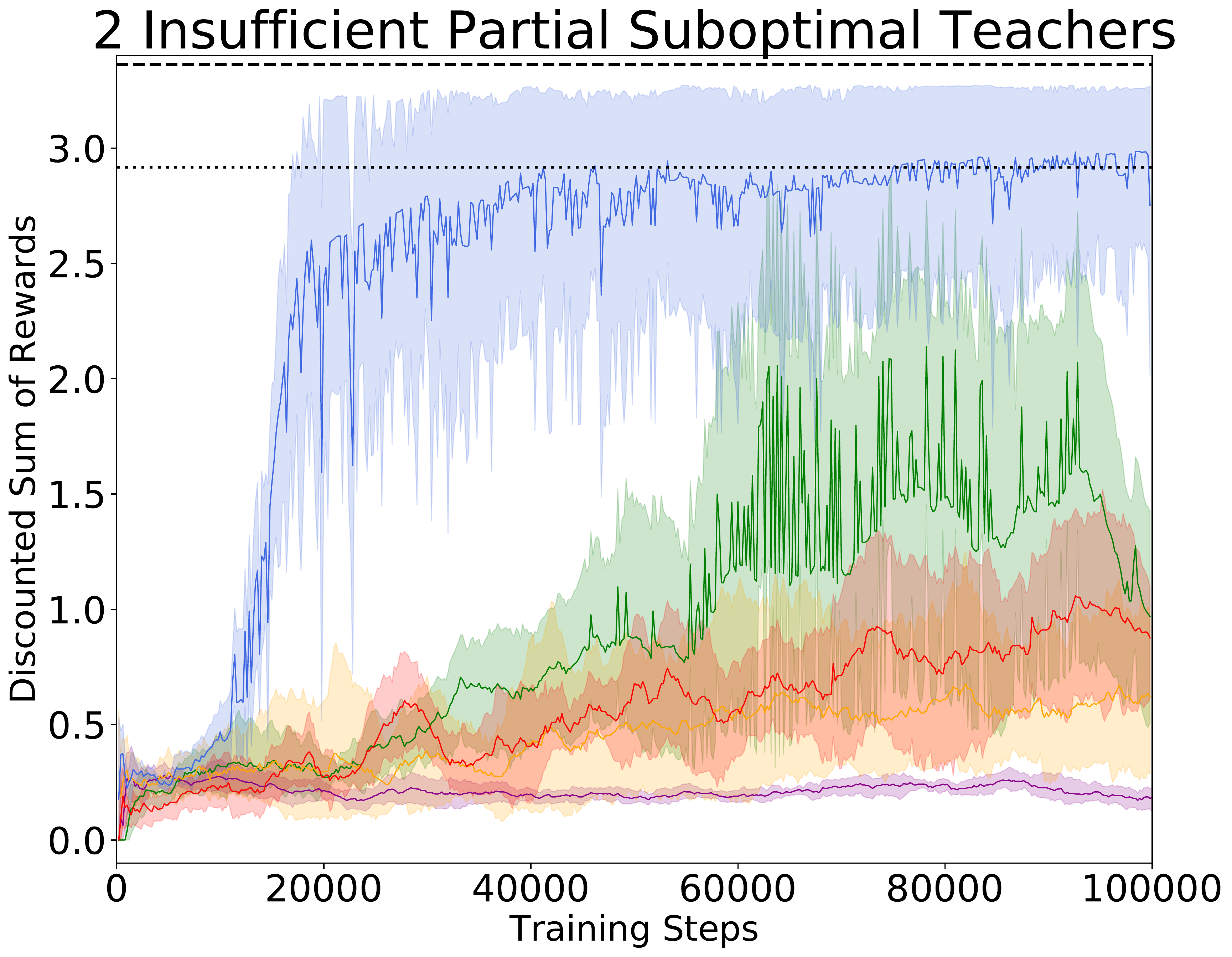}
    \end{subfigure}
    \hfill
    \begin{subfigure}[b]{0.245\textwidth}
        \centering
        \includegraphics[width=\textwidth]{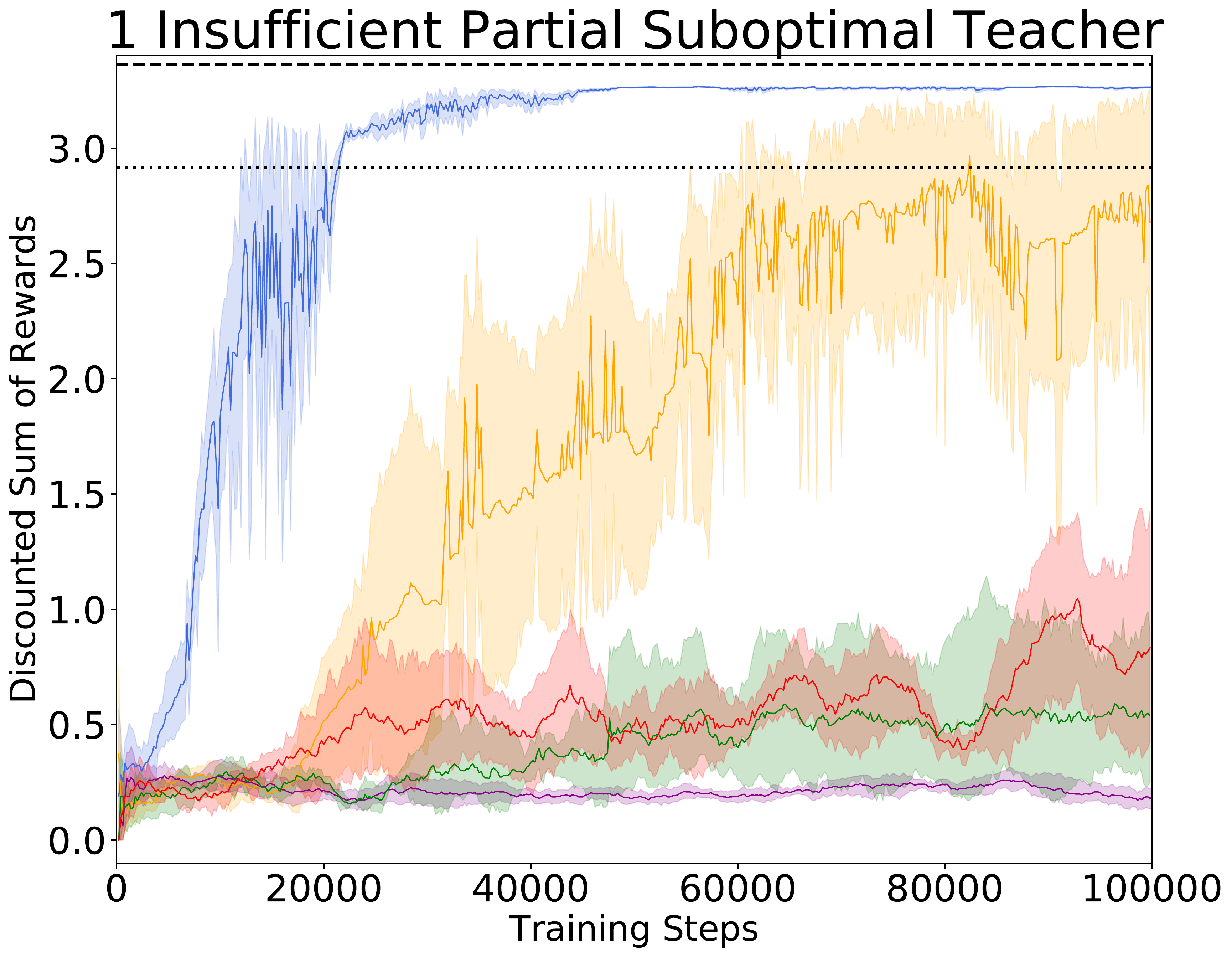}
    \end{subfigure}
    \hfill
    \begin{subfigure}[b]{0.245\textwidth}
        \centering
        \includegraphics[width=\textwidth]{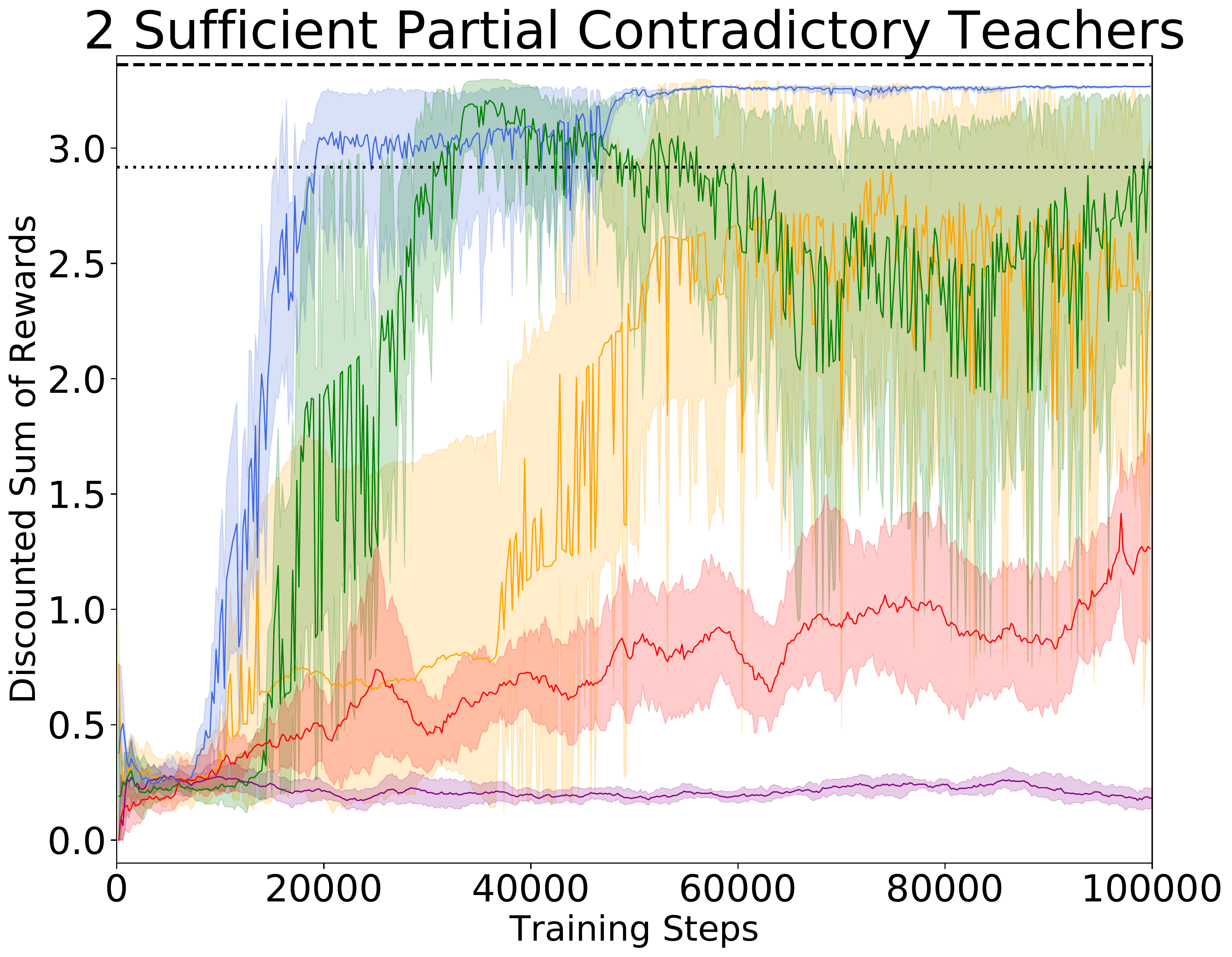}
    \end{subfigure}
    \begin{subfigure}[b]{0.245\textwidth}
        \centering
        \includegraphics[width=\textwidth]{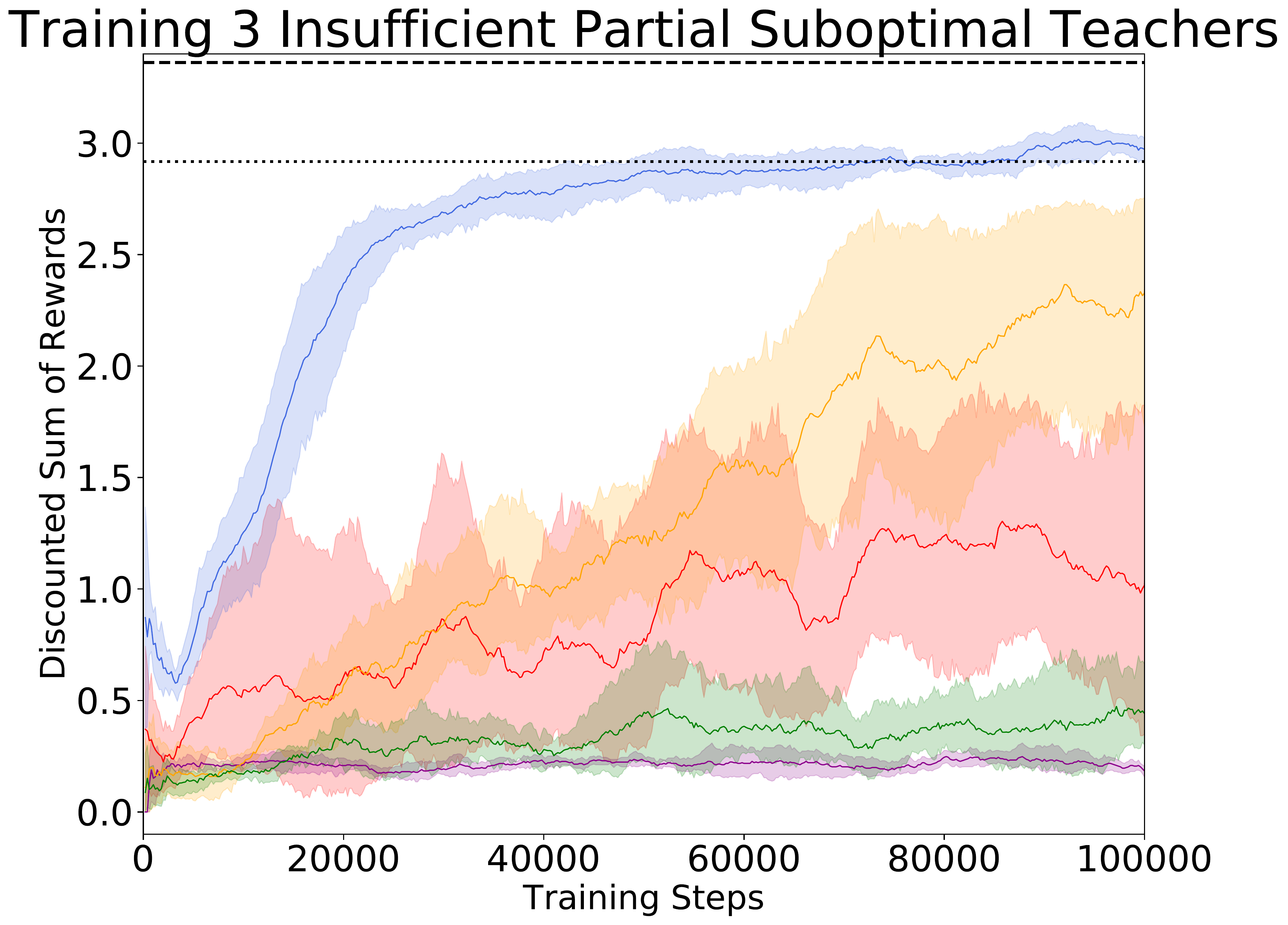}
    \end{subfigure}%
    \hfill
    \begin{subfigure}[b]{0.245\textwidth}
        \centering
        \includegraphics[width=\textwidth]{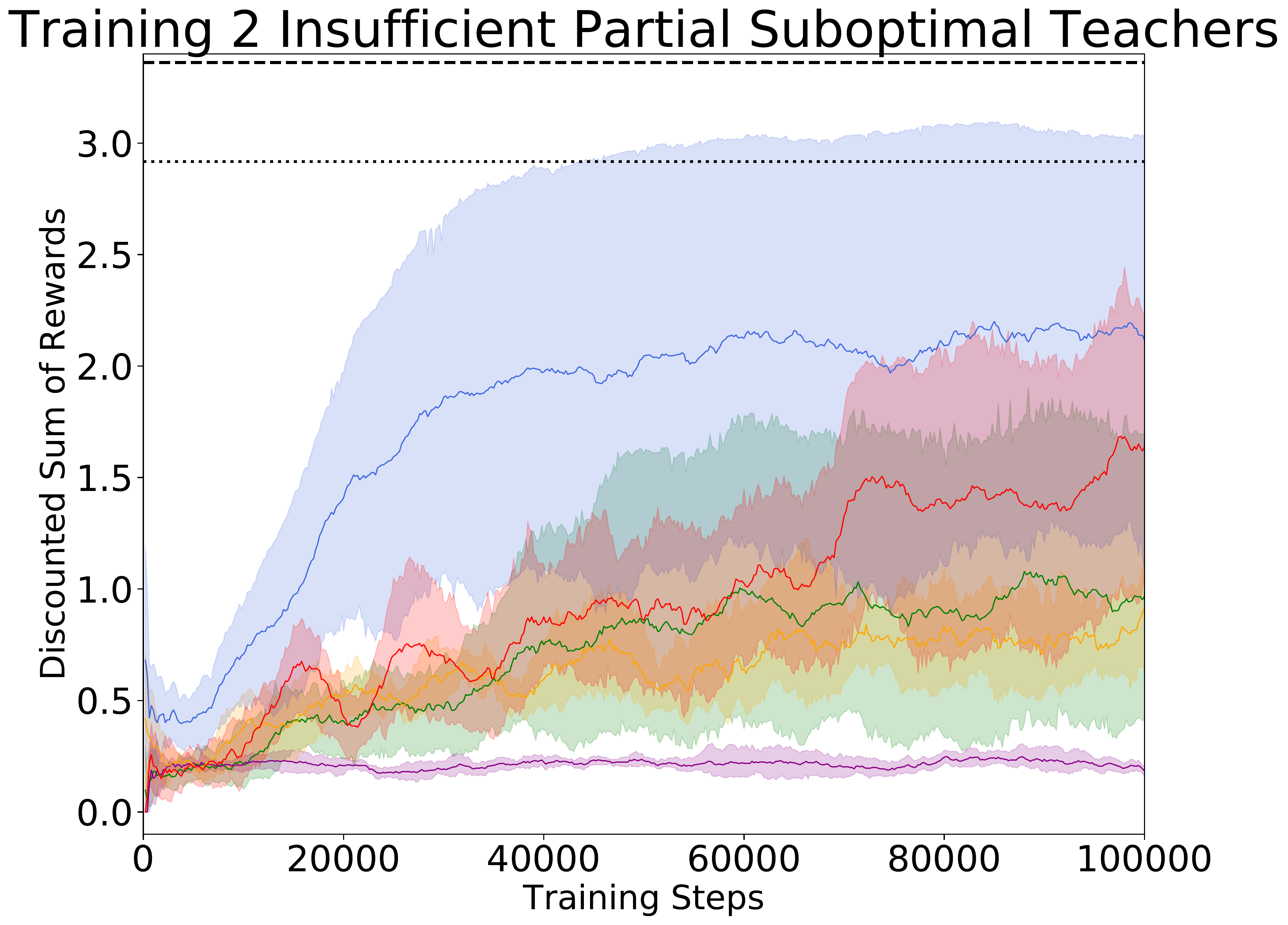}
    \end{subfigure}
    \hfill
    \begin{subfigure}[b]{0.245\textwidth}
        \centering
        \includegraphics[width=\textwidth]{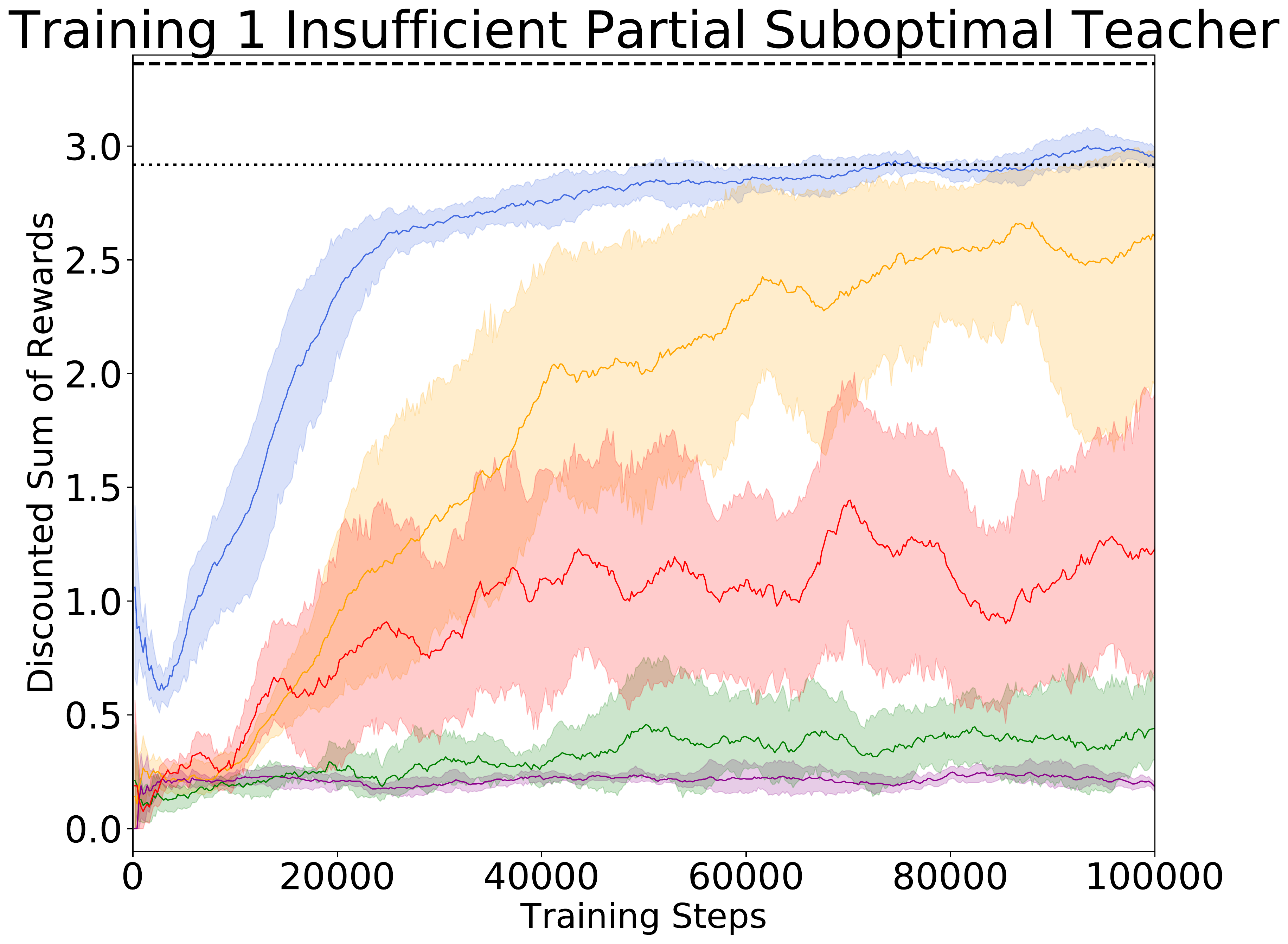}
    \end{subfigure}
    \hfill
    \begin{subfigure}[b]{0.245\textwidth}
        \centering
        \includegraphics[width=\textwidth]{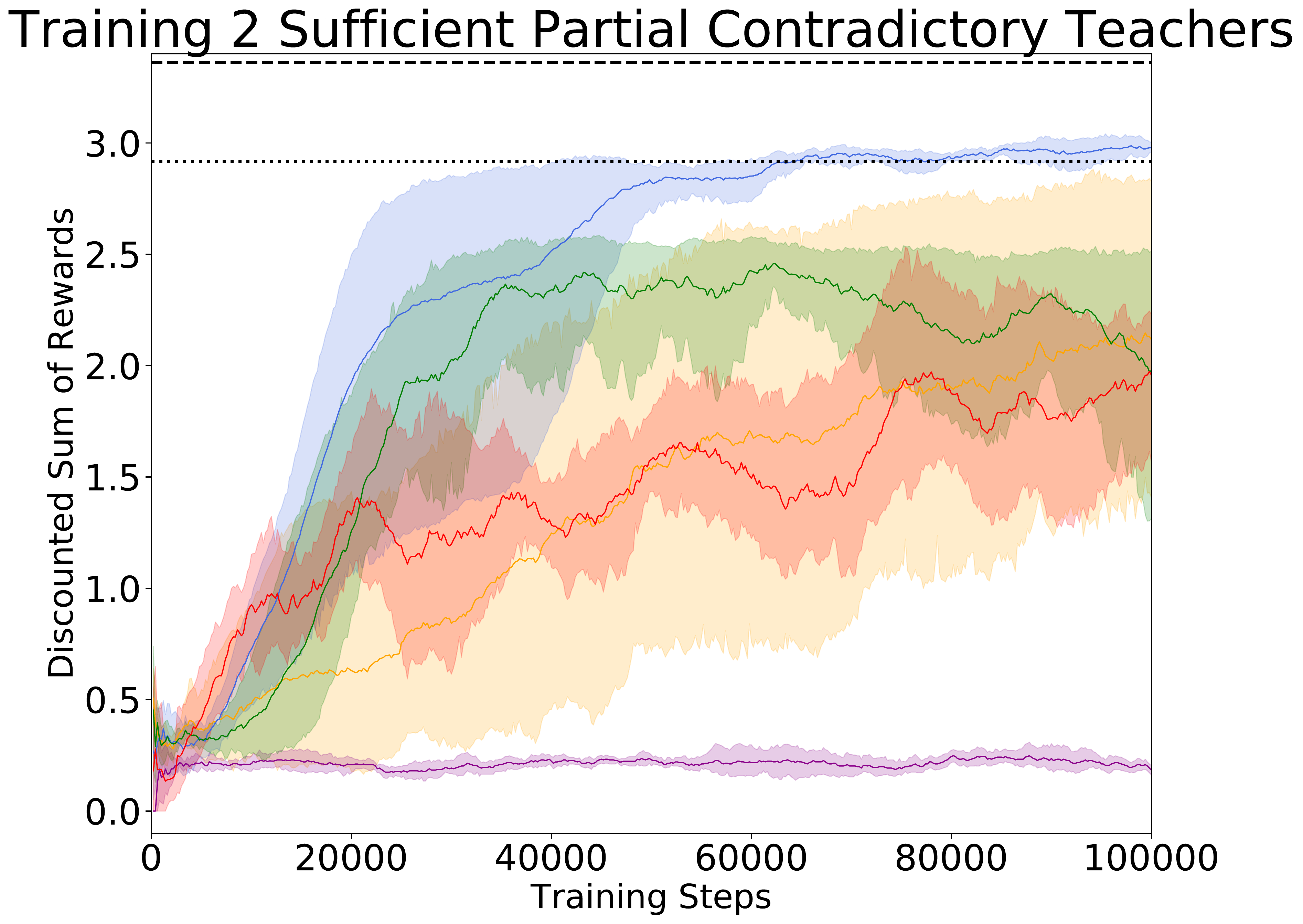}
    \end{subfigure}
   \caption{\textbf{Evaluating~\methodname on \textit{insufficient} and \textit{contradictory} teacher sets.}  The first three sets are incomplete since they are missing teachers for 1, 2, and 3 goals respectively, while the last set is \textit{contradictory} because the midpoint teacher and endpoint teacher try to move the agent to different parts of the state space. We see that our method significantly outperforms others in both convergence speed and asymptotic performance, indicating that it is capable of learning parts of the task where no teacher offers useful advice, as well as learning when to leverage each teacher's advice, even when the advice is \textit{contradictory}.}
   \label{fig:path_b}
\end{figure}

Teacher set A consists of 3 \textit{suboptimal}, \textit{partial} teachers, one that can navigate to each goal. 
Notice that one of the 4 goals will not have a corresponding teacher. 
Similarly, teacher sets B and C consist of 2 \textit{suboptimal}, \textit{partial} teachers and 1 \textit{suboptimal}, \textit{partial} teacher respectively. 
Finally, teacher set D consists of a \textit{midpoint} teacher and a \textit{endpoint} teacher that can be \textit{contradictory}. 

Notice that sets A, B, and C all test for the ability to learn from \textit{insufficient} sets, with set C being the most difficult as the teacher only helps with a quarter of the task. 
Meanwhile, set D tests for the ability to learn from a set of teachers that offers \textit{contradictory} advice, as the midpoint teacher always tries to move to the midpoint between the current and previous goal and the endpoint teacher will move away from the midpoint, towards either one goal or the other, depending on which is closer. 

Fig.~\ref{fig:path_b} demonstrates that our algorithm is consistently able to leverage the incomplete sets of teachers to learn the task - whether learning from 3, 2, or even just 1 of the single-goal teachers. The agent quickly learns where to trust the teachers' expertise and how to act in parts of the state space where teachers do not offer useful advice. 

\begin{figure}[!t]
   \centering
    \includegraphics[width=\textwidth]{figs/legend.png}
    \begin{subfigure}[b]{0.245\textwidth}
        \centering
        \includegraphics[width=\textwidth]{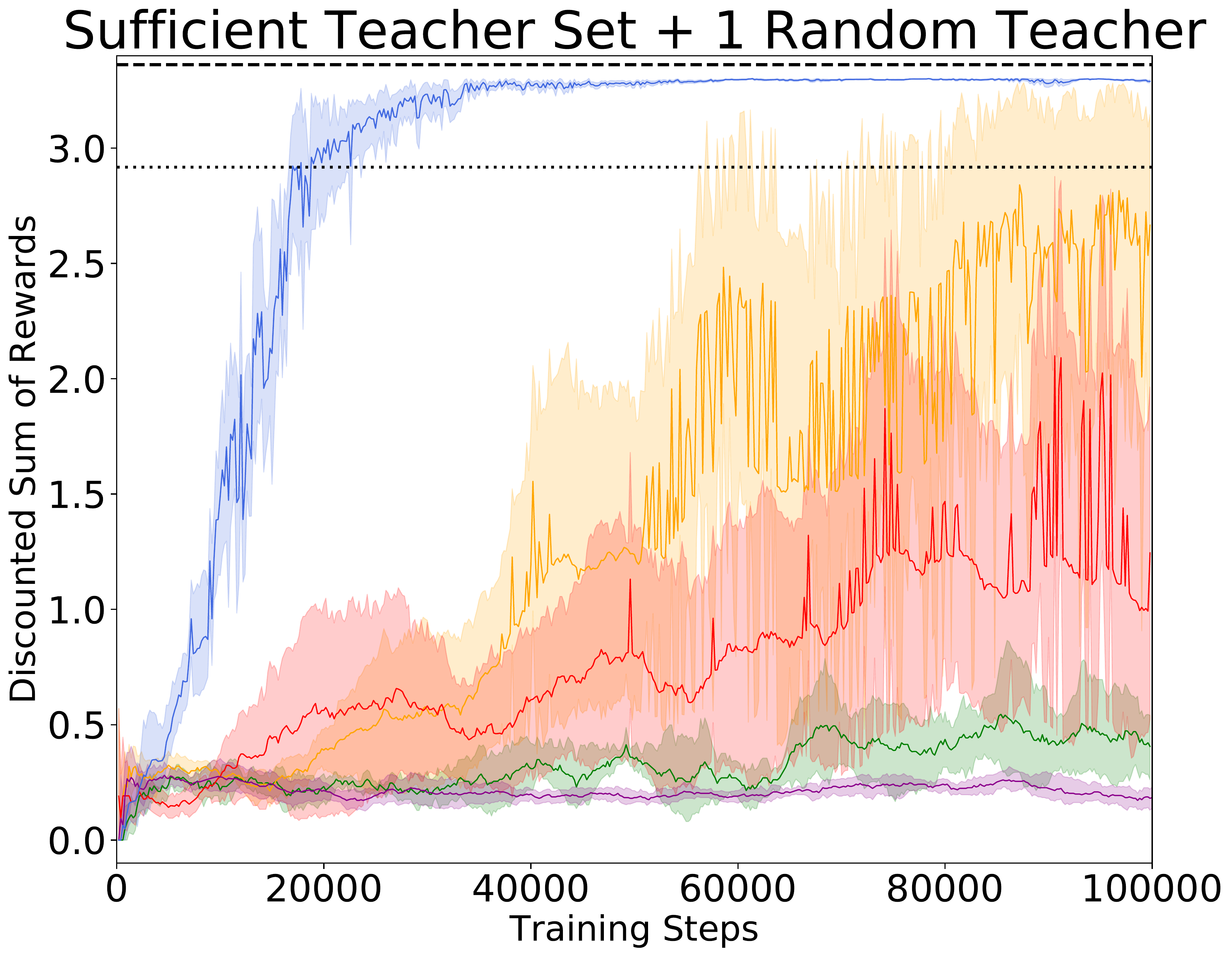}
    \end{subfigure}%
    \hfill
    \begin{subfigure}[b]{0.245\textwidth}
        \centering
        \includegraphics[width=\textwidth]{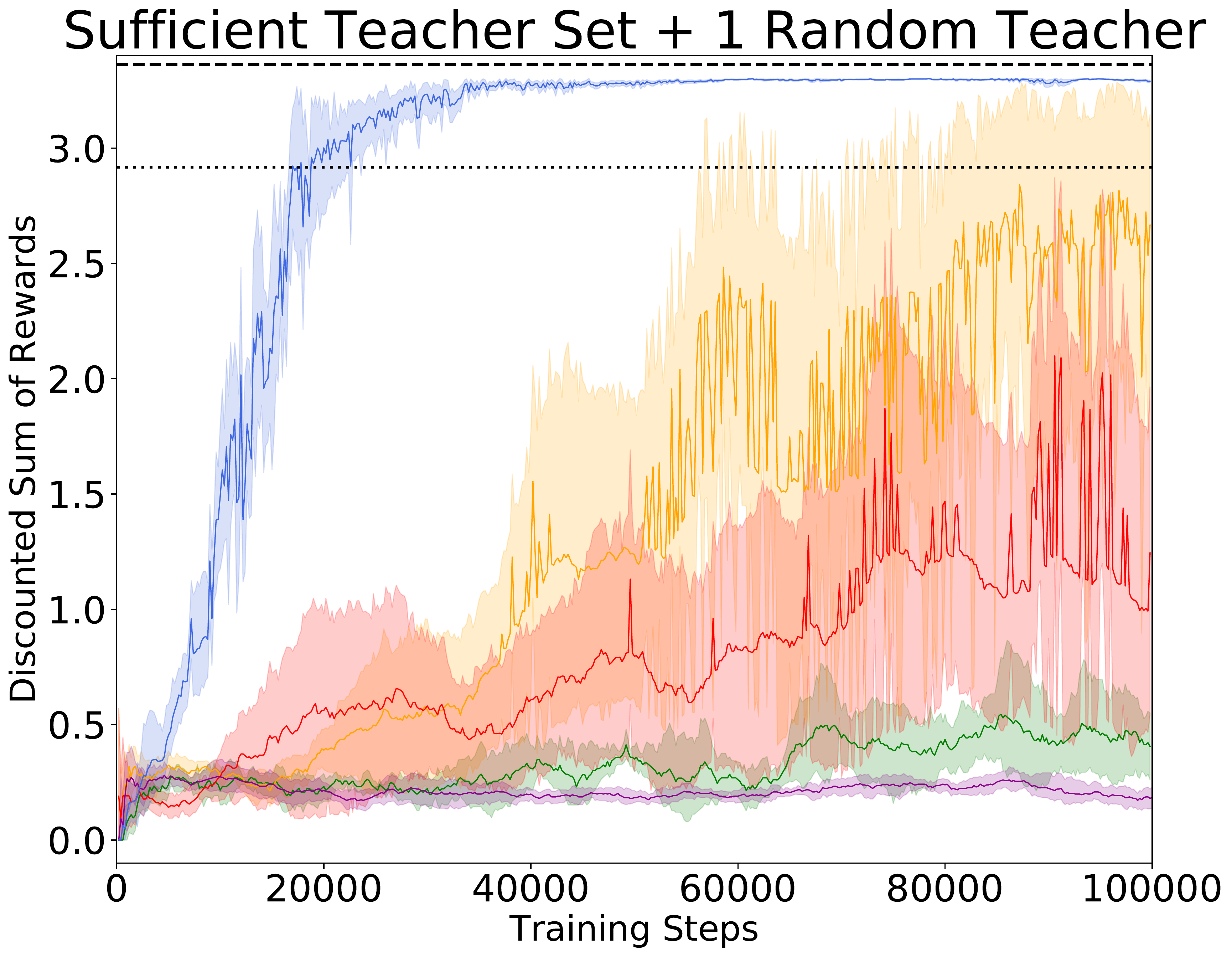}
    \end{subfigure}
    \hfill
    \begin{subfigure}[b]{0.245\textwidth}
        \centering
        \includegraphics[width=\textwidth]{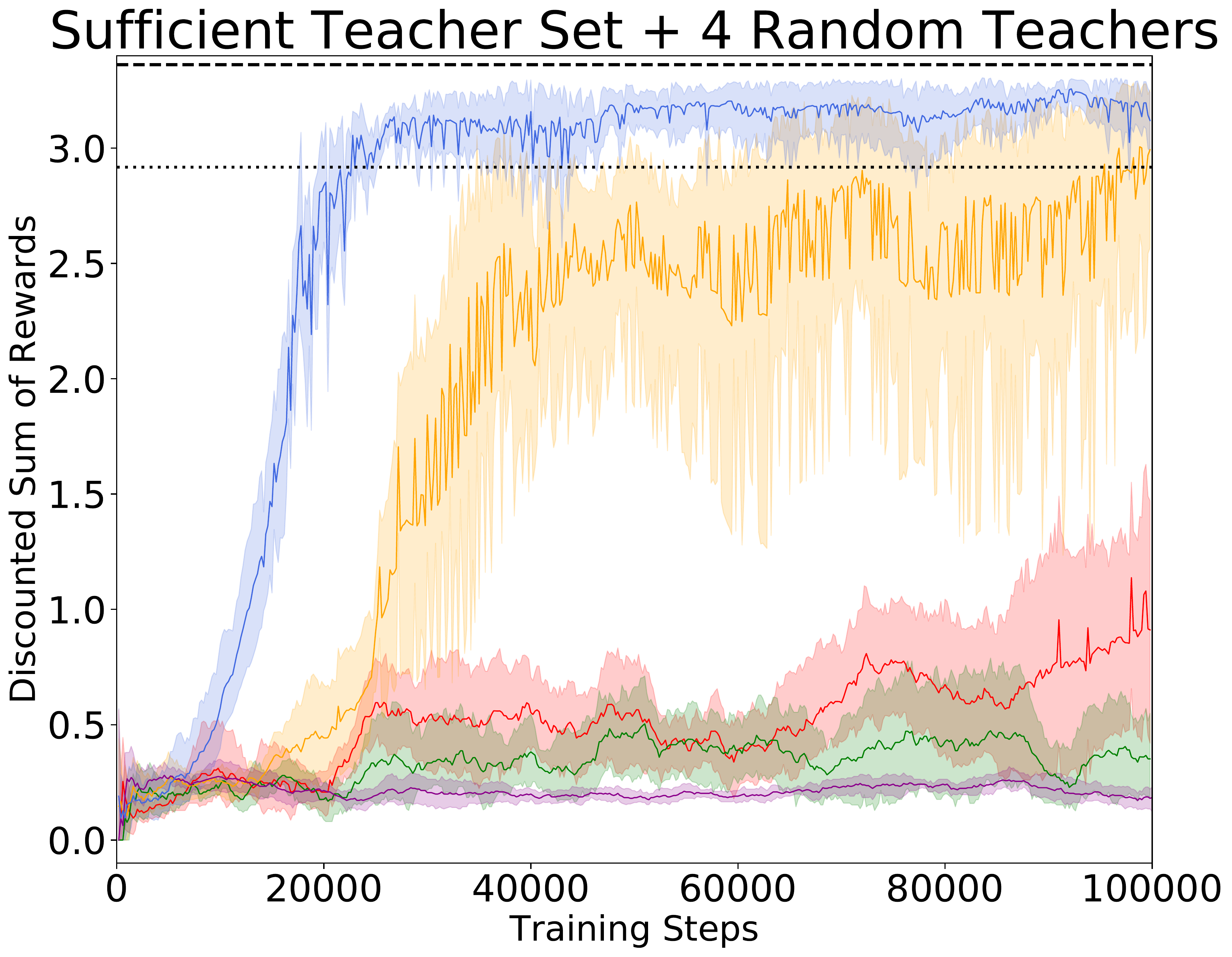}
    \end{subfigure}
    \hfill
    \begin{subfigure}[b]{0.245\textwidth}
        \centering
        \includegraphics[width=\textwidth]{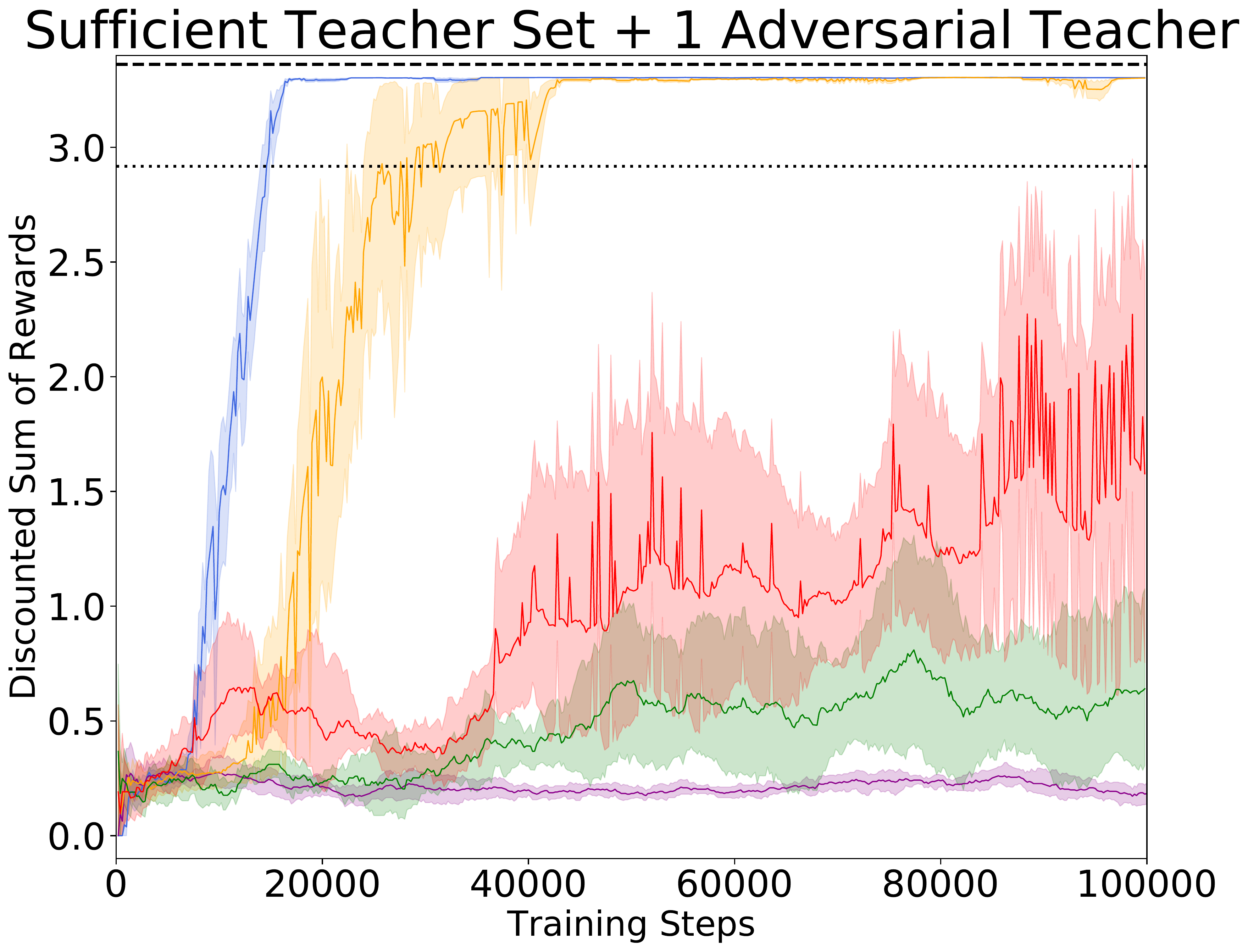}
    \end{subfigure}
    
    \begin{subfigure}[b]{0.245\textwidth}
        \centering
        \includegraphics[width=\textwidth]{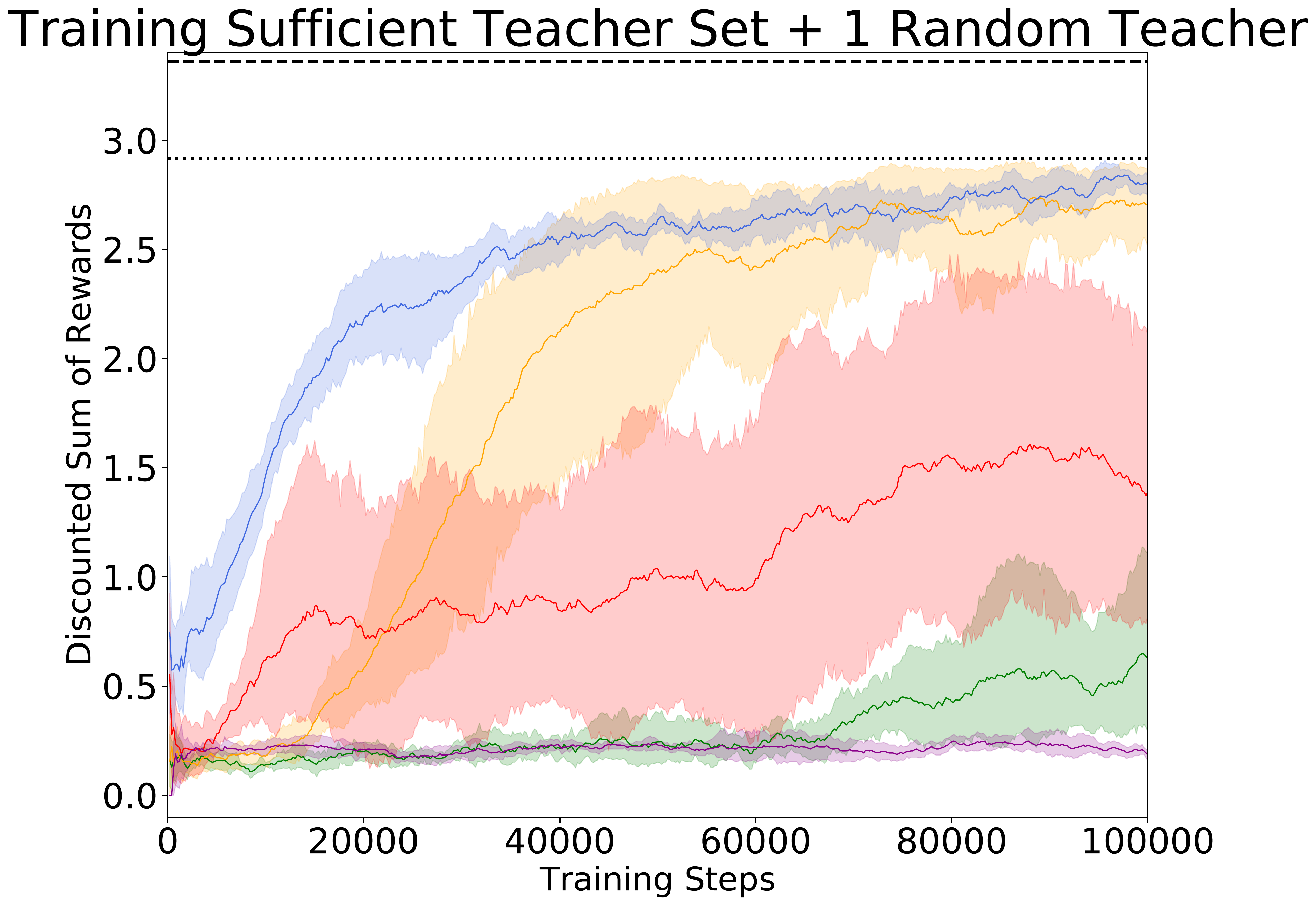}
    \end{subfigure}%
    \hfill
    \begin{subfigure}[b]{0.245\textwidth}
        \centering
        \includegraphics[width=\textwidth]{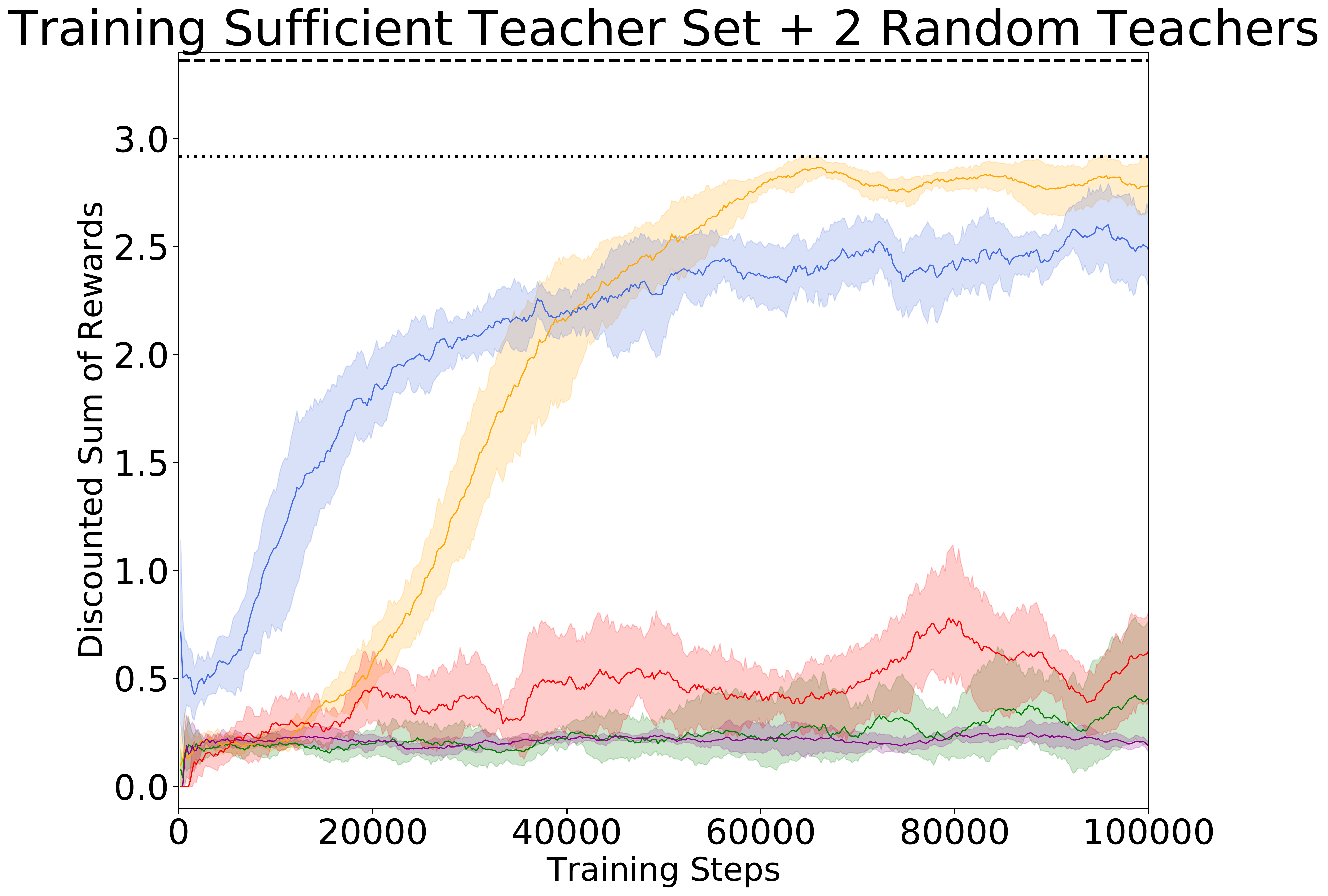}
    \end{subfigure}
    \hfill
    \begin{subfigure}[b]{0.245\textwidth}
        \centering
        \includegraphics[width=\textwidth]{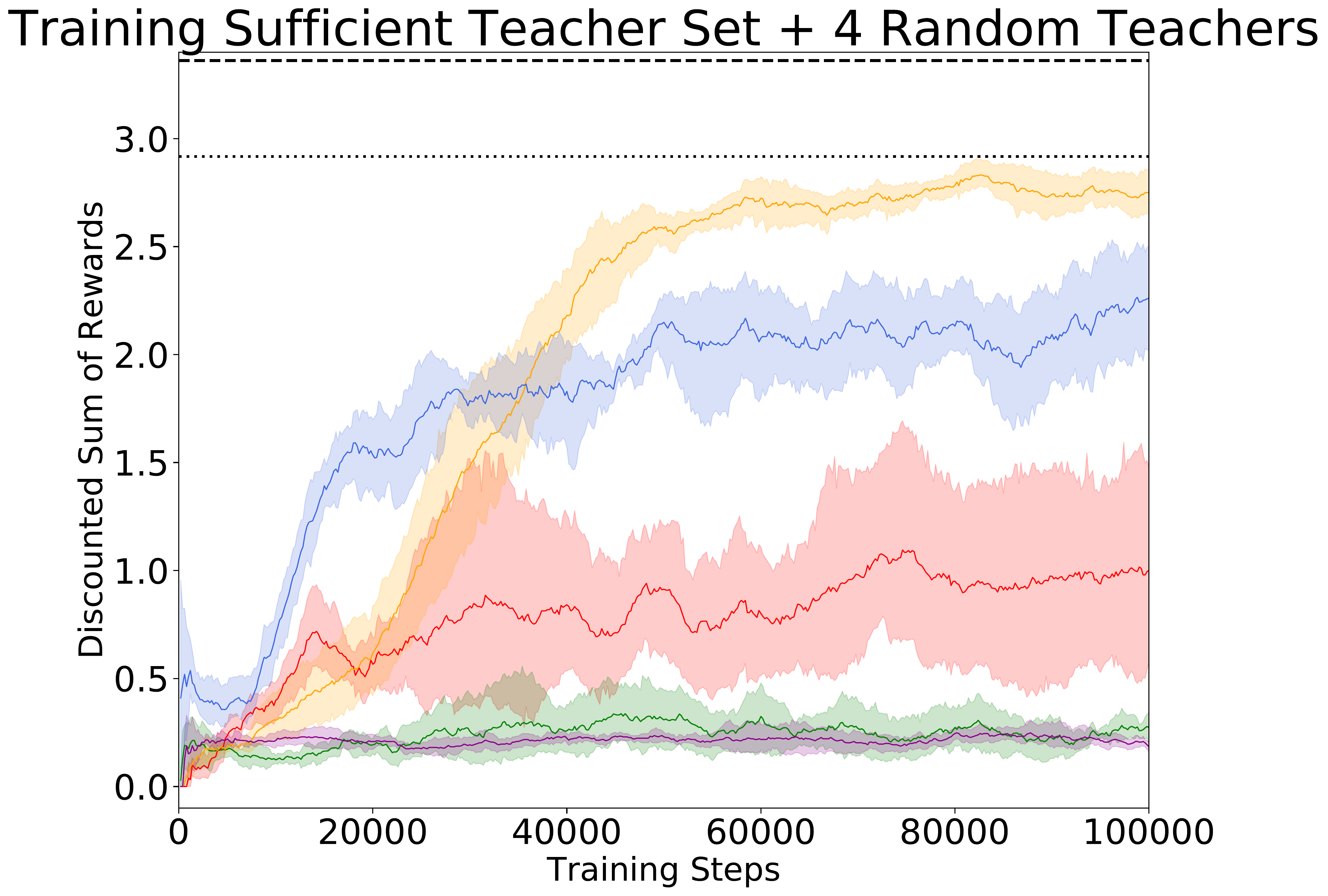}
    \end{subfigure}
    \hfill
    \begin{subfigure}[b]{0.245\textwidth}
        \centering
        \includegraphics[width=\textwidth]{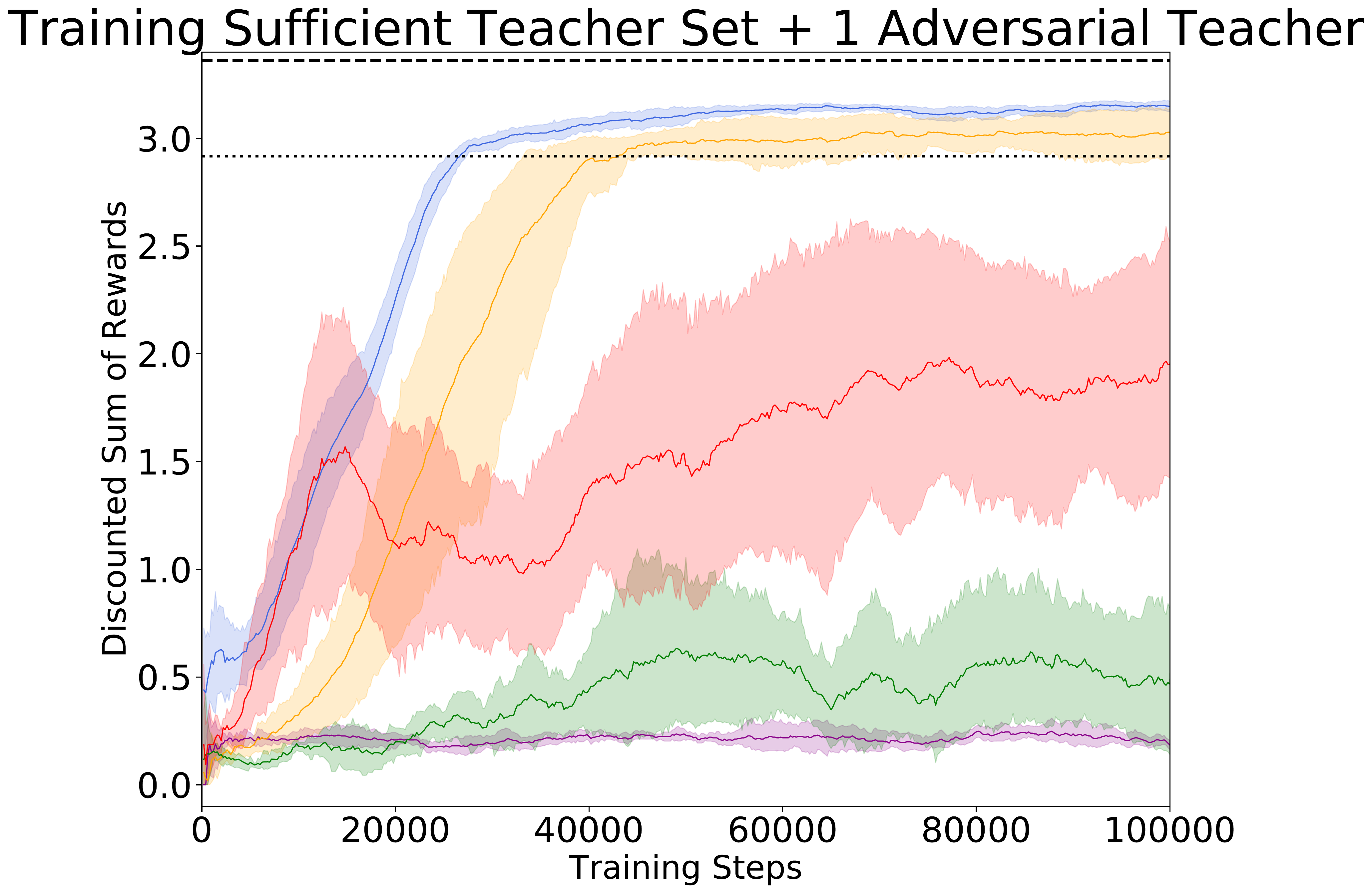}
    \end{subfigure}
   \caption{\textbf{Evaluating~\methodname on large and \textit{adversarial} teacher sets.} Test and train time agent performance on the \texttt{Path Following} task for~\methodname and other baselines for the following teacher sets (from left to right): 4 noisy \textit{partial} teachers + 1 random teacher (set E), 4 noisy \textit{partial} teachers + 2 random teachers (set F), 4 noisy \textit{partial} teachers + 3 random teachers (set G), and 1 \textit{sufficient} teacher + 1 \textit{adversarial} teacher (set H). Our method significantly outperforms the baselines in both convergence speed and asymptotic performance, indicating robustness to large sets of teachers where some teachers are random, or even \textit{adversarial}.}
   \label{fig:path_c}
\end{figure}

\begin{figure}[!h]
   \centering
    \begin{subfigure}[b]{0.325\textwidth}
        \centering
        \includegraphics[width=\textwidth]{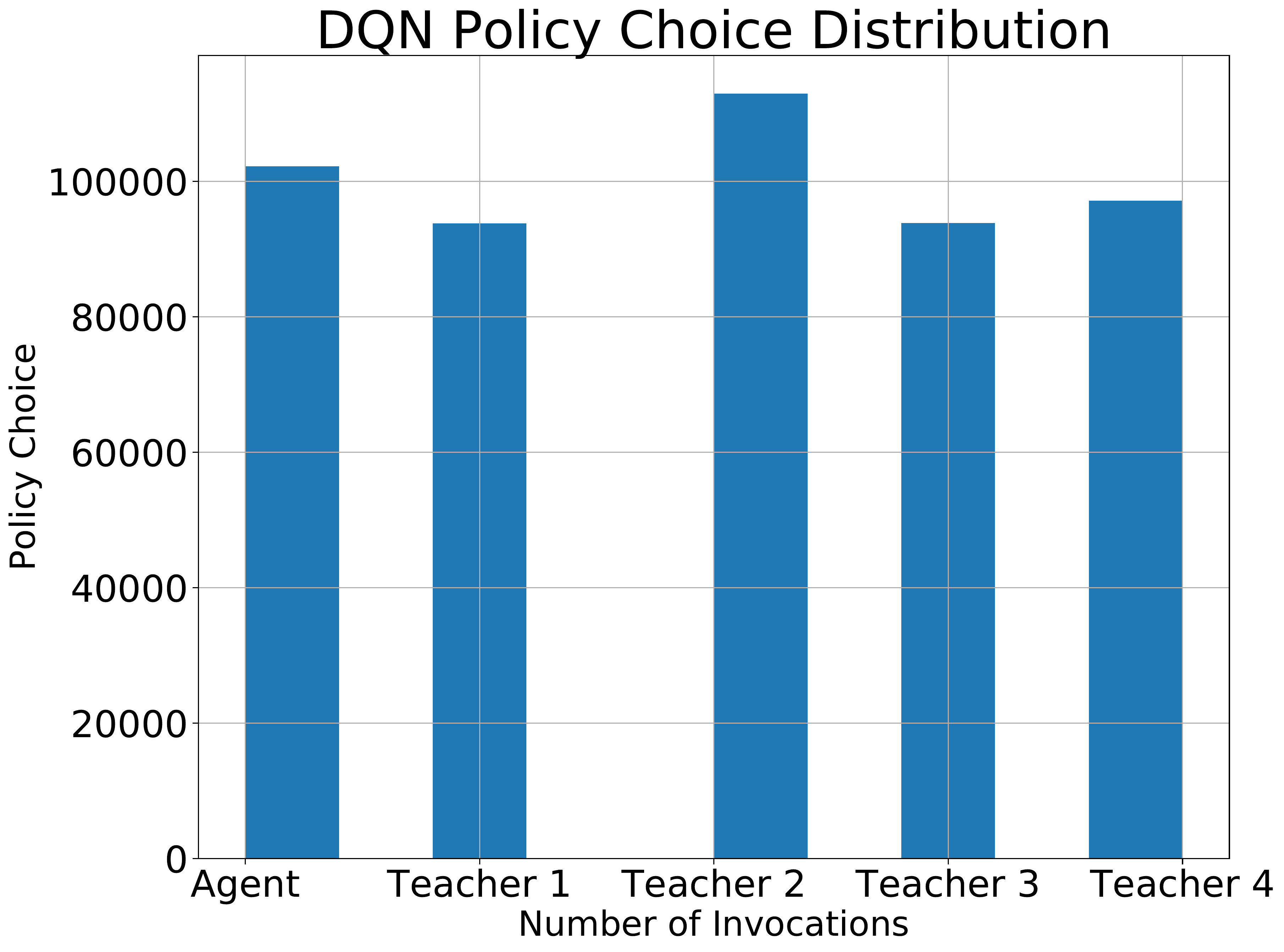}
    \end{subfigure}%
    \hfill
    \begin{subfigure}[b]{0.325\textwidth}
        \centering
        \includegraphics[width=\textwidth]{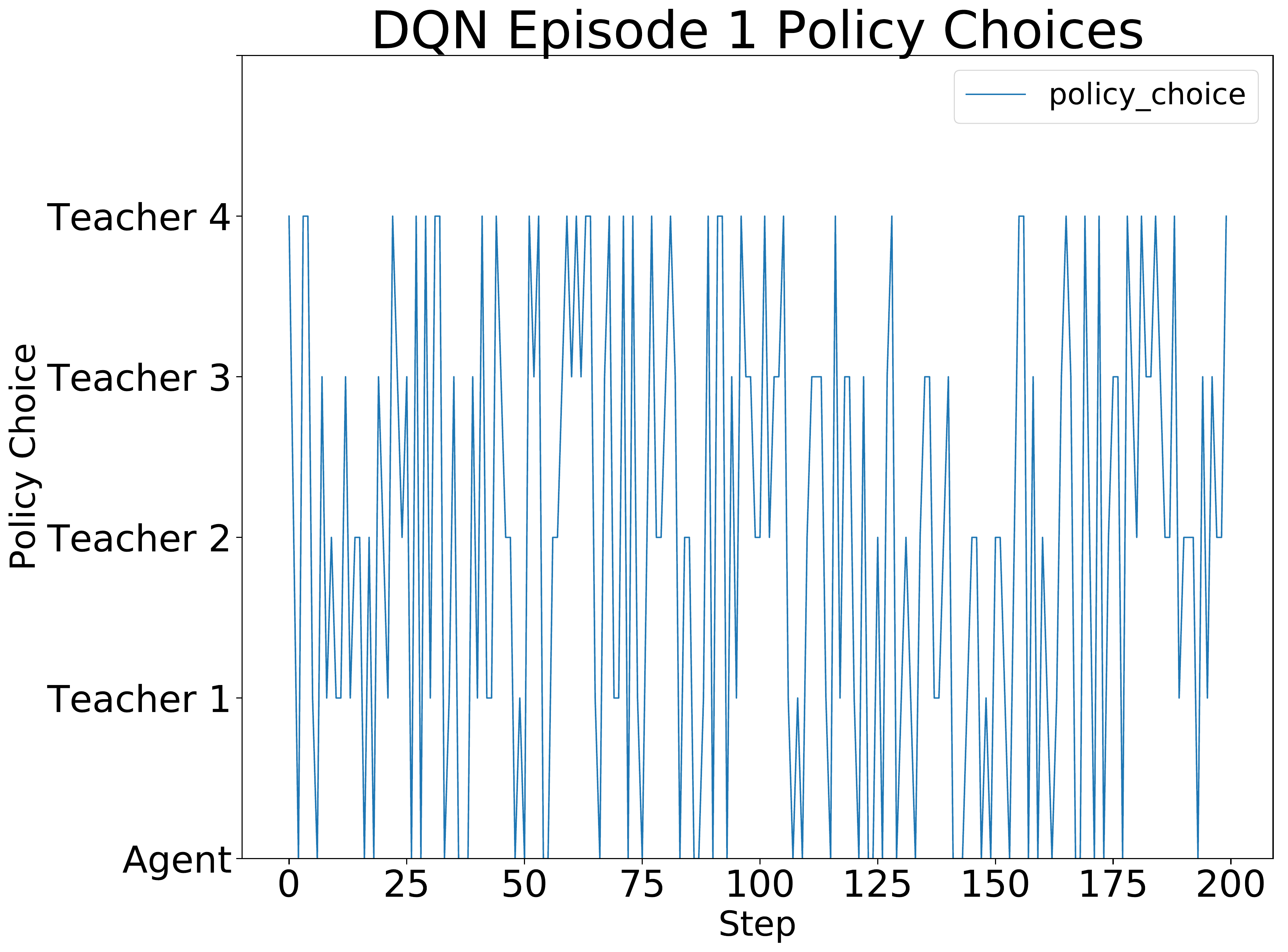}
    \end{subfigure}
    \hfill
    \begin{subfigure}[b]{0.325\textwidth}
        \centering
        \includegraphics[width=\textwidth]{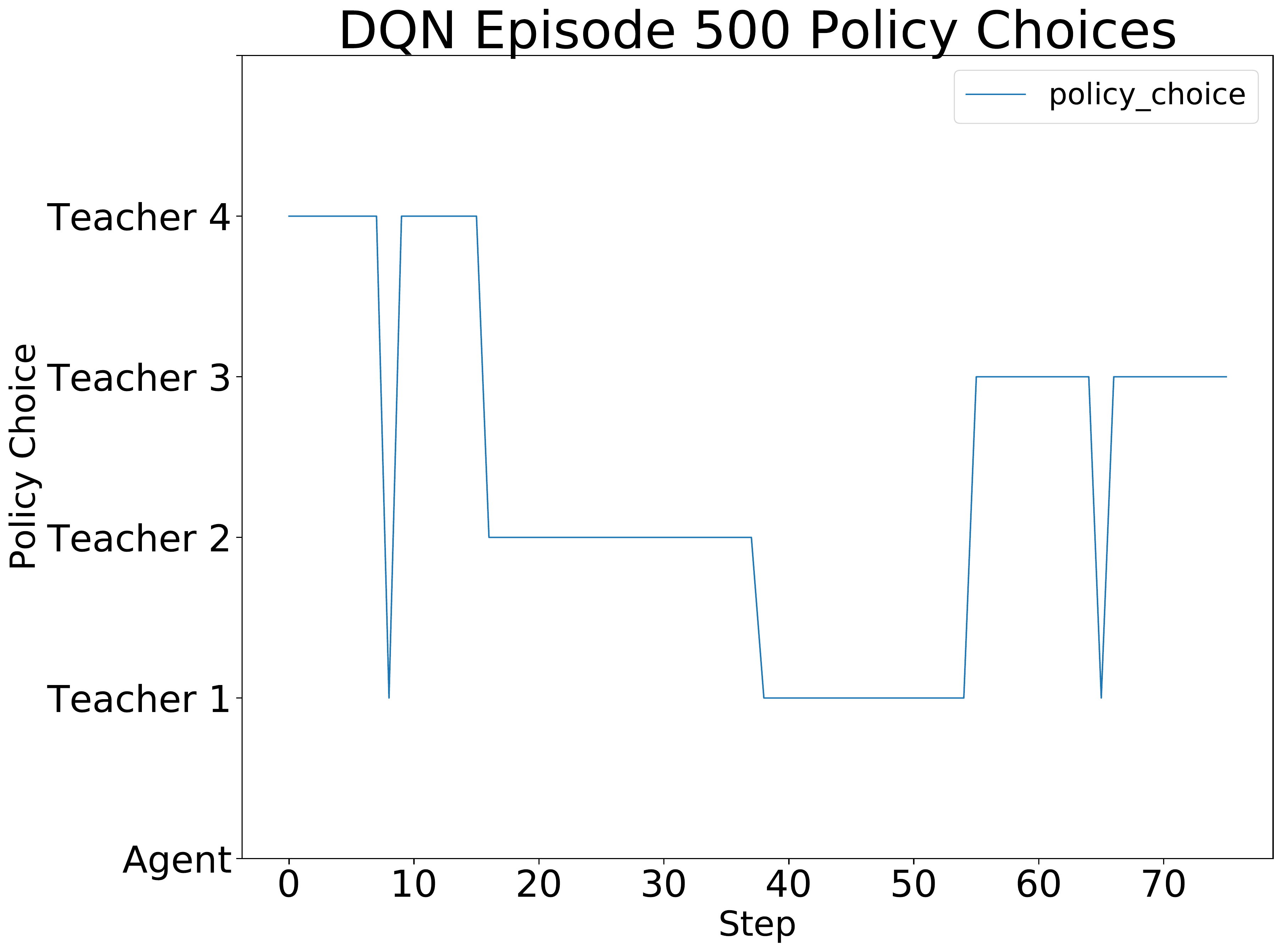}
    \end{subfigure}
    \hfill
    
    \begin{subfigure}[b]{0.325\textwidth}
        \centering
        \includegraphics[width=\textwidth]{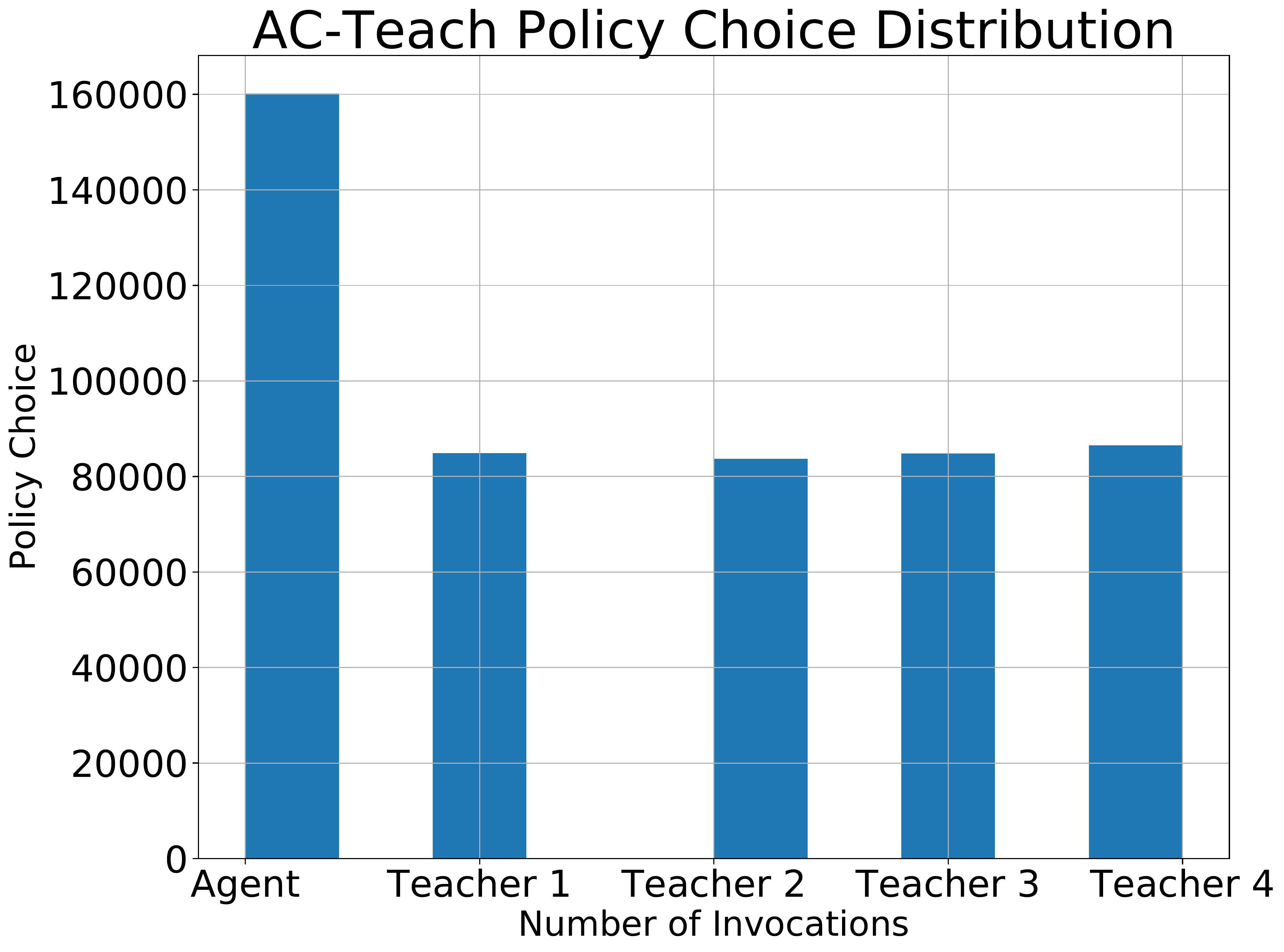}
    \end{subfigure}%
    \hfill
    \begin{subfigure}[b]{0.325\textwidth}
        \centering
        \includegraphics[width=\textwidth]{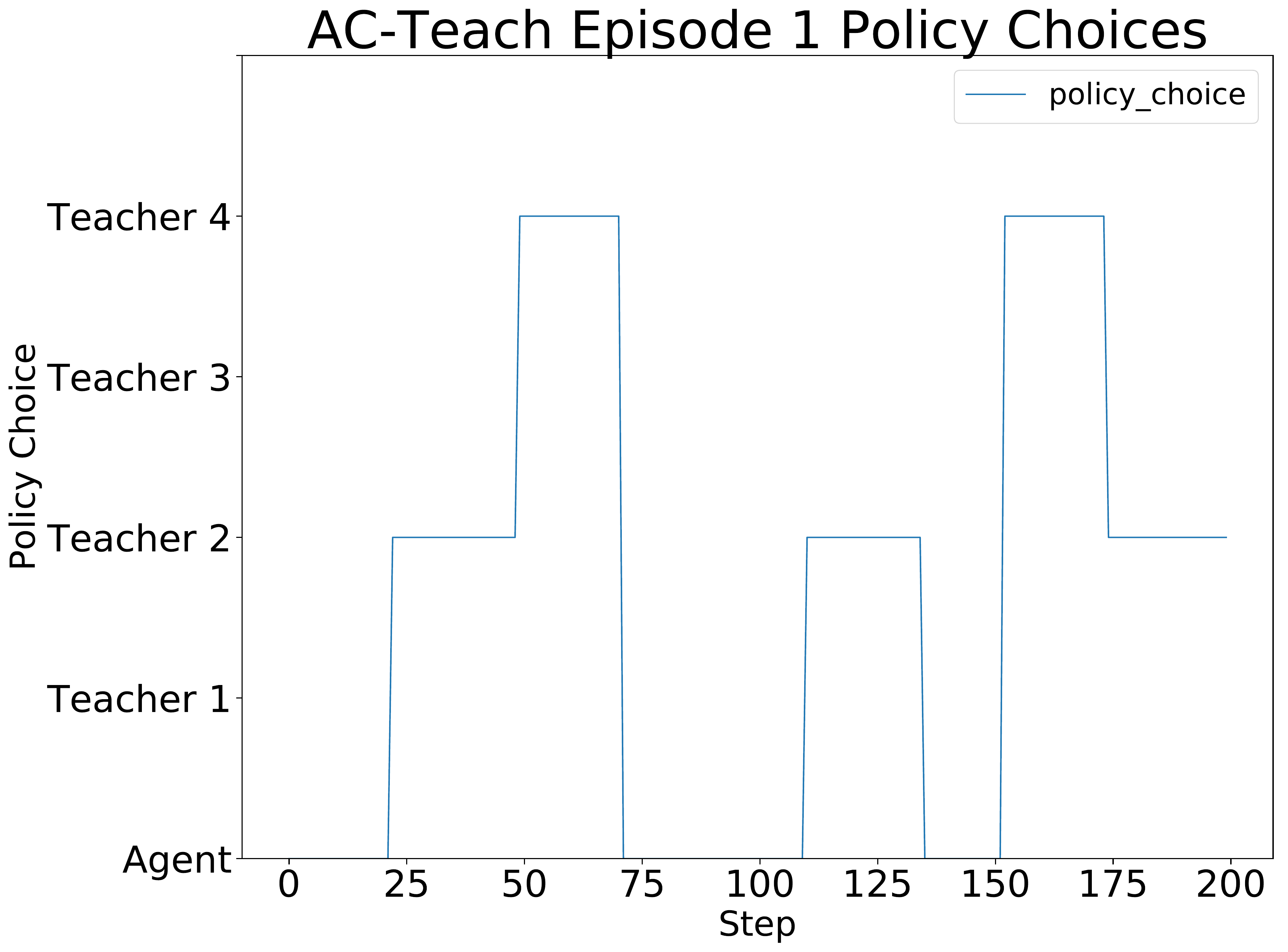}
    \end{subfigure}
    \hfill
    \begin{subfigure}[b]{0.325\textwidth}
        \centering
        \includegraphics[width=\textwidth]{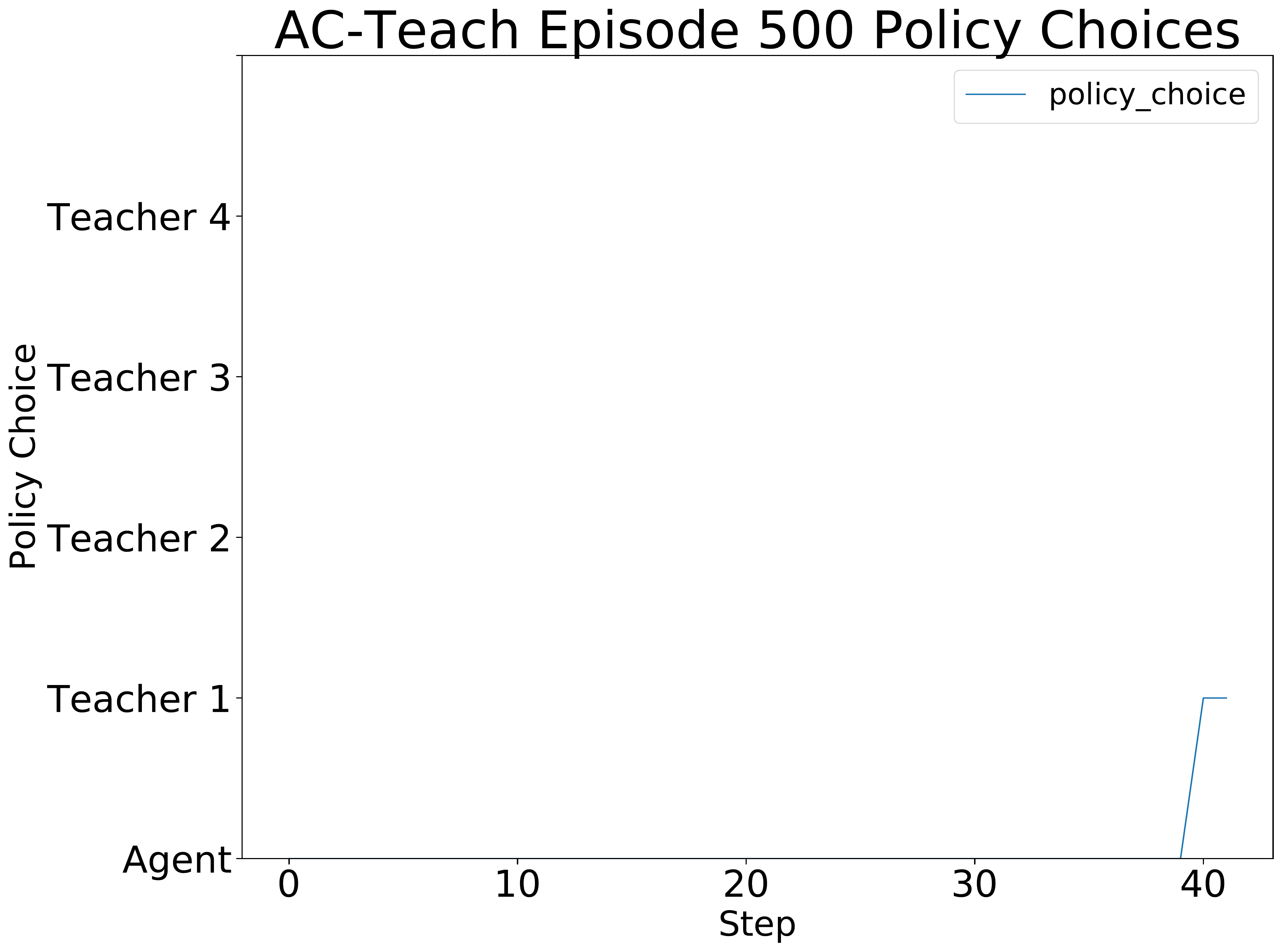}
    \end{subfigure}
    \hfill
   \caption{\textbf{Comparing policy choices of the~\methodname behavioural policy with the DQN behavioural policy.} Results from the \texttt{Path Following} task with the sufficient noisy teacher set. The first column demonstrates that over the course of training~\methodname leverages the agent more than the DQN. The second column demonstrates that our commitment mechanism results in less switching between different policies than the baselines. The third column demonstrates that at the end of training the DQN learns to only leverage the teachers, whereas our method learns to use the agent.}
   \label{fig:path_d}
\end{figure}

Additionally, we construct four teacher sets to evaluate the tolerance of our algorithm to teacher sets that are larger and lower quality. Teacher sets E, F, and G all consist of four \textit{suboptimal} \textit{partial} teachers, but they additionally include 1, 2, and 4 additional random teachers respectively. Teacher set H consists of a \textit{suboptimal} \textit{sufficient} teacher and an \textit{adversarial} teacher that tries to move away from the current goal. 

Fig.~\ref{fig:path_c} shows the results of this experiment. The performance of our algorithm does not degrade in the presence of additional random teachers that do not help to accomplish the task (sets E, F, and G), while our algorithm quickly detects the \textit{adversarial} teacher, learns to ignore it and to query the advice of the \textit{sufficient} teacher in set H. 

We also present a breakdown of the policy choices made during training by the baseline DQN behavioural policy and the~\methodname behavioural policy in Fig.~\ref{fig:path_d}. The figure demonstrates that~\methodname leverages the learning agent more over the course of training, switches less between policies at the start of training due to its commitment mechanism, and converges to use the learner and not the teachers at the end of training. 

\subsection{\texttt{Hook Sweep} Experiments}
\label{app:hook_sweep_exps}

We present experiments on the \texttt{Hook Sweep} environment in Fig.~\ref{fig:hook_a}. 

\begin{figure}[!h]
   \centering
    \includegraphics[width=\textwidth]{figs/legend.png}
    \begin{subfigure}[b]{0.245\textwidth}
        \centering
        \includegraphics[width=\textwidth]{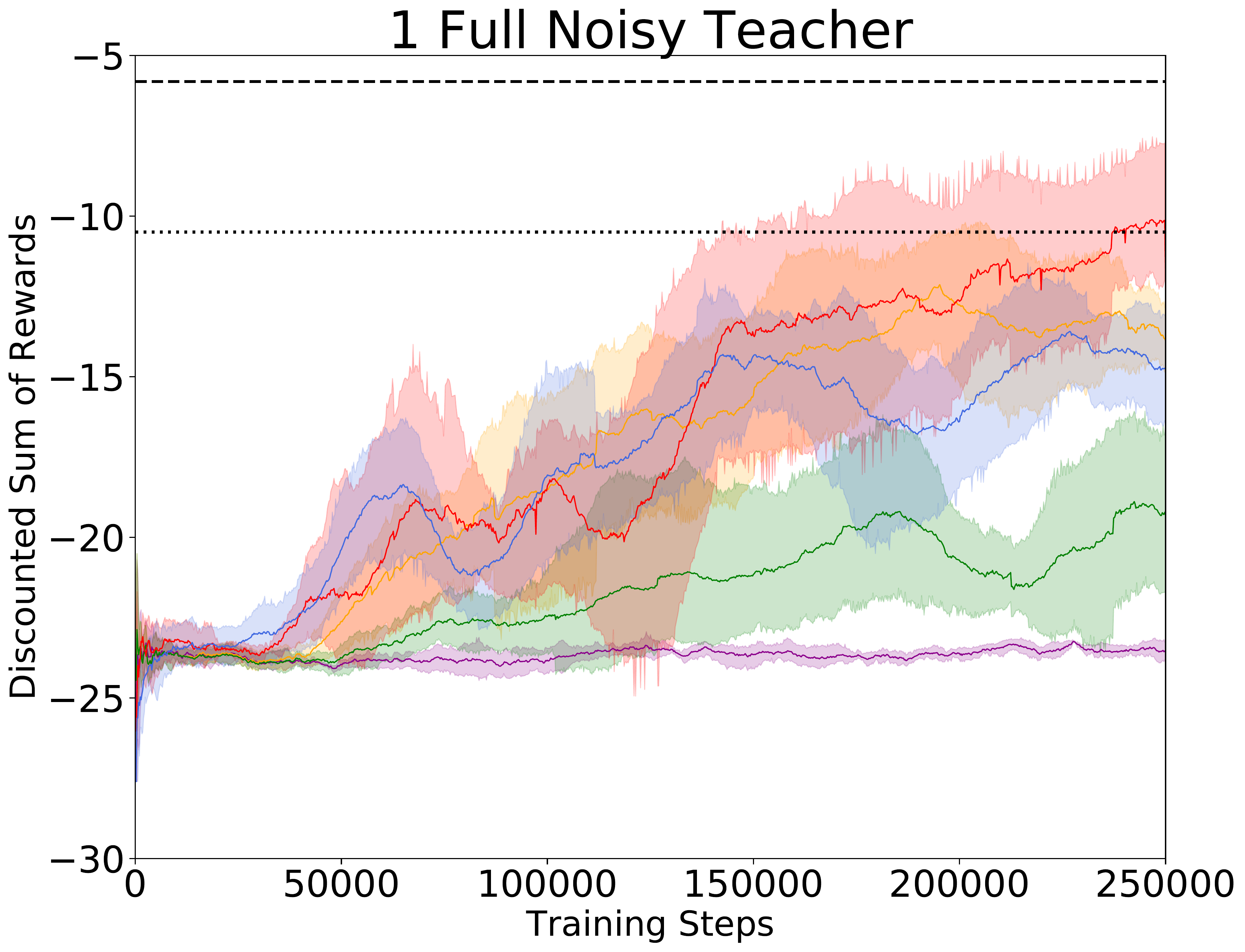}
    \end{subfigure}%
    \hfill
    \begin{subfigure}[b]{0.245\textwidth}
        \centering
        \includegraphics[width=\textwidth]{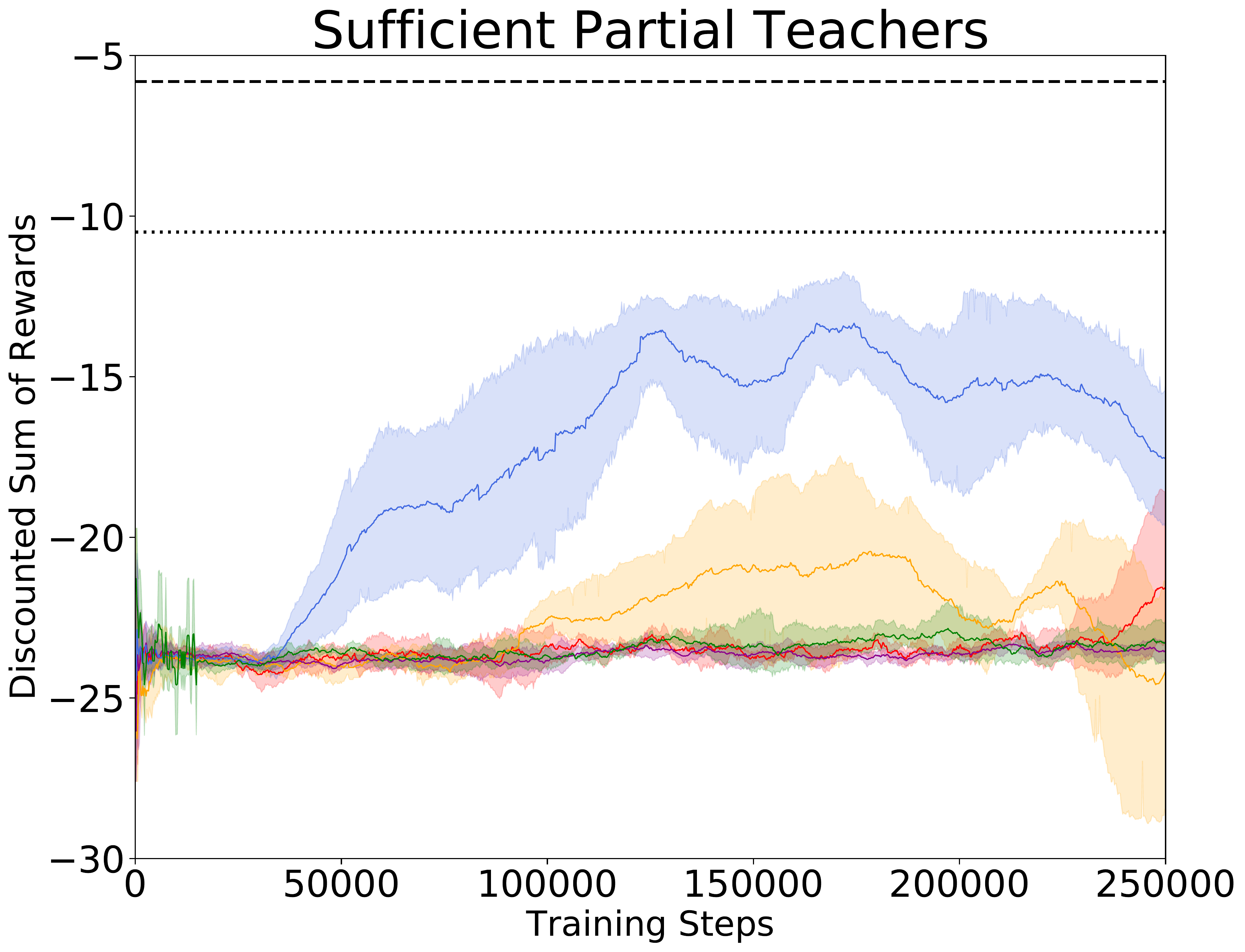}
    \end{subfigure}
    \hfill
    \begin{subfigure}[b]{0.245\textwidth}
        \centering
        \includegraphics[width=\textwidth]{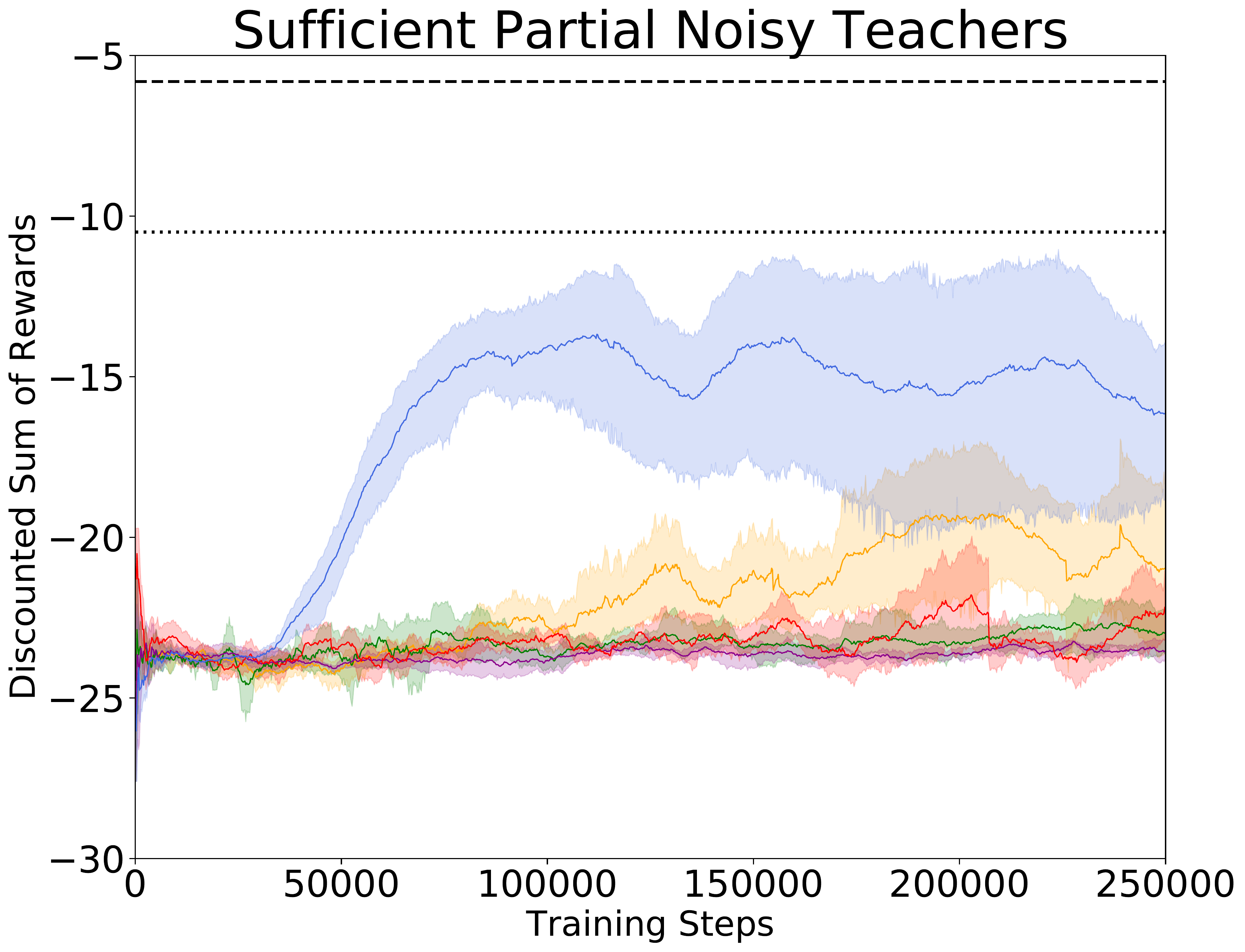}
    \end{subfigure}
    \hfill
    \begin{subfigure}[b]{0.245\textwidth}
        \centering
        \includegraphics[width=\textwidth]{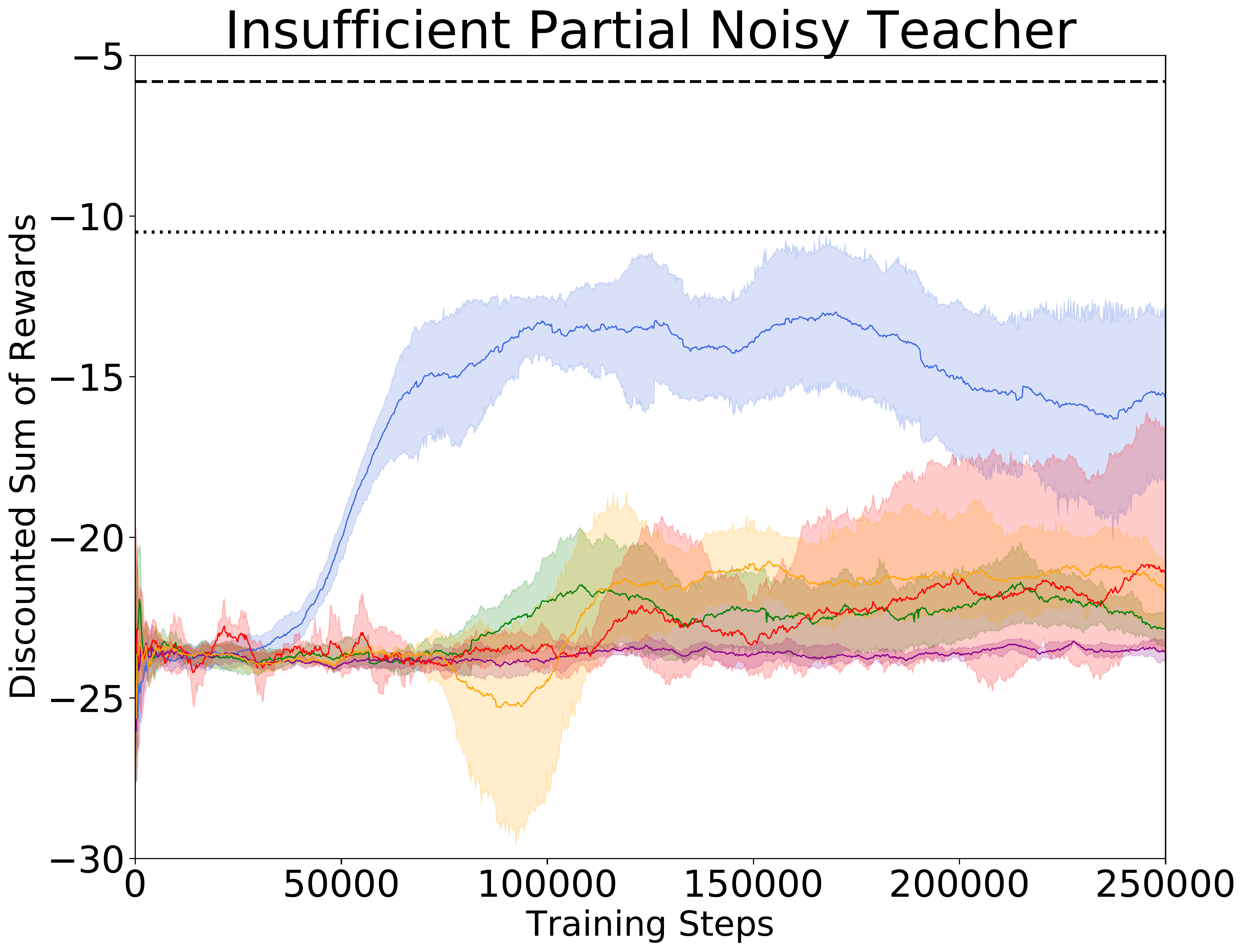}
    \end{subfigure}
    \begin{subfigure}[b]{0.245\textwidth}
        \centering
        \includegraphics[width=\textwidth]{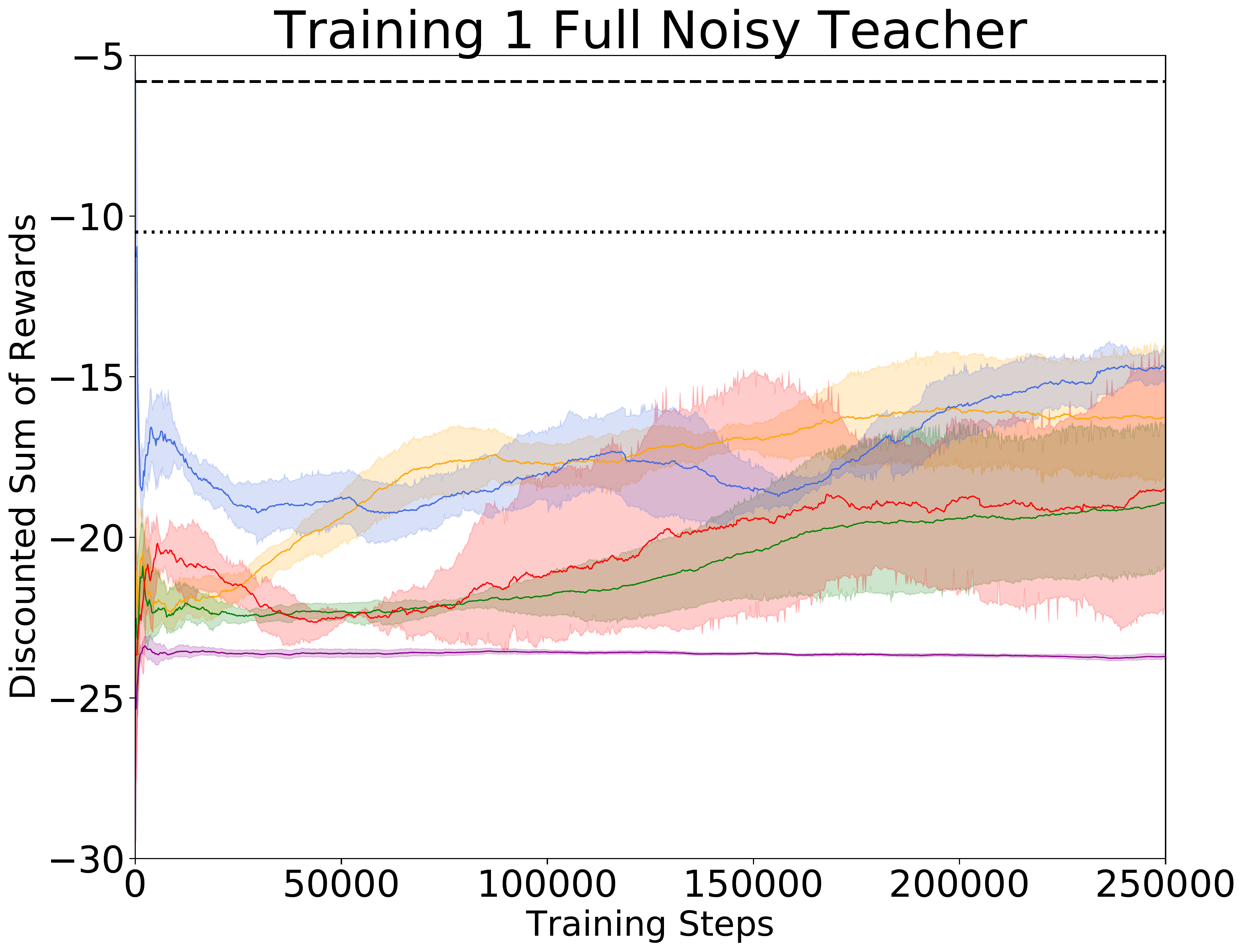}
    \end{subfigure}%
    \hfill
    \begin{subfigure}[b]{0.245\textwidth}
        \centering
        \includegraphics[width=\textwidth]{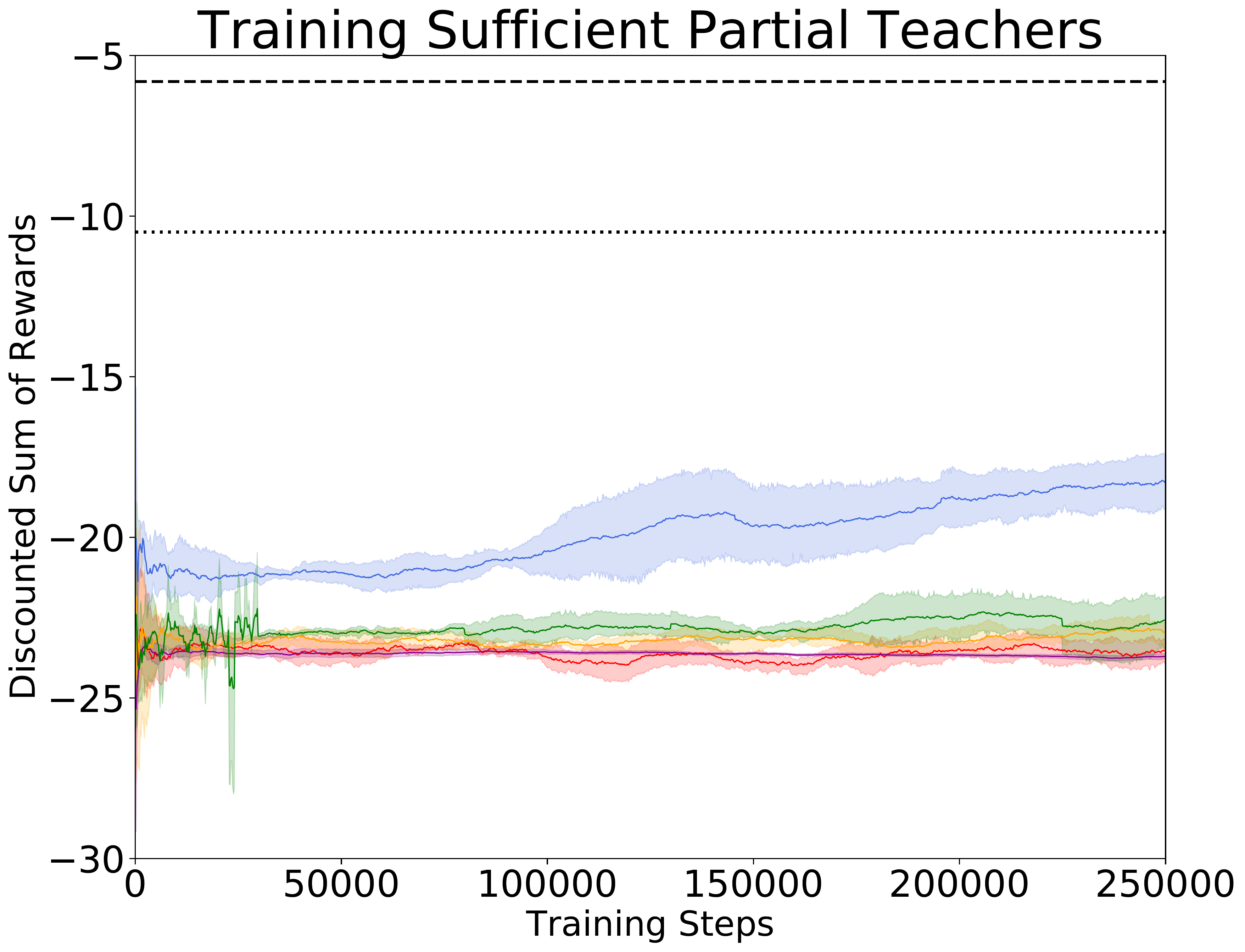}
    \end{subfigure}
    \hfill
    \begin{subfigure}[b]{0.245\textwidth}
        \centering
        \includegraphics[width=\textwidth]{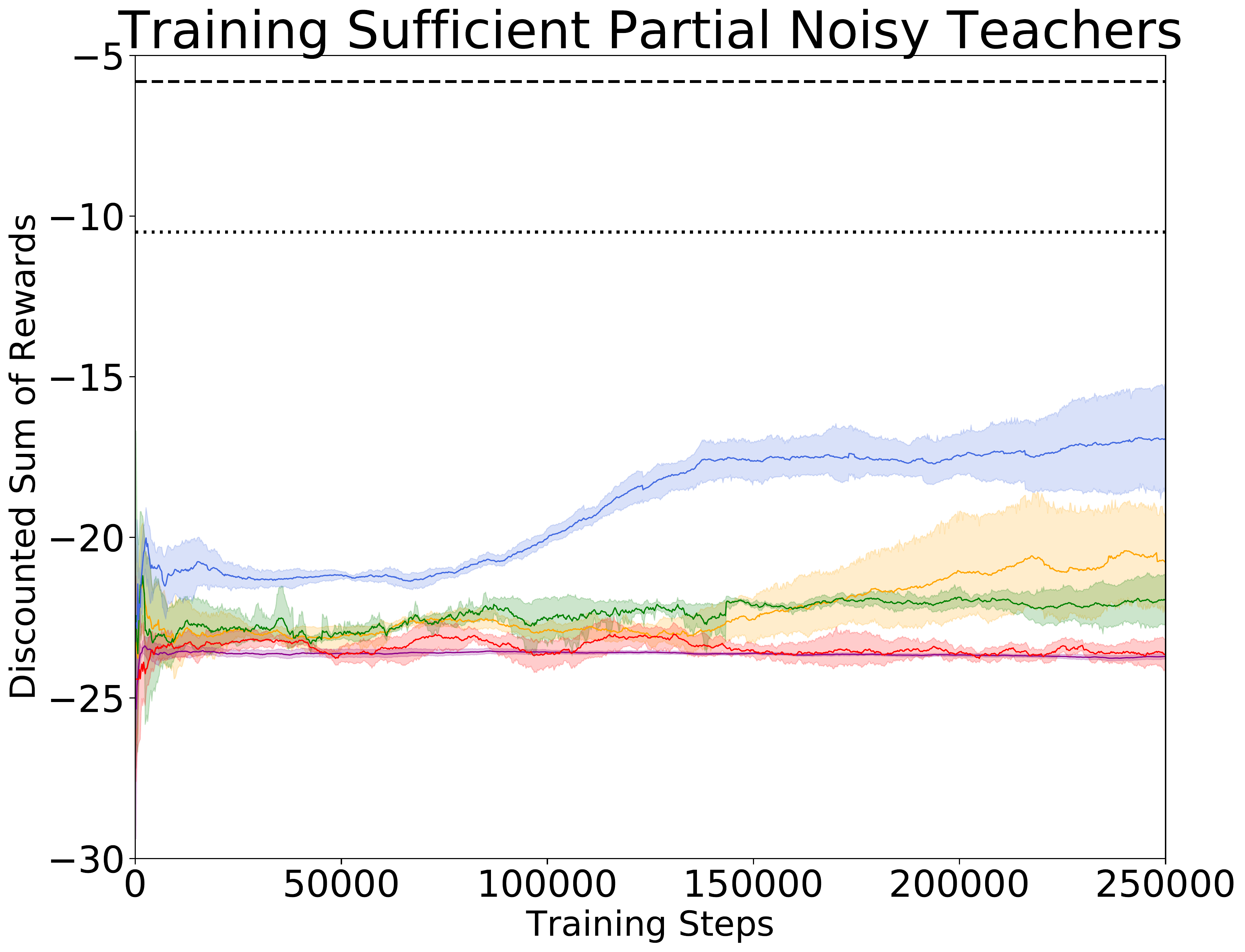}
    \end{subfigure}
    \hfill
    \begin{subfigure}[b]{0.245\textwidth}
        \centering
        \includegraphics[width=\textwidth]{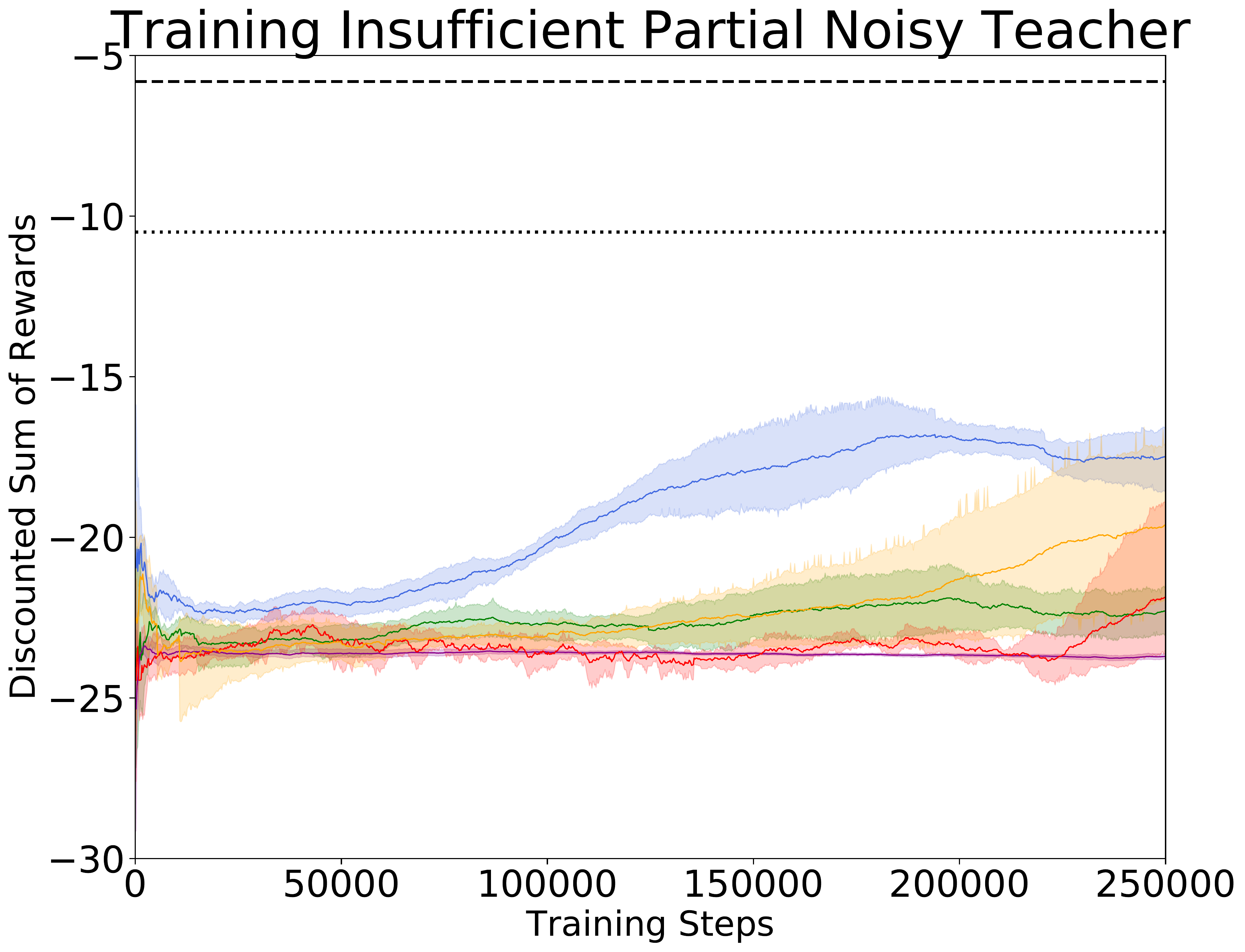}
    \end{subfigure}
   \caption{\textbf{Evaluating~\methodname on different teacher sets on the \texttt{Hook Sweep} task.} As with the prior tasks, we varied the types of teachers we trained the agent for the hook sweep task with. As in the prior tasks, our algorithm is demonstrated to be the most robust.}
   \label{fig:hook_a}
\end{figure}

\subsection{Ablation Analysis}
\label{app:ablation}
We test several modified versions of \methodname to evaluate the effect of each of its components, in particular different variants of time commitment and behavioral target. Fig.~\ref{fig:ablation} (a) and (b) depicts the result of our analysis. Interestingly, time commitment does not grant any benefit in the \texttt{Path Following} task despite being useful in the \texttt{Pick and Place} and \texttt{Hook Sweep} tasks. On the contrary, the behavioral target is not relevant for the performance of \methodname in the \texttt{Pick and Place} and \texttt{Hook Sweep} tasks despite being important in the \texttt{Path Following}. 

\begin{figure}[!h]
\centering
\includegraphics[width=\textwidth]{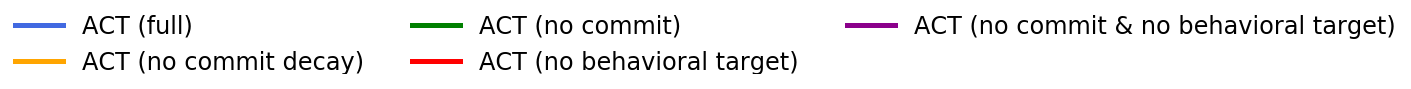}
\begin{subfigure}[b]{0.325\textwidth}
\centering
\includegraphics[width=\textwidth]{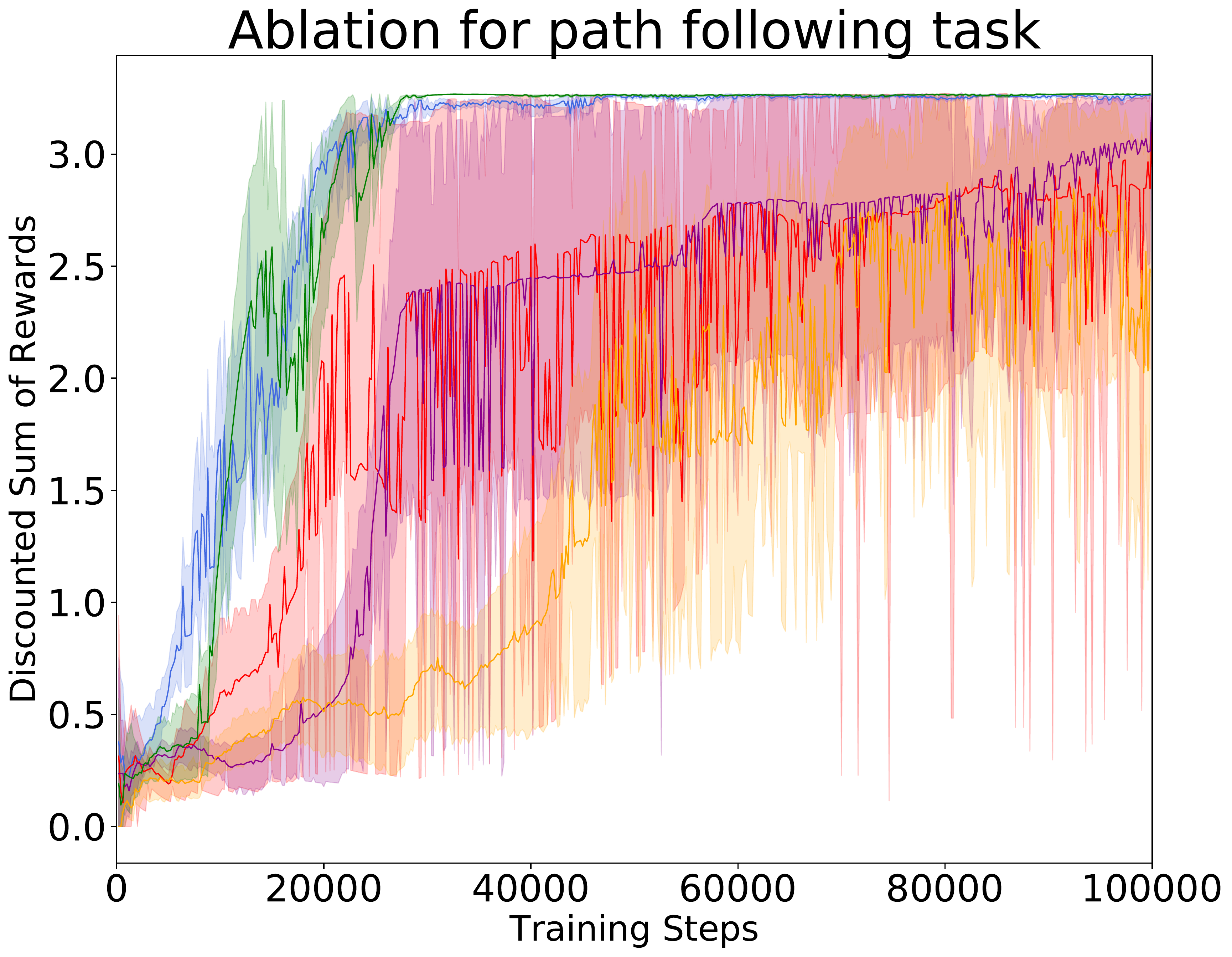}
\caption{}
\end{subfigure}%
\hfill
\begin{subfigure}[b]{0.325\textwidth}
\centering
\includegraphics[width=\textwidth]{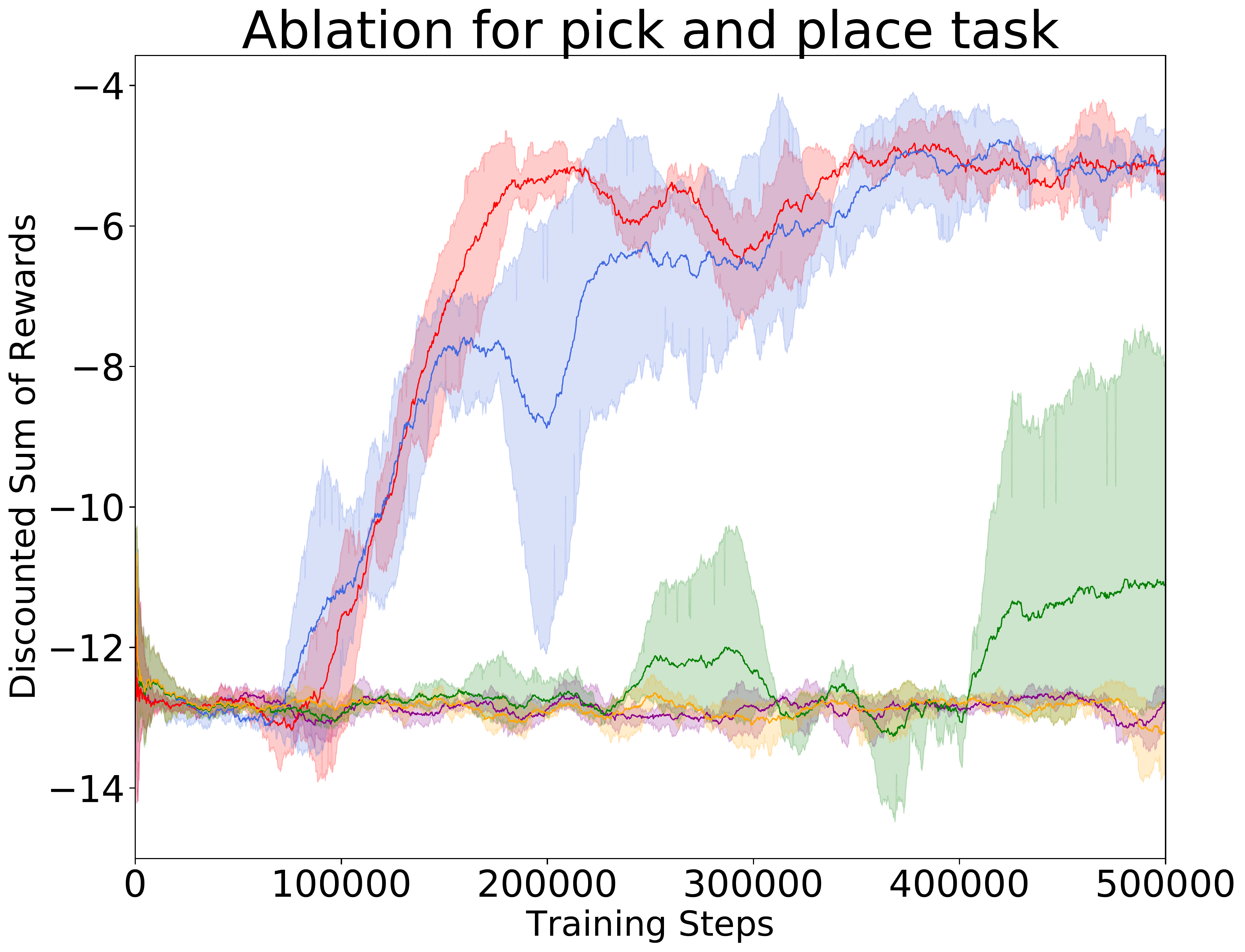}
\caption{}
\end{subfigure}
\hfill
\begin{subfigure}[b]{0.325\textwidth}
\centering
\includegraphics[width=\textwidth]{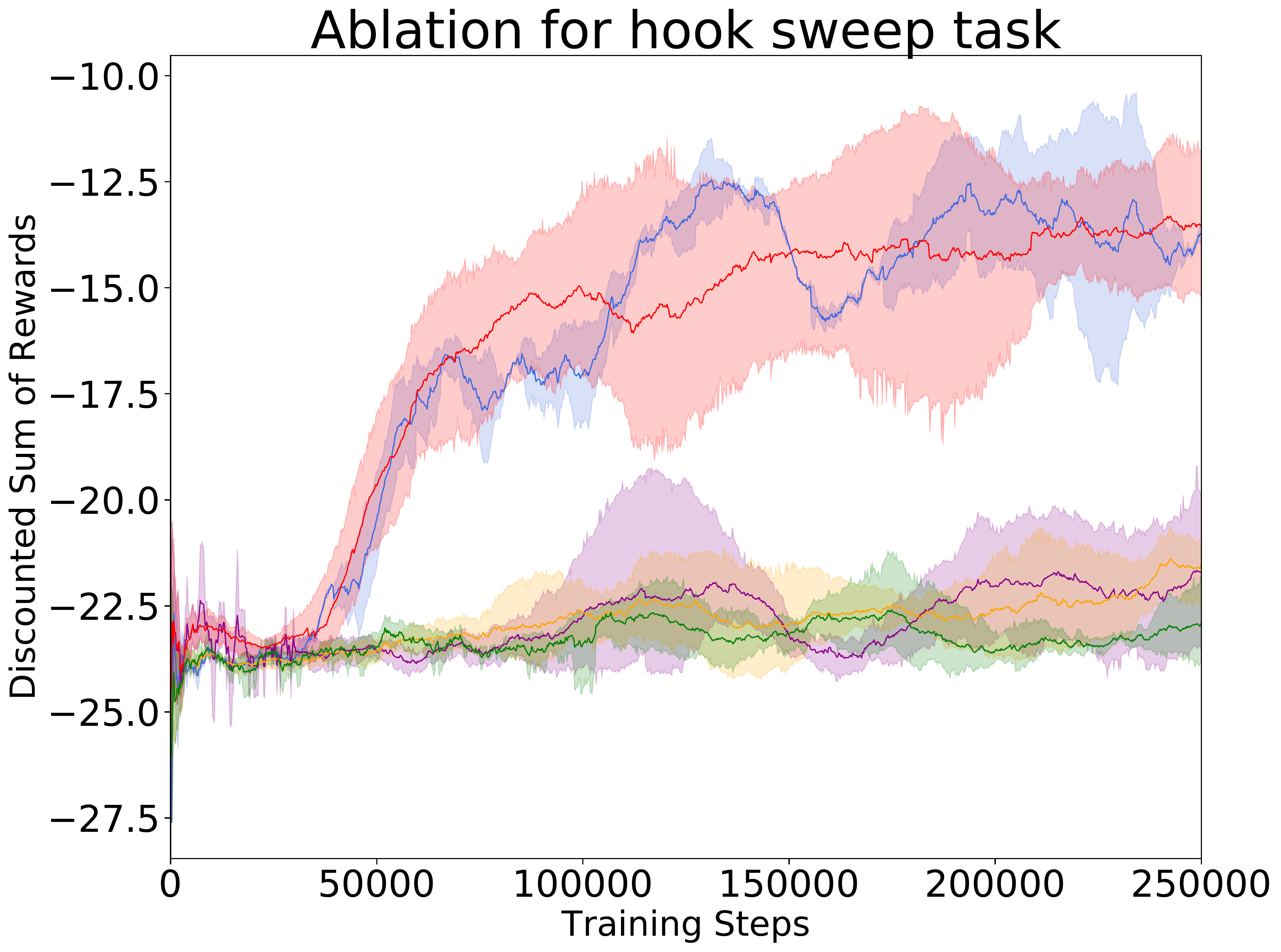}
\caption{}
\end{subfigure}
\caption{\textbf{Ablation Analysis} for the \texttt{Path Following} (a), \texttt{Pick and Place} (b), and \texttt{Hook Sweep} (c) tasks.}
\label{fig:ablation}
\end{figure}

We hypothesize that the tasks present different characteristics that benefit (or not) from the features of \methodname. In the case of \texttt{Path Following}, it includes little stochasticity and so the potential extrapolation error that is countered with the behavioral target is larger. For the \texttt{Pick and Place} task, both the end-effector and cube locations are randomized so the collected experience is varied enough, making extrapolation error less problematic. However, the shorter time horizon to complete the \texttt{Pick and Place} task and the possibility of pushing the cube away make commitment beneficial.

\subsection{Sensitivity Analysis}
\label{app:sensitivity}
We also evaluate the sensitivity of our algorithm to different parameters and hyperparameters. In our sensitivity analysis we use the \texttt{Path Following} task because the simplified dynamics allows to better dissentangle and analyze the results.

Fig.~\ref{fig:sensi} depict the results of our sensitivity analysis. We observe that most parameters can have a range of values and still work well. An exception to this is that we also observe that our algorithm is sensitive to the value of the dropout precision parameter, $\tau$, which is used in establishing the regularization of the critic network as well as in the computation of the probability of a new policy being better in our commitment mechanism. With more or less samples than the ones we used in our experiments the performance drops significantly. 

\begin{figure}[h!]
\centering
\begin{subfigure}[b]{0.495\textwidth}
\includegraphics[width = 0.99\textwidth]{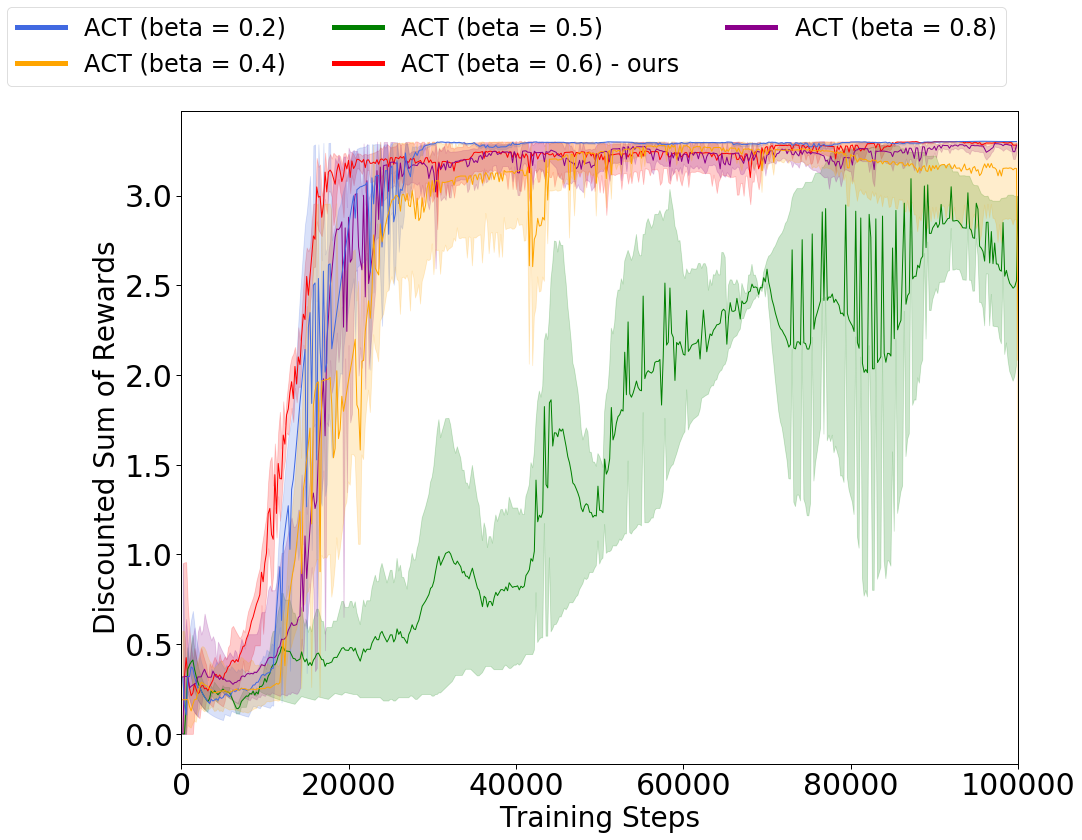}
\caption{}
\end{subfigure}
\begin{subfigure}[b]{0.495\textwidth}
\includegraphics[width = 0.99\textwidth]{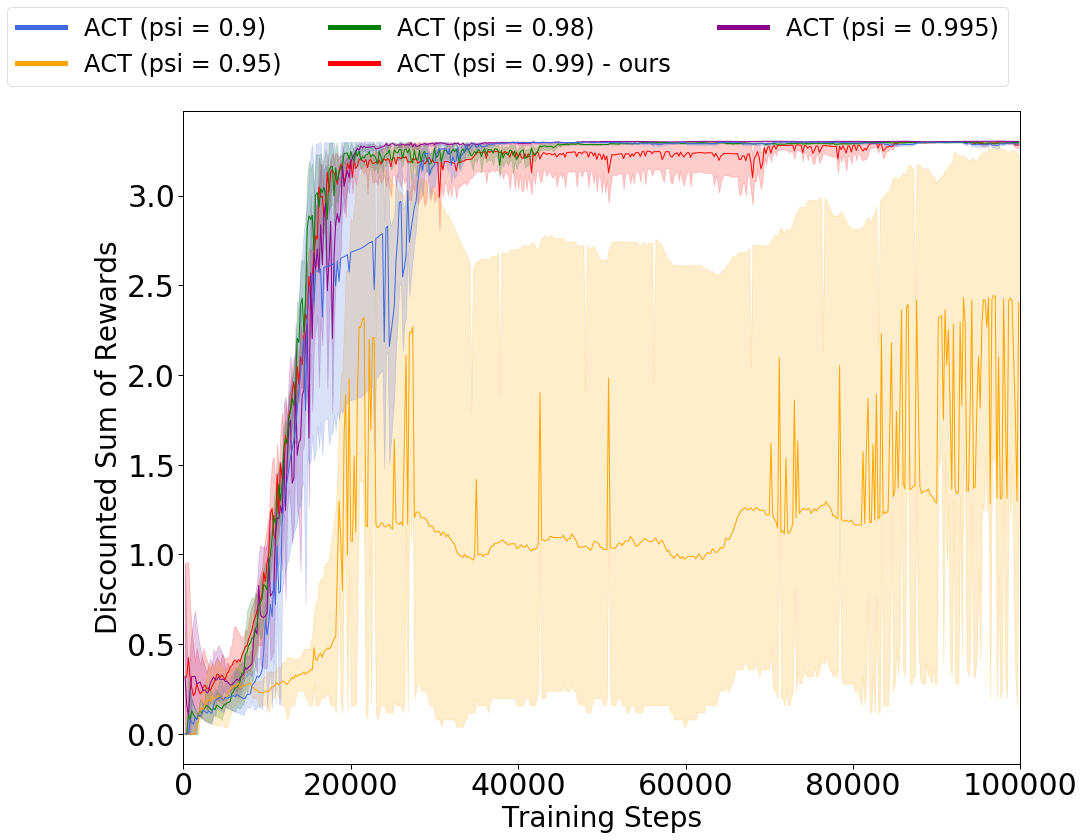}
\caption{}
\end{subfigure}
\begin{subfigure}[b]{0.495\textwidth}
\includegraphics[width = 0.99\textwidth]{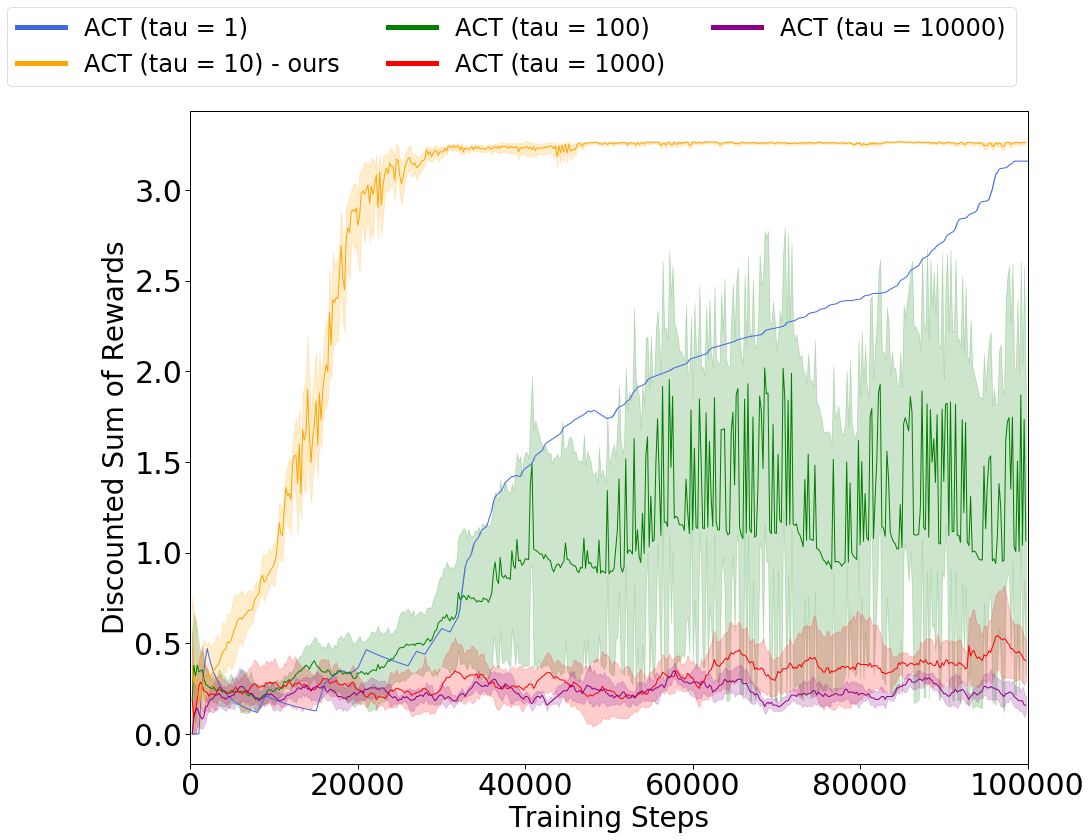}
\caption{}
\end{subfigure}
\begin{subfigure}[b]{0.495\textwidth}
\includegraphics[width = 0.99\textwidth]{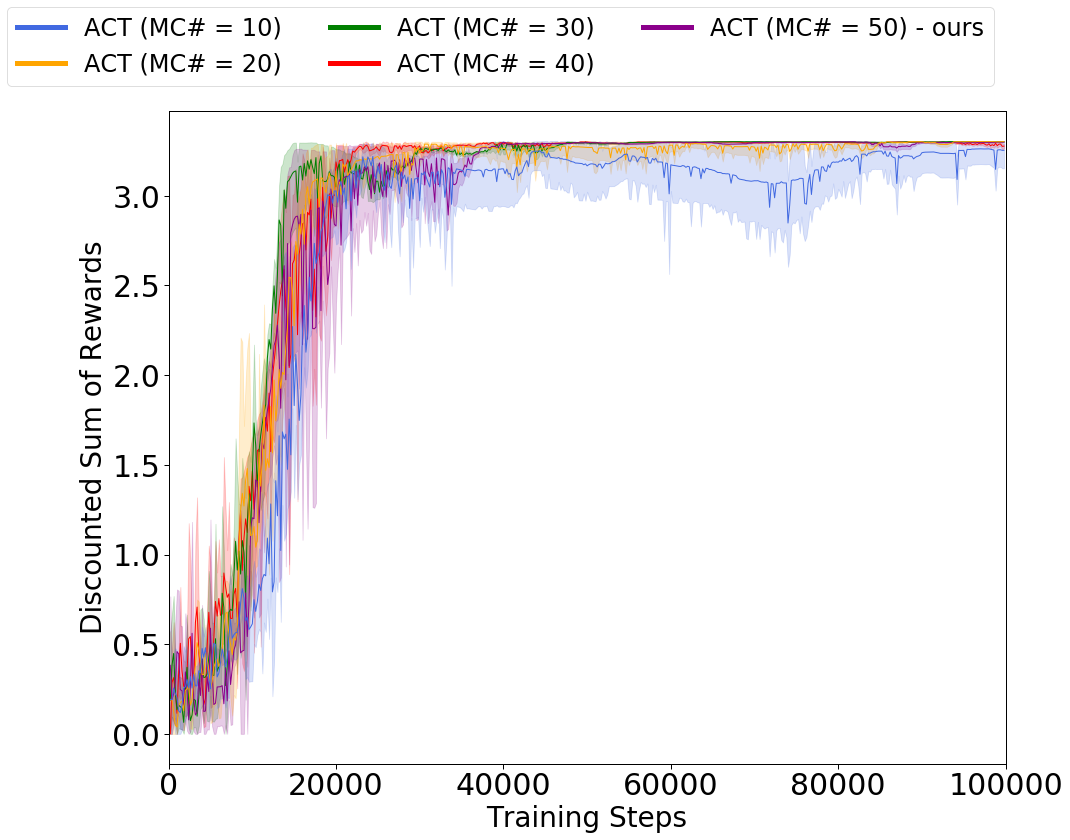}
\caption{}
\end{subfigure}
\caption{\textbf{Sensitivity Analysis} to $\beta$ (a), $\psi$ (b), $tau$ (c) parameters, and to the number of dropout Monte Carlo samples (c)}
\label{fig:sensi}
\end{figure}


\section{Experimental Setup}


\subsection{Tasks}
\label{app:tasks}

\textbf{\texttt{Path Following}:} A point agent starts each episode at the origin of a 2D plane and needs to visit the four corners of a square centered at the origin. These corners must be visited in a specific order that is randomly sampled at each episode. We set the goal locations to be at $(-0.25, -0.25), (-0.25, 0.25), (0.25, -0.25)$, and $(0.25, 0.25)$. A goal location is considered visited if the agent position is within 0.05 L2 distance. The observations are 5-dimensional and consist of its current location, the current goal point, and the number of goals left to visit. The agent generates bounded 2D displacements as actions with deterministic outcome. The max offset in each axis per single action is 0.045. A sparse reward of 1 is given when the agent visits each corner, making 4 the maximum undiscounted return per episode. Episodes are 200 steps long and do not terminate early.

\textbf{\texttt{Pick and Place}:} A robot manipulation problem from~\cite{fetch}. The goal is to pick up a cube and place it at a target location on the table surface. The initial position of the object and the robot end effector are randomized at each episode start, but the goal is constant. The initial cube location in each episode is centered near $(1.25, 0.55)$, with random offsets of magnitude $[-0.05,0.05]$ in the x and y planes.
The goal location is at $(1.45, 0.55, 0.425)$ -- an arbitrarily chosen location near one of the corners of the table. The end-effector begins above the center of the table near $(1.34, 0.75, 0.5)$ with random offsets of magnitude $[-0.05,0.05]$ in the x and y planes and a random offset of magnitude $[-0.025, 0.025]$ in the z plane. The observation and action space is the same as in Plappert et al.~\cite{fetch}. The agent applies delta position commands to the end effector and can actuate the gripper. Rewards are dense and defined as the negative L2 distance from the cube to the target location. Episodes are 100 steps long and do not terminate early.

\textbf{\texttt{Hook Sweep}:} A robot manipulation problem adapted from~\cite{silver2018residual}. The objective is to actuate a robot arm to move a cube to a particular goal location. The cube is initialized out of reach of the robot arm, and so the robot must use a hook to sweep the cube to the goal. The goal location and initial cube location are randomized such that in some episodes the robot arm must use the hook to sweep the cube closer to its base and in other episodes the robot arm must use the hook to push the cube away from its base to a location far from the robot. On every environment reset, with 50\% probability, the cube is initialized in front of the hook, and the goal is set to be far away from the robot base, and with 50\% probability, the cube is initialized behind the hook, and the goal is set to be close to the robot base. Thus, the robot has to learn to leverage the hook to either sweep the cube closer to itself or sweep the cube further away, depending on the goal initialization. The observation and action space is the same as in Plappert et al.~\cite{fetch}, with additional observation dimensions added for the hook state. The agent applies delta position commands to the end effector and can actuate the gripper. Rewards are dense and defined as the negative L2 distance from the cube to the target location. Episodes are 100 steps long and do not terminate early.

\subsection{Implementation}
We implement the BDDPG agent as well as the BDDPG + DQN and DDPG + Critic baselines on top of the existing implementation in Stable Baselines \cite{stable-baselines} along with relevant modifications from \cite{henderson2017bayesian} taken from \url{https://github.com/Breakend/BayesianPolicyGradients}. 

\subsection{Training}
All experiment results are presented over 5 seeds. Neither observations nor rewards were normalized. The \texttt{Path Following} task was run by alternating $200$ steps of task interaction with $100$ gradient updates to the policy, with a soft update to the target after every gradient update, until $100k$ task interactions were reached. The \texttt{Path Following} task was run by alternating $200$ steps of task interaction with $40$ gradient updates to the policy, with a soft update to the target after every gradient update, until $500k$ task interactions were reached.

\subsection{ACT Hyperparameters}

Unless otherwise stated for specific experiments, we use commitment threshold $\beta=0.6$ and commitment decay $\psi=0.99$ in experiments throughout the paper.

\subsection{DDPG Hyperparameters}
We use most hyperparameters from~\cite{henderson2017bayesian}, with values derived from \url{https://github.com/Breakend/BayesianPolicyGradients}, as shown in Table~\ref{tab:hyp1}.

\begin{figure}[h!]
\centering
\begin{subfigure}[t]{0.45\textwidth}
    \centering
    \resizebox{\linewidth}{!}{%
    \begin{tabular}{|l|c|}
        \cline{1-2}
        Hidden Layers (\texttt{Path Following}) & $[64,64]$ \\ \cline{1-2}
        Hidden Layers (\texttt{Mujoco}) & $[64,64,64]$ \\ \cline{1-2}
        DDPG~Exploration (\texttt{Path Following}) & $\mathcal{N}(0, 0.3)$ \\ \cline{1-2}
        DDPG~Exploration (\texttt{Mujoco}) & $\mathcal{N}(0, 0.1)$ \\ \cline{1-2}
        Target Update $\tau$ (\texttt{Path Following}) & $0.01$ \\ \cline{1-2}
        Target Update $\tau$ (\texttt{Mujoco}) & $0.001$ \\ \cline{1-2}
        $\alpha$ & $0.5$ \\ \cline{1-2}
        Dropout~$\tau$ & $10.0$ \\ \cline{1-2}
        Discount~Factor & $0.99$ \\ \cline{1-2}
        Keep~Prob (\texttt{Path Following}) & $0.8$ \\ \cline{1-2}
        Keep~Prob (\texttt{Mujoco}) & $0.9$ \\ \cline{1-2}
        MC~Samples & $50$ \\ \cline{1-2}
        Batch~Size & $128$ \\ \cline{1-2}
        Reward~Scale & $1$ \\ \cline{1-2}
        Actor~L2~Loss (\texttt{Path Following}) & $0.0$ \\ \cline{1-2}
        Actor~L2~Loss (\texttt{Mujoco}) & $0.1$ \\ \cline{1-2}
    \end{tabular}
    }
    \caption{DDPG Hyperparameters}
    \label{tab:hyp1}
\end{subfigure}
\hfill
\begin{subfigure}[t]{0.525\textwidth}
    \centering
    \resizebox{\linewidth}{!}{%
    \begin{tabular}{|l|c|}
        \cline{1-2}
        Hidden Layers & $[64,64]$ \\ \cline{1-2}
        Learning Rate & $5.0e-4$ \\ \cline{1-2}
        Exploration Final $\epsilon$ & $0.02$ \\ \cline{1-2}
        Exploration Timesteps & $100000$ \\ \cline{1-2}
        Buffer Size & $100000$ \\ \cline{1-2}
        Discount~Factor & $0.99$ \\ \cline{1-2}
        Train~Freq & $10$ \\ \cline{1-2}
        Target~Network~Update~Freq & $1000$ \\ \cline{1-2}
        Batch~Size & $32$ \\ \cline{1-2}
        Reward~Scale & $1$ \\ \cline{1-2}
    \end{tabular}
    }
    \caption{DQN Hyperparameters}
    \label{tab:hyp2}
\end{subfigure}
\end{figure}


Parameters are largely the same across the three tasks, except for the several with listed differences above; these were found via manual tuning and seem to reflect the Mujoco tasks being harder.

\subsection{DQN hyperparameters}
We use DQN with non-prioritized replay, with the hyperparameters listed in Table~\ref{tab:hyp2}.
